\documentclass[11pt]{article}
\usepackage[left=2.54cm,top=2.54cm,right=2.54cm,bottom=2.54cm]{geometry}

\input{macros/preamble.tex}

\usepackage{pgffor}
\usepackage{rotating}
\usepackage{color}
\usepackage[colorlinks=true,
            linkcolor=black,
            urlcolor=blue,
            citecolor=blue]{hyperref}

\newcommand{\tht}{\btheta}

\newcommand{\ths}{\btheta^*}
\newcommand{\hth}{\hat{\btheta}}

\newcommand{\re}[1]{{\color{black}#1}}

\begin{document}




  \title{\huge Contextual Online Uncertainty-Aware Preference Learning for Human Feedback}

  \author{Nan Lu\thanks{
    Department of Biostatistics, Harvard T.H. Chan School of Public Health,  \textit{nanlu@hsph.harvard.edu}}\hspace{.2cm}~~~~
    Ethan Lee\thanks{
    Department of Biostatistics, Harvard T.H. Chan School of Public Health,  \textit{ethanlee@g.harvard.edu}}\hspace{.2cm}~~~~
    Ethan X. Fang \thanks{Department of Biostatistics \& Bioinformatics, Duke University, \textit{ethan.fang@duke.edu}}\hspace{.2cm}~~~~
    Junwei Lu\thanks{
    Department of Biostatistics, Harvard T.H. Chan School of Public Health,  \textit{junweilu@hsph.harvard.edu}}\hspace{.2cm}~~~~}
\date{}
  \maketitle


\bigskip
\begin{abstract}
Reinforcement Learning from Human Feedback (RLHF) has become a pivotal paradigm in artificial intelligence to align large models with human preferences. In this paper, we propose a novel statistical framework to simultaneously conduct the online decision-making and statistical inference on the optimal model using human preference data based on dynamic contextual information. Our approach introduces an efficient decision strategy that achieves both the optimal regret bound and the asymptotic distribution of the estimators. A key challenge in RLHF is handling the dependent online human preference outcomes with dynamic contexts. To address this, in the methodological aspect, we propose a two-stage algorithm starting with $\epsilon$-greedy followed by exploitations; in the theoretical aspect, 
we tailor anti-concentration inequalities and matrix martingale concentration techniques to derive the uniform estimation rate and asymptotic normality of the estimators using dependent samples from both stages. Extensive simulation results demonstrate that our method outperforms state-of-the-art strategies. We apply the proposed framework to analyze the human preference data for ranking large language models on the Massive Multitask Language Understanding dataset, yielding insightful results on the performance of different large language models for medical anatomy knowledge.
\end{abstract}

\noindent%



\section{Introduction}
Recent advances in large language models (LLMs), model alignment, and decision-making processes have revolutionized the field of natural language processing, enabling significant progress in solving complex tasks. Central to these advancements is model alignment with human preferences, which serves as a cornerstone for enhancing the performance and usability of LLMs. In particular, reinforcement Learning from Human Feedback (RLHF) emerges as a transformative paradigm, allowing machine learning models to harness human evaluators' feedback for optimization \citep{christiano2023deepreinforcementlearninghuman, pporl}, and the RLHF  finds tremendous successes in application systems such as ChatGPT \citep{openai2024gpt4technicalreport}, Claude \citep{claude2024}, and LLaMA2 \citep{touvron2023llama}. Unlike conventional approaches that rely on preset reward functions, RLHF directly optimizes rewards grounded in human preferences. A pivotal technique for modeling human preferences within RLHF is the celebrated Bradley-Terry-Luce (BTL) model for ranking~\citep{pporl, touvron2023llama, dpo}. Specifically, when a human labeler is presented with two items, say the answers from models $i$ and $j$, the preference probability under the BTL model is
\begin{align}\label{BTsimple}
\PP(\text{item }j\text{ is preferred over }i)=
\frac{e^{s_j}}{e^{s_i}+e^{s_j}},
\end{align}
where $s_i$ and $s_j$ represent the latent scores of items $i$ and $j$, respectively.

Besides the applications in LLMs, the BTL model is widely applied to product recommendation systems \citep{he2018adversarial,LiMS}, information retrieval \citep{dwork2001rank,cossock2006subset,geyik2019fairness} and other areas \citep{elo1967proposed,Aouad} where user interactions naturally manifest as preference data. In many such applications, we collect data in online fashions, which requires the decision-makers to make a series of adaptive decisions based on progressively accumulated information. In such online decision-making problems, we select one of the available actions in each new time period, and observe a random reward. We consider the setting where we observe contextual information $\bX_t$  at time $t$. Based on this context, consider the problem of comparing items $i$ and~$j$. Following the BTL model~\eqref{BTsimple}, we have
\begin{align}\label{cBT}
\PP(\text{model }j\text{ is preferred over }i)=
\frac{e^{\bX_{t}^\top\btheta_j^{*}}}{e^{\bX_{t}^\top\btheta_i^{*}}+{e^{\bX_{t}^\top\btheta_j^{*}}}},
\end{align}
where $\ths_i$ and $\ths_j$ denote the intrinsic attributes of items $i$ and $j$, respectively. Here we model the latent scores as a linear model of the contextual information $\bX_t$. For many AI models using the architectures including multi-layer perceptrons and transformers \citep{alpaca,kim2023preference}, the last layer of these models is typically a fully connected network, and the input of this layer can be used as the contextual information $\bX_t$. Learning the intrinsic attributes $\btheta_i$'s from \eqref{cBT} is to fine-tune the last fully connected layer aligning with human preferences.
We note that this contextual setting is 
different from most existing approaches for ranking problems in the statistical literature \citep{Negahban2017,chen2019spectral, Gao2021}. In these papers, they assume that the models or items have some fixed latent scores. 
 While effective in certain applications, this assumption limits practical applicability for dynamic situations. \citet{fan2024uncertaintyquantificationmleentity} consider item-dependent covariates, but their covariate effect is modeled by a common parameter shared across different items which cannot capture the model-specific characteristics.  In particular, in RLHF scenarios, human preference outcomes often vary across different models and contexts. This motivates us to consider the contextual framework \eqref{cBT}, where item scores are determined by the interaction between contextual factors and item-specific attributes. 

In our setting, we consider a finite time horizon of length $T$ that at each time $t$, the decision-maker performs actions $(i_t,j_t)$, i.e., comparing the selected item $i_t$ with another item $j_t$, over $T$ rounds, and we let the reward of item $i$ with context $\bX_t$ be the latent score $\bX_t^{\top} \btheta_i^*$. To evaluate the performance of our algorithm, we use the metric of average regret, which quantifies the shortfall in reward achieved by the policy compared with an oracle with perfect knowledge of the true item attributes. We define the accumulated regret that
\begin{align}\label{eq:regret}
R(T)=\frac{1}{T}\sum_{t=1}^{T}\Big(\bX_{t}^\top\btheta_{i_t^*}^{*}-\bX_{t}^\top\btheta_{i_t}^{*}\Big),
\end{align}
where $i_t^*$ denotes the optimal item such that $i_t^*=\argmax_{i\in[n]}\bX_{t}^\top\btheta_i^{*}$. 
There are two major objectives for our online uncertainty-aware preference learning of the model \eqref{cBT}: (1) to select the optimal item minimizing the regret, and (2) to conduct inference on the estimators of item attributes $\btheta_i^{*}$'s. These two objectives leverage the goals from both AI model alignment to identify the optimal model and the statistical inference to quantify the uncertainty of the model evaluation, thus empowering the RLHF procedure with reproducibility and reliability. 
For the first objective of minimizing regret, it is crucial to balance exploration and exploitation. 
On one hand, exploiting existing knowledge allows us to select the item believed to yield the highest reward, but it risks missing opportunities to gather information that could lead to better decisions in the future. On the other hand, exploring actions with uncertain rewards enables more exploration but may result in suboptimal immediate outcomes. To balance, various strategies have been developed, including randomized and interval-based approaches. Randomized strategies, such as $\epsilon$-greedy, exploit the current best-known action by default, but occasionally explore alternative actions with a small probability,  enabling the algorithm to learn about uncertain options \citep{JMLR:v17:13-210,AOS1534,contextbandit21}. Interval-based strategies, such as the Upper Confidence Bound (UCB), systematically balance exploration and exploitation by selecting the action with the highest upper confidence bound on the expected reward, prioritizing promising yet under-explored options \citep{ucb1,pmlr-v22-kaufmann12}. Despite the aforementioned vast literature on these two strategies for standard bandit or reinforcement learning problems, these two strategies cannot be directly applied to preference learning using human feedback. These standard problems can observe the rewards directly assigned to the actions or items. In comparison, in preference learning, we can only observe the preference outcome from the model \eqref{cBT} instead of the latent scores as the rewards. Intuitively, comparing a pair of items only provides partial information about their relative ranks, but it offers limited insight into the absolute rank of an individual item among all the candidate items. This interconnected nature of preference learning introduces a unique challenge for effective decision-making. We introduce a novel two-stage algorithm to tackle this challenge: it starts with an $\epsilon$-greedy stage mixing both exploration and exploitation in order to achieve a good initialization for the global ranking information of all items, and then it transits to the exploitation stage to optimize the regret. We show that this two-stage algorithm can simultaneously achieve the second objective of assessing the uncertainty of attribute estimators by integrating the dependent observations from both stages to maximize efficiency.


\subsection{Major Contributions} 
We propose a novel algorithm with theoretical guarantees to address the challenges in the RLHF problem. Our major contributions are  four-fold.
\begin{itemize}
\item {\bf A novel two-stage algorithm for pairwise arms and preference outcomes.} Compared with standard bandit or reinforcement learning problems with single arms and single rewards \citep{ JMLR:v17:13-210, MSbandit,contextbandit21}, the RLHF involves pairwise comparison arms without observing the explicit rewards. That makes the estimation of the latent score of one item involve multiple pairwise preference outcomes. To address this problem, we propose a new two-stage algorithm that chooses the pairs to compare achieving both the optimal regret bounds and rates of convergence for estimation.


\item {\bf Estimation and inference using dependent decision-making observations.} Different from existing BTL literature \citep{chen2019spectral,lagran23,fan2024uncertaintyquantificationmleentity}, we adopt a novel linear contextual BTL model to be more compatible with the RLHF setting. This distincts our work in utilizing dependent samples generated from the RLHF decision-making process. We develop novel theoretical tools to characterize the rates of convergence for our estimators and asymptotic distributions of the latent scores using dependent samples from both exploration and exploitation stages.

\item {\bf A nearly optimal regret bound for online preference learning.} 
We characterize when transitioning from the $\epsilon$-greedy stage to the exploitation stage is optimal for the regret bound. Besides, we also identify the optimal $\epsilon$ for the $\epsilon$-greedy stage for both the dense and sparse comparison graphs. 
Following this, we establish a regret bound of order $O(T^{-1/2})$ matching the nearly optimal regret bound for standard bandit problems \citep{li2019nearly}. 

\item {\bf A new uncertainty-aware RLHF method to rank different LLMs.} Unlike RLHF approaches for tuning a specific LLM with a fixed architecture \citep{zhu2024principledreinforcementlearninghuman,xiong2024iterative}, our approach allows different parameters for different LLMs. Our inferential results also provide a statistical guarantee for the uncertainty assessment of the RLHF procedure. Practical analysis demonstrates that the proposed algorithm outperforms state-of-the-art strategies including the upper confidence bound (UCB) methods.
\end{itemize}

\subsection{Related Research}
The standard RLHF is popularized by \citet{christiano2023deepreinforcementlearninghuman}, drawing the attention of the RL community to preference-based feedback. The most widely adopted RLHF framework is detailed in the InstructGPT paper \citep{pporl}, Claude \citep{bai2022traininghelpfulharmlessassistant}, and LLaMA2 \citep{touvron2023llama}. \citet{pbrlnov} and \citet{drlsaha} investigate tabular online RLHF settings, while \citet{pbrlchen} study online RLHF with general function approximation. \citet{rlhfwang} adapt sample-efficient algorithms from standard reward-based RL to online RLHF, while further algorithmic advancements are introduced by \citet{wu2024making}. A significant branch of RLHF focuses on learning the reward function from human feedback \citep{NEURIPS2020_2f10c157, rlhfsurvey}. The BTL model is widely used to model the pairwise comparison outcomes to align human preferences \citep{pporl, bai2022traininghelpfulharmlessassistant, dpo}. Despite numerous impactful applications of this framework, theoretical analysis remains limited. \citet{song2023reward} introduce prompt-aware utility functions to address reward collapse. \citet{zhu2024principledreinforcementlearninghuman} propose a pessimistic algorithm for offline RLHF, while \citet{zhan2024provable} analyze broader settings with general function approximations. Additionally, \citet{tiapkin2024demonstrationregularized} and \citet{xiong2024iterative} explore a KL-regularized reward optimization and its theoretical properties. Moreover, \citet{xiao2024algorithmic} present preference matching RLHF to mitigate algorithmic bias, and \citet{zhong2024provable} propose multi-party RLHF to align with heterogeneous individual preferences. In contrast to these studies, our work focuses on item-specific attributes and introduces a novel problem setting, which leads to completely different policy strategies and statistical theoretical~formulation.

The bandit problems focus on designing policies that maximize rewards by selecting arms based on learning from observed rewards. An extension is the contextual bandit problem, introduced by \citet{ucbauer2002}. Contextual bandit techniques have been applied in a range of personalization applications, such as healthcare \citep{Tewari2017,pmlr-v84-kallus18a,mintz} and recommendation systems \citep{li, Agrawal}.
Several approaches have been proposed in the contextual bandit literature.  \citet{pmlr-v28-agrawal13} introduce a randomized algorithm based on a Bayesian approach. \citet{JMLR:v17:13-210} provide an analysis of a randomized strategy based on kernel estimation, while \citet{pmlr-v80-foster18a} develop confidence-based methods utilizing regression oracles. \citet{bastaniOR} investigate the linear contextual bandit problem in high-dimensional settings, and \citet{MSbandit} propose a greedy-first strategy. Additionally, \citet{AOS1534} examine the asymptotic properties of reward under a general parametric model. \citet{kannanbandit} study scenarios in which an adversary selects the observed contexts, and \citet{contextbandit21} study the $\epsilon$-greedy policy while addressing model misspecification.
\re{\citet{deshpande2018accurate} consider the inference problem for sequentially collected data under standard linear models with history-dependent covariates and propose decorrelated estimators, exemplified through multi-armed bandit applications, while \citet{khamaru2021near} further study the optimality of such inference procedures.}
Different from these works, our work focuses on human preference feedback, where the observed data consists of pairwise comparisons rather than direct rewards for individual arms. This paradigm shift changes the decision-making process: instead of selecting a single item to evaluate, we choose an item pair to compare. In the pairwise comparison setting, estimating different item attributes becomes highly correlated, necessitating distinct action policies and theoretical analysis. 

For the BTL model, \citet{Negahban2017} develop a spectral method. \citet{chen2019spectral} study the uniform bound of item ability estimations. \citet{Gao2021} introduce the leave-two-out technique to derive the asymptotic distributions of the item ability estimators. Meanwhile, \citet{lagran23} propose a Lagrangian debiasing approach to obtain asymptotic distributions for ranking scores and build a novel combinatorial inference framework to infer the ranking properties. Additionally, \citet{FANmultiway} and \citet{fan2024spectralrankinginferencesbased} provide comprehensive studies on the multiway comparison model. However, a crucial assumption underlying all these works is that the item scores are assumed to be fixed, meaning the models do not account for covariates. One exception is \citet{fan2024uncertaintyquantificationmleentity}, where the authors consider item-dependent covariates, with the covariate effect modeled by a common parameter shared across different items. In contrast, our framework addresses a scenario where different items interact with contextual covariates in unique ways, depending on their individual abilities. From an applied perspective, \citet{Chiang} directly employ the BTL model to rank LLMs. By comparison, our approach focuses on selecting LLMs based on online prompts while adhering to budget constraints.  Furthermore, all the aforementioned works assume that the samples are independent, meaning the pairwise comparison results are independent across different item pairs, and the observations for each pair are independently and identically distributed. In contrast, we focus on a more challenging yet practical case, where dynamic environmental information, combined with estimations based on historical data, determines the compared pairs. This introduces a complex dependent relationship between observations. This significant shift offers compelling practical advantages but simultaneously introduces distinct theoretical challenges.
\re{While \citet{oh2019thompson} and \citet{chen2020dynamic} incorporate contextual information with sequentially collected data, they focus on the multinomial logit choice model with observable item-specific contextual features and a common unknown factor. In contrast, we focus on the contextual BTL model with unknown item-specific parameters, which leads to alternative policy approaches and customized statistical theory.}\\

\noindent{\bf Notations.} We write $a_n\lesssim b_n$ or $a_n=O(b_n)$ if there exists a constant $c>0$ such that $a_n\leq cb_n$ for all $n$. We denote $a_n\asymp b_n$ if $b_n\lesssim a_n$ and $a_n\lesssim b_n$. Besides, we write $a_n=o(b_n)$ if $\lim_{n\rightarrow\infty}a_n/b_n=0$. We denote by $[n]=1,\ldots,n$ for any positive integer $n$. Let $\be_i,\, i\in [n]$, be the canonical basis in $\RR^n$. $\Ib_{d}$ represents $d$-dimensional identity matrix, and $\bm{0}$ represents the vector or matrix composed entirely of zeros. We denote by $\otimes$  the Kronecker product. For a vector $\bv$,  $\|\bv\|_2$  denotes the $\ell_2$-norm. We let $\|\Ab\|_2$ be the spectral norm of the matrix $\Ab$, and $\lambda_{\min}(\Ab)$ represent the minimum eigenvalue of square matrix $\Ab$. Throughout our discussion, we use generic constants $c$ and $C$ which may differ from one place to another.\\

\noindent{\bf Paper Organization.} 
The remainder of the paper is structured as follows. Section~\ref{sec:frm} introduces the problem formulation. We then formally detail the proposed method and algorithm in Section~\ref{sec:BT}. Section~\ref{sec:thry} presents the theoretical results, including the estimation error, regret bound, and parameter inference. Section~\ref{sec:simu} showcases simulations and Section~\ref{sec:real} demonstrates an application for ranking large language models. We conclude the paper in Section~\ref{sec:conclu}. Proofs are provided in the supplemental material.


\section{Contextual Bradley-Terry-Luce Model}\label{sec:frm}
We consider the task of making ranking decisions over  $T$ time periods, indexed by $t\in\{1,2,\ldots,T\}$, where we rank $n$ items. At each time  $t\in[T]$, we have access to contextual information represented as a vector $\bX_{t}\in\RR^d$. This contextual vector may include prompt-specific information relevant to the decision-making process. We assume that $\bX_t$'s are independent and identically distributed. 

For each $t\in[T]$,  action $\ba_t$ involves selecting a comparison pair $(i_t,j_t)$, where $i_t,j_t$ are two different items. After taking action $\ba_t$, we observe $M$ pairwise comparison results between the two items, denoted as $\{y_{ij}^{(m)}(t)\}_{m\in[M]}$, where $M$ is a positive integer. We note that our method applies to the case where the number of comparisons may vary over time, and we assume a constant number of comparisons only for the simplicity of presentation. As we discussed in the introduction, given contextual information $\bX_t$, we model the performance of item $i$ by its latent score  $\bX_t^\top\btheta_{i}^{*} $, where $\ths_i\in \RR^d$ is item $i$'s true item-specific attribute. Specifically, for ranking LLMs, $\bX_t$ corresponds to the prompts, and $\btheta_i^{*}$ captures the strengths of  LLM $i$. Next, we adopt the celebrated BTL model to model the comparison that the $m$-th comparison at time $t$ that
%
%
%
%
%
%
%
\begin{align*}
y_{ij}^{(m)}(t) = \begin{cases}1 & \text { with probability } \frac{e^{\bX_{t}^\top\btheta_j^{*}}}{e^{\bX_{t}^\top\btheta_i^{*}}+{e^{\bX_{t}^\top\btheta_j^{*}}}}, \\ 0 & \text { otherwise,}\end{cases}
\end{align*}
where $y_{ij}^{(m)}(t) = 1$ represents that item $j$ is preferred over item $i$ in the corresponding comparison, and  $y_{ij}^{(m)}(t) = 0$ otherwise. 
We refer  this as the contextual BTL model.
Given  action $\ba_t = \{i_t,j_t\}$, we assume that
$y_{i_{t},j_{t}}(t)=\{y_{i_{t},j_{t}}^{(m)}(t), m\in[M] \}$ are $M$ independent observations.

For  contextual information $\bX_t$, recall that we let $i_t^*$ denote the optimal item such that $\bX_{t}^\top\btheta_{i_t^*}^{*}=\max_{i\in[n]}\bX_{t}^\top\btheta_i^{*}$. At the beginning of round $t+1$, we have the $\sigma$-field generated by the historical data that
\begin{align*}
\cH_{t}=\sigma\Big(\bX_1,\ba_1,y_{i_1,j_1}(1),\ldots,\bX_{t},\ba_{t},y_{i_{t},j_{t}}(t)\Big),
\end{align*}
and $\ba_{t+1}$ is determined based on $\cH_{t}$. We assume that  contextual information $\bX_{t+1}$ is independent of history $\cH_{t}$. And the pairwise comparison outcome $y_{i_t,j_t}(t)$ is conditionally independent of $\cH_{t-1}$ given  action $\ba_t$ and  context $\bX_t$.
For the selected pair $(i_t,j_t)$, we let~$i_t$ be the major item for ease of presentation. As $\bX_{t}^\top\btheta_i^{*}$ represents the score of item $i$ in the context of $\bX_t$, the objective is to pick the major item in each round as close as possible to the optimal one, in the sense that $\bX_{t}^\top\btheta_{i_t^*}^{*}-\bX_{t}^\top\btheta_{i_t}^{*}$ is small.
We aim to design an action policy  to minimize the cumulative regret $R(T)$ defined in \eqref{eq:regret}.


For most cases, any pair of items can be chosen for comparison. 
However, for some applications with limited resources \citep{ahmed2024studying,sebastian2023privacy}, we may only compare certain pairs of items restricted to a certain comparison graph \citep{negahban2012iterative}. To handle such cases, we use a random graph to model the comparison graph. In particular, let $\mathcal{G}=(\mathcal{V}, \mathcal{E})$ represent a graph characterizing the comparison scheme for the $n$ items, where the vertex set $\mathcal{V}=\{1,2, \ldots, n\}$ corresponds to the $n$ items of interest, and an edge $(i, j)$ is included in  edge set $\mathcal{E}$ if and only if items $i$ and $j$ can be compared. 
We assume that $\mathcal{G}$ follows the Erdős-Rényi (ER) graph model $\mathcal{G}_{n, p}$, where an edge between a pair of vertices exists  with probability~$p$ independently. This random graph framework naturally accommodates scenarios where some comparisons are infeasible. In the special case where all pairwise comparisons are feasible, the model simplifies to $p = 1$.
As for the contextual BTL model,  
since $1/(1+e^{\bX_{t}^\top(\ths_i-\ths_j) })$ is invariant with $\{\ths_{i}\}_{i\in[n]}$ replaced by $\{\ths_{i}+\bv\}_{i\in[n]}$ for any $\bv\in\RR^{d}$, we assume $\sum_{i\in[n]}\ths_{i}=\bm{0}$ for the identifiability, and we let $\Theta=\{\btheta\in \RR^{nd}\given\sum_{i\in[n]}\btheta_{i}=\bm{0}\}$ represent the parameter space, where $\btheta=(\btheta_{1}^{\top},\btheta_{2}^{\top},\ldots,\btheta_{n}^{\top})^{\top}$.


\section{Two-Stage Ranking Bandit Method}\label{sec:BT}
We begin by generating an ER graph $\cG=(\mathcal{V}, \mathcal{E})$, and then select comparison pairs from its edges. Let $\hat{\btheta}(0)$ be an initial estimator for $\btheta^*$, which can be an estimator based on  earlier samples or prior knowledge.
For each pair $(i,j)$, define $\cT_{t,ij}$ as the collection of time points at which the pair $(i,j)$ is sampled by the end of time $t$ that $\cT_{t,ij} = \{\ell\in[t]:\ba_\ell=(i,j)\}$, where $\{\ba_\ell=(i,j)\}$  if $(i_\ell,j_\ell)=(i,j)$ or $(i_\ell,j_\ell)=(j,i)$.
To efficiently exploit, we estimate the ability vector $\hat{\btheta}(t)$ using the data collected up to time $t$.
In particular, the negative log-likelihood function at time point $t$ is
\begin{align}\label{lik}
\cL_{t}(\btheta)=&\sum_{(i,j)\in \cE,i>j;}\sum_{\ell\in \cT_{t,ij};}\sum_{m=1}^{M}\{-y_{ji}^{(m)}(\ell)\bX_{\ell}^\top(\btheta_{i}-\btheta_{j}) +\log(1+e^{\bX_{\ell}^\top(\btheta_{i}-\btheta_{j}) })\}.
\end{align}

Our algorithm consists of two stages: the $\epsilon$-greedy stage and the exploitation stage. The $\epsilon$-greedy stage is from the beginning to some pre-specified time $T_0$, where we mix with both exploration and exploitation. Then we transit to the pure exploitation stage after $T_0$. A sufficiently large $T_0$ guarantees a good initial value for full exploitation. We provide the choice of $T_0$ in Theorem~\ref{thm:reg}, and discuss how to choose $T_0$ in Remark~\ref{rm:RT}. 

Specifically, at each time point $t\in[T_0]$, letting $\alpha>0$ be a constant, we explore with probability $1/t^\alpha$ that we randomly select a pair $(i,j)\in \cE$ for comparison. Meanwhile, with probability $1-1/t^\alpha$, we exploit the historical information $\cH_{t-1}$ that we choose
\begin{align}\label{it}
i_t=\argmax_{i\in[n]} \bX_{t}^\top\hat{\btheta}_{i}(t-1).
\end{align}
Then, among the items connected to $i_t$, we randomly select  item $j_t$ for comparison. After observing $y_{i_t,j_t}(t)$, we update our estimator $\hat{\btheta}(t)$ by minimizing the negative log-likelihood function that 
\begin{align}\label{opt1}
\hat{\btheta}(t)=\argmin_{\btheta\in\tilde{\Theta}}\cL_{t}(\btheta),
\end{align}
where $\tilde{\Theta}=\{\btheta\in\Theta:\|\btheta_i\|_2\leq \gamma\}$, and $\gamma$ is a constant. 

In the second stage, from $T_0+1$ to the end of horizon $T$, we  exploit at every step. In particular, at each time point $t$, we select $i_t$ by (\ref{it}), and let $j_t$ be a random item connected to $i_t$ in $\cG$. Performing action $(i_t,j_t)$, we observe the result $y_{i_t,j_t}(t)$. We then update the estimator using the regularized maximum likelihood estimator $\hat{\btheta}$ that
\begin{align}\label{opt}
\hat{\btheta}(t)=\argmin_{\btheta}\cL_{\lambda,t}(\btheta) =\argmin_{\btheta} \cL_{t}(\btheta)+\lambda_t/2\|\btheta\|^2_2.
\end{align}
 We point out that the estimators in the two stages are slightly different for the ease of presentation of the technical proofs. In practice, we observe that using either estimator in the two stages gives similar empirical performances.
We summarize the methods in these two stages in the algorithm below. We refer to our method as the Ranking Bandit (RB).

\begin{algorithm}[H]
\caption{Two-Stage Ranking Bandit Method}
\begin{algorithmic}\label{alg}
\STATE \textbf{Result:} Actions $\ba_1, \ba_2, \ldots, \ba_T$.
\STATE \textbf{Input:} Number of items $n$, graph parameter $p$, number of comparisons per step $M$, exploration parameter $\alpha$ and optimization parameter $\gamma$.
\STATE \textbf{Initialization:} Initialize $\hat{\btheta}(0)$ randomly. Generate an ER graph $\cG_{n,p}$.
\FOR{$t$ from $1$ to $T_0$}
\STATE Observe covariate $\bX_{t}$. With probability $1/t^\alpha$, randomly select an edge $(i_t,j_t)$ from $\cG$. Otherwise, set $i_t=\argmax_{i\in[n]} \bX_{t}^\top\hat{\btheta}_{i}(t-1) $, and uniformly sample $j_t$ from the items connected to $i_t$ in $\cG$. 
\STATE Perform action $\ba_t=(i_t,j_t)$ and observe pairwise comparison results $y_{i_t,j_t}(t)$. 
\STATE Update the estimate as $\hat{\btheta}(t)=\argmin_{\btheta\in\tilde{\Theta}}\cL_{t}(\btheta),$ where $\tilde{\Theta}=\{\btheta\in\Theta:\|\btheta_i\|_2\leq \gamma\}$.
\ENDFOR
\FOR{$t$ from $T_0+1$ to $T$}
\STATE Observe covariate $\bX_{t}$. Set $i_t=\argmax_{i\in[n]} \bX_{t}^\top\hat{\btheta}_{i}(t-1) $, and uniformly sample $j_t$ from the items connected to $i_t$ in $\cG$.
\STATE Perform action $\ba_t=(i_t,j_t)$ and observe pairwise comparison results $y_{i_t,j_t}(t)$. 
\STATE Update the estimate as $\hat{\btheta}(t)=\argmin_{\btheta}\cL_{\lambda,t}(\btheta)$.
\ENDFOR
\end{algorithmic}
\end{algorithm}


\section{Theoretical Results}\label{sec:thry}
We provide theoretical guarantees of the proposed method.  In Section~\ref{sec:rate}, we establish the rates of convergence of our estimators in the two stages. Section~\ref{sec:reg} presents the regret analysis. We provide the statistical inference results for the estimators in Section~\ref{sec:inf}.

\subsection{Statistical Rates}\label{sec:rate}
We derive the rates of convergence of the estimators in the two stages, which is fundamental to regret analysis. Before going further, we impose some mild assumptions to facilitate our discussions. 
\begin{assumption}\label{ass:Sigx}
There exist positive constants $c$ and $C$, such that the minimum eigenvalue of $\EE(\bX\bX^\top)$ satisfies $\lambda_{\min}(\EE(\bX\bX^\top))>c$, and $\|\bX\|_2\leq C$.
\end{assumption}
\begin{assumption}\label{ass:Hesx}
There exists a constant $c>0$ such that $\max_{i\in[n]}\tilde{\bSigma}_{i}/\min_{i\in[n]}\lambda_{\min}(\bSigma_{i}^*)\leq c$, where $\bSigma_{i}^*=\EE[\bm{1}(\argmax_j(\bX^\top\btheta_j^{*})=i)\bX\bX^\top]$ and $\tilde{\bSigma}_{i}=\EE[\bm{1}(\argmax_j(\bX^\top\btheta_j^{*})=i)\bX^\top\bX]$.
\end{assumption}
\begin{assumption}\label{assmp:kappa}
There exists a constant $C>0$ such that $\max_{1\leq i< j\leq n}\|\btheta^{*}_{i}-\btheta^{*}_{j}\|_{2}\leq C$.
\end{assumption}
\begin{assumption}\label{ass:dif}
There exists a constant $c>0$ such that  $\max_{a=\pm\delta/2}\PP(|\max_{j\in[n]\setminus\{i\}}\bX^{\top}(\ths_{j}-\ths_{i})-a|\leq \delta)\leq c\delta$ holds for all $\delta>0$ and  $i\in[n]$.
\end{assumption}

Assumption~\ref{ass:Sigx} ensures the well-behaved nature of the Hessian matrix derived from the contextual variable $\bX$. The condition $\lambda_{\min}(\EE(\bX\bX^\top))>c$ assumes that the minimum eigenvalue of the covariance of the contextual vector $\bX$ is bounded away from zero, which is a common assumption for covariates \citep{han2020sequential,chen21,wang2023efficient}.
The condition $\|\bX\|_2\leq C$ bounds the magnitude of the contextual variable. We note that this condition can be relaxed to hold with high probability under some distributional assumptions. For example, it holds when $\bX$ follows a sub-Gaussian distribution. Assumption~\ref{ass:Hesx} guarantees that no item dominates disproportionately or contributes negligibly, guaranteeing well-behaved Hessian matrices.
\re{It imposes covariate diversity and is standard in contextual bandit problems; see, for example, Assumption EC.1 in \citet{bastaniOR} and Assumptions 3 and 4 in \citet{MSbandit}. This assumption holds if the probability density of $\bX$ is bounded below by a positive constant in an open neighborhood around the origin, covering common distributions such as uniform or truncated Gaussian. For a detailed illustration, we refer to \citet{MSbandit}.}
This assumption is necessary since the selected arms depends on the covariate. In the BTL literature \citep{Negahban2017,chen2019spectral, Gao2021}, where covariates are absent, and even in studies incorporating covariate information, such as \citet{fan2024uncertaintyquantificationmleentity}, where the relationship between covariates and items is known in advance, we do not need this assumption.
We also note that while prior literature requires that the sampling frequencies for all item pairs are of the same order, we do not make this assumption. Instead, we emphasize that Assumption~\ref{ass:Hesx} is a more essential condition for BTL models in the online-learning setting.
Assumption~\ref{assmp:kappa} assumes that the abilities of different entities lie within a bounded range, which is a common assumption in the BTL model \citep{chen2019spectral,Gao2021,FANmultiway,fan2024spectralrankinginferencesbased}. 
\re{Assumption~\ref{ass:dif} is a common margin condition in the contextual bandit literature, for example, Assumption~2 in \citet{hu2022fast} and Assumption~4 in \citet{hu2020smooth}.}  This assumption helps control the divergence between the observed covariate matrix with the expected one, as detailed in Lemma~C.1 in the Appendix. Assumption 4.4 is equivalent to $\PP(0< \bX^\top\btheta^{*}_{i}-\max_{j\in[n]\setminus\{i\}}\bX^\top\btheta^{*}_j\leq \delta)\leq c\delta$ and $\PP(0<\max_{j\in[n]\setminus\{i\}}\bX^\top\btheta^{*}_j-\bX^\top\btheta^{*}_i\leq \delta)\leq c\delta$ holding for $\delta>0$ and $i\in[n]$. It regulates the score gap between the optimal and sub-optimal items. By leveraging the anti-concentration inequality 
\citep{chernozhukov2015comparison}, a stronger result, $\sup_{a\in\RR}\PP(|\max_{j\in[n]\setminus\{i\}}\bX^{\top} (\ths_{j}-\ths_{i})-a|\leq \delta)\leq c\delta$, can be derived. Specifically, we deduce the explicit value of $c$ for the Gaussian distribution as follows. If $\bX\sim N(0,\bSigma)$, let $\sigma_j^2=(\ths_{j}-\ths_{i})^\top\bSigma(\ths_{j}-\ths_{i})>0$ for $j\in[n]\setminus\{i\}$, $\underline{\sigma}=\min \sigma_j$ and $\overline{\sigma}=\max\sigma_j$.
Then, for any $\delta>0$, we have 
\begin{align*}
\sup_{a\in\RR}\PP\Big(\Big|\max_{j\in[n]\setminus\{i\}}\bX^{\top}(\ths_{j}-\ths_{i})-a\Big|\leq \delta\Big)\leq c\delta\big(a_{n}+\sqrt{1\vee \log(\underline{\sigma}/\delta)}\big)
\end{align*}
using anti-concentration inequality, where $a_{n}=\EE[\max(\bX^{\top}(\ths_{j}-\ths_{i})/\sigma_j)]\leq \sqrt{2\log n}$ in the worst case, and $c$ depends only on $\underline{\sigma}$ and $\overline{\sigma}$.

The next theorem provides the convergence rate of our estimator in the exploration stage.

\begin{theorem}[$\epsilon$-Greedy stage rate]\label{thm:2rate}
Suppose that Assumption~\ref{ass:Sigx} holds. For the estimator $\hat{\btheta}(t)$ defined in (\ref{opt1}), if $\gamma$ is large enough such that $\ths\in\tilde{\Theta}$, there exists a constant $c > 0$ such that for any $t\in[(cn^2p\log T)^{1/(1-\alpha)},T_0]$ and $np>c\log n$, we have
\begin{align*}
\|\hat{\btheta}(t)-\btheta^*\|_2\lesssim  \sqrt{\frac{\log T}{M}}nt^{\alpha-1/2}
\end{align*}
holds with probability $1-O(\max\{T^{-3},n^{-10}\})$.
\end{theorem}
\begin{remark}
The lower bound $(cn^2p\log T)^{1/(1-\alpha)}$  ensures that there is at least one observation for each pair in the initial burn-in period, which is a necessary requirement to make all parameters identifiable. Additionally, the condition $np>c \log n$ is crucial to ensure the comparison graph is connected with high probability.
Theorem~\ref{thm:2rate} shows that the order of estimation error in the first stage is $O(t^{\alpha-1/2})$. Notably, if we set $\alpha=0$, i.e., we have i.i.d. samples, and the total number of comparisons $tM=Ln(n-1)p/2$, where $n(n-1)p/2$ represents the expected number of item pairs in an ER graph and $L$ denotes the average number of comparisons per item pair, the statistical rate reduces to $O_P(\sqrt{{1}/({pL}}))$, omitting the logarithm term. This rate matches the results for the simple BTL model without contextual information, as shown in \citet{Negahban2017}, \citet{chen2019spectral} and \citet{Gao2021}.
\end{remark}

Suppose we set  $\alpha\in(0,1/2)$, and 
\begin{align}\label{tT0}
\tilde T_0=\max\Big\{(n^2p)^{1/\alpha+1}(\log T)^{1/(1-\alpha)},
n^{6/(1-2\alpha)}(\log T)^{1/(1-2\alpha)}/(pM)\Big\},
\end{align} for the $\epsilon$-greedy stage. We have the following result for the exploitation stage.
\begin{theorem}[Exploitation stage rate]\label{thm:BT}
Suppose that Assumptions~\ref{ass:Sigx}-\ref{ass:dif} hold, and $T_0\geq C \tilde T_0$, $ np>C \log n$ for some sufficiently large constant $C$. If  $\lambda_t\asymp \sqrt{tM\log T/n}$, for the estimator $\hat{\btheta}(t)$ defined in (\ref{opt}), we have
\begin{align*}
\max_{i\in [n]}\|\hat{\btheta}_{ i}(t)-\btheta^{*}_{i}\|_{2}\lesssim 
\sqrt{\frac{n\log T}{ptM}}
\end{align*} holds with probability $1-O(\max\{T^{-2},n^{-10}\})$ for $t\geq T_0$.
\end{theorem}
\begin{proof}[Proof Sketch]
We outline the proof here and provide the detailed proof in Appendix Section~C.

We first have that the log-likelihood function (\ref{lik}) is equivalent to the following form where we explicitly write  $\ba_{\ell}$ that
\begin{align}\label{lika}
\cL_{t}(\btheta)=&\sum_{(i,j)\in \cE,i>j;}\sum_{\ell\in [t];}\sum_{m=1}^{M}\bm{1}(\ba_{\ell}=(i,j))\{-y_{ji}^{(m)}(\ell)\bX_{\ell}^\top(\btheta_{i}-\btheta_{j}) +\log(1+e^{\bX_{\ell}^\top(\btheta_{i}-\btheta_{j}) })\}.
\end{align}
Notice that our loss function utilizes the whole trajectory including both the $\epsilon$-greedy stage and the exploitation stage to boost the estimation accuracy. Our next step is to bound the estimation rate of estimators derived from the loss with the dependent observations.
For all exploitation steps in both stages, we observe the preference after  action~$\ba_{\ell}$, which depends on history $\cH_{\ell-1}$. This dependency introduces correlations among the summations in $\cL_t$, posing challenges for theoretical analysis.
A critical step in the proof is bounding the eigenvalues of the Hessian matrix involving dependent samples. Note that the Hessian matrix is 
\begin{align*}
\nabla^{2} \cL_t(\btheta)=&M\sum_{(i,j)\in \mathcal{E},i>j;}\sum_{\ell\in [t]}\bm{1}(\ba_\ell=(i,j))\frac{e^{\bX_{\ell }^\top \btheta_{i}} e^{\bX_{\ell }^\top \btheta_{j}}}{(e^{\bX_{\ell }^\top \btheta_{i}}+e^{\bX_{\ell }^\top \btheta_{j}})^{2}}(\bc_{\ell i}-\bc_{\ell j})(\bc_{\ell i}-\bc_{\ell j})^\top,
\end{align*}
where $\bc_{\ell i} = \be_{i}\otimes \bX_\ell \in\mathbb{R}^{nd}$; $\ba_\ell = (i_\ell,j_\ell)$, and $i_\ell=\argmax_{i\in[n]} \bX_{\ell}^\top\hat{\btheta}_{i}(\ell-1) $. 
We begin with considering $\sum_{\ell=1}^t\bm{1}(i_\ell=i)\bX_\ell\bX_\ell^\top$ and $\sum_{\ell=1}^t\bm{1}(i_\ell=i)\bX_\ell^\top\bX_\ell$. Specifically, we show that the  event 
{\small
\begin{align}\label{Hesbd}
\Big\{t\min_{i\in[n]}\lambda_{\min}(\bSigma_{i}^*)/2\leq \lambda_{\min}(\sum_{\ell=1}^t\bm{1}(i_\ell=i)\bX_\ell\bX_\ell^\top)\leq \sum_{\ell=1}^t\bm{1}(i_\ell=i)\bX_\ell^\top\bX_\ell\leq 3t\max_{i\in[n]}\tilde{\bSigma}_{i}/2,i\in[n]\Big\}
\end{align}}
holds with high probability. For $\sum_{\ell=1}^t\bm{1}(i_\ell=i)\bX_\ell\bX_\ell^\top$, in Lemma~C.1, we control the term induced by historical estimation error, 
\[
\|\EE_{\bX}[\bm{1}(\argmax_j \bX^\top\hth_j(t-1)=i)\bX\bX^\top]-\bSigma_i^*\|\lesssim\Big(\max_{j\in[n]}\|\hth_j(t-1)-\ths_j\|_2\Big)^{1/2}
\] 
by employing anti-concentration inequalities. Additionally, we show in Lemma~C.2 that the error 
\[
    \Big\|\sum_{\ell=1}^t\bm{1}(i_\ell=i)\bX_\ell\bX_\ell^\top-\EE_{\bX}[\bm{1}(\argmax_j \bX^\top\btheta_j =i)\bX\bX^\top]\Big\|\lesssim\sqrt{t\log T}
\]
is negligible using matrix martingale concentration. By combining these results with a sufficiently good  initialization of the exploitation stage and bounds on historical estimation error, we prove in Lemma~C.3 that the event in (\ref{Hesbd}) holds with high probability.

Then, we bound the eigenvalues of the Hessian matrix conditioning on the above event. Leveraging the randomness of $j_\ell$ and utilizing graph properties, we establish in Lemma~C.5 that
\begin{align}\label{Hes}
tM\min_{i\in[n]}\lambda_{\min}(\bSigma_{i}^*)\lesssim
\lambda_{\min, \perp}\left(\nabla^2 \cL_t(\btheta)\right) \leq \lambda_{\max }\left(\nabla^2 \mathcal{L}_{t}(\btheta)\right)\lesssim tM\max_{i\in[n]}\tilde{\bSigma}_{i}.\end{align} 

Next, we construct a gradient descent sequence $\{\btheta^{s}\}_{s=0,1,\ldots,t}$ and  bound $\|\btheta^{t}-\hat{\btheta}(t)\|_\infty$ and $\max_{i\in [n]}\|\hat{\btheta}_{i}(t)-\btheta_{i}^{*}\|_{2}$. 
Using the convergence properties of the gradient descent sequence, we show in Lemma~C.7 that
\begin{align*}
\|\btheta^{t}-\hat{\btheta}(t)\|_\infty \lesssim \sqrt{\frac{n\log T}{ptM}}.
\end{align*} 
Further, by applying the  matrix martingale concentration to each step of $\btheta^s$ together with the Hessian bounds in (\ref{Hes}), we prove in Lemma~C.8 that
\begin{align*}
\max_{i\in [n]}\|\btheta_{i}^{t}-\btheta_{i}^{*}\|_{2}
\lesssim \sqrt{\frac{n\log T}{ptM}}.
\end{align*}
Combining the two bounds above, we conclude the proof.
\end{proof}
\begin{remark}
Theorem~\ref{thm:BT} establishes the uniform convergence rate for all items' abilities. Let $L$ be the average number of comparisons per item pair, satisfying $tM=Ln(n-1)p/2$ to facilitate intuitive comparison with existing results in the BTL model literature. Our rate simplifies to $\sqrt{\frac{\log T}{np^2L}}$. \citet{fan2024uncertaintyquantificationmleentity} analyze the BTL model with item abilities represented as $\balpha^*_i+\bX_i^\top\bbeta^*$, where $\balpha^*_i\in\RR$ is the individual attribute, $\bX_i\in\RR^d$ is the individual feature, and $\bbeta^*\in\RR^d$ is a common parameter across items. Their work establishes convergence rates based on independent observations, with $\|\hat{\balpha}_i-\balpha_i^*\|_\infty\lesssim\sqrt{\frac{\log n}{npL}}$ and $\|\hat{\bbeta}-\bbeta^*\|_2\lesssim\sqrt{\frac{\log n}{pL}}$. Compared with these results, our rate has the same order in terms of $n$ and $L$, ignoring the logarithmic term. The additional factor of $p$ arises due to varying contextual information across comparison pairs.
The term $T$ appearing in the expression for $T_0$ is introduced for theoretical analysis. However, due to the logarithmic relationship, the order of $T$ can grow as large as $\exp(T_0)$. This means that the appearance of $T$ in $T_0$ has minimal effect on practical usage.
\end{remark}

\subsection{Regret Analysis}\label{sec:reg}
We present a non-asymptotic bound on the regret incurred by our strategy. This result demonstrates that the average cumulative regret diminishes at a rate of $O(T^{-1/2})$. Our analysis applies to the implementation of our algorithm with no historical data, where the initial value $\hth(0)$ is set randomly.
A key step in the regret analysis is bounding the  error $\|\btheta_{i_t^*}^{*}-\btheta_{i_t}^{*}\|_2$, which can be controlled by $\max_{i\in[n]}\|\hth_i(t)-\ths_i\|_2$. When the order of $T$ is larger than $T_0$, we have the following regret bound.
\begin{theorem}\label{thm:reg}
Under the conditions of Theorems~\ref{thm:2rate} and \ref{thm:BT}, for $T\gg T_0$, we have
\begin{align*}
R(T)
\lesssim&\sqrt{\frac{n\log T}{pMT}}+\frac{T_0^{1-\alpha}}{T}
+ \sqrt{\frac{\log T}{M}}\frac{nT_0^{\alpha+1/2}}{T}
+\frac{1}{T}\Big(n^2p\log T\Big)^{1/(1-\alpha)}
\end{align*}
with probability at least $1-O(\max\{T^{-1},n^{-10}\})$.
\end{theorem}
\begin{remark}\label{rm:RT}
Theorem~\ref{thm:reg} guarantees that the average cumulative regret converges to zero at the rate of $O((\log T/T)^{1/2})$ as $T$ increases. The first term results from the second stage and is derived utilizing a uniform convergence rate for the ability estimations. The second and third terms are associated with the first stage, balancing the regret induced by exploration and exploitation steps. The last term arises from the initial phase of the strategy, representing the regret incurred to ensure that all pairs, approximately $n^2p$ pairs induced by the ER graph, are selected at least once. From Theorem~\ref{thm:reg}, $R(T)$ increases as $T_0$ increases. Therefore, we set $T_0$ as the same order of $\tilde{T}_0$, as defined in \eqref{tT0}. Considering the order with respect to $n$, $\alpha=(\sqrt{6}-2)/2$ minimizes the regret when $p$ is constant, while $\alpha=1/8$ if $p\asymp\log n/n$.

\end{remark}

\re{
\begin{remark}\label{re:M}
In classical bandit problems, $M$ is typically 1. In our framework, $M$ can be any positive integer, and we introduce it for generality. The parameter $M$ influences the regret bound by affecting the estimator's convergence rate. Specifically, the convergence rate for both stages scales as $M^{-1/2}$, as shown in Theorems~\ref{thm:2rate} and \ref{thm:BT}, reflecting the effect of $M$ on each exploitation step. This leads to the appearance of the terms $\sqrt{n \log T/(p M T)}$ and $\sqrt{\log T/M} \cdot n T_0^{\alpha + 1/2}/T$ in the regret bound. The parameter $M$ does not affect the terms arising from the exploration steps.
\end{remark}

\begin{remark}\label{re:L}
Recall that $L$ represents the average number of pulls per arm. Given that $TM = Ln(n-1)p/2$, the main term of the regret bound in our setting is of order $\sqrt{\log T/(nL)}$ for constant $p$. This differs from the classical bandit problem, where the optimal regret bound typically scales as $1/\sqrt{L}$ \citep{auer1995gambling,audibert2009minimax,DSL}. The difference arises from the nature of the ranking problem, where the number of compared items helps improve the estimation of a single item's ability \citep{simons1999asymptotics}, highlighting the distinction between ranking bandit problems and traditional bandit settings.
\end{remark}
}

\subsection{Asymptotic Normality}\label{sec:inf}
We derive the asymptotic distribution of the estimator $\hth(t)$, which facilitates inference. Following the Lagrangian debiasing approach developed in \cite{lagran23}, we first construct a debiased estimator $\hat{\btheta}^{d}(t)$ based on $\hat{\btheta}(t)$, and then establish its asymptotic normality.

Define the constraint function $f_t(\btheta)=\big(f_{t1}(\btheta),f_{t2}(\btheta),\ldots,f_{td}(\btheta)\big)^\top\in\RR^d$, where $f_{tk}(\btheta)=tM\sum_{i\in[n]}\btheta_{ik}$, $k\in[d]$. The MLE method solves the problem: $
\min_{\btheta}\cL_{t}(\btheta),\text{ subject to }f_t(\btheta)=0.$
The corresponding Lagrangian dual problem is 
\begin{align}\label{Lagdual}
\max_{\mu_t}\min_{\btheta}\cL_{t}(\btheta)+\mu^\top_t f_t(\btheta),
\end{align}
where $\mu_t\in\RR^d$ is the Lagrangian multiplier.
Using the estimator $\hth(t)$, we derive the debiased estimator $\hth^{d}(t)$ by solving the following equations for $\btheta$ and $\mu_t$, which represent the first-order approximation of the optimality condition derived from (\ref{Lagdual}) that
\begin{align}\label{deb}
\left(\begin{array}{cc}
\nabla^2 \cL_t(\hth(t)) & \nabla f_t(\hth(t)) \\
\nabla f_t(\hth(t))^{\top} & \bm{0}
\end{array}\right)\binom{\btheta-\hth(t)}{\mu_t}=\binom{-\nabla \cL_t(\hth(t))}{-f_t(\hth(t))}.
\end{align}
Define the matrix
\begin{align*}
\bPsi_t(\btheta)=\left(\begin{array}{cc}
\nabla^2 \cL_t(\btheta) & \nabla f_t(\btheta) \\
\nabla f_t(\btheta)^{\top} & \bm{0}
\end{array}\right),
\end{align*}
and let $\hat{\bPsi}_t=\bPsi_t(\hth(t))$, $\bPsi^*_t=\bPsi_t(\ths)$ for simplicity. We prove that $\bPsi^*_t$ and $\hat{\bPsi}_t$ are invertible with high probability in Lemma~E.1 in the Appendix. Let inverse of $\bPsi_t^{*}$ be
\begin{align*}
\bPhi^*_t=\left(\begin{array}{cc}
\bPhi^*_{t,11} & \bPhi^*_{t,12} \\
\bPhi^*_{t,21} & \bPhi^*_{t,22}
\end{array}\right),
\end{align*}
where $\bPhi^*_{t,11}$ is the $nd\times nd$ block of the $(n+1)d\times (n+1)d$ matrix $\bPhi^*_t$. Similarly, let $\hat{\bPhi}_t$ and $\hat{\bPhi}_{t,11}$ denote the counterparts derived from $\hat{\bPsi}_t$.
The debiased estimator from (\ref{deb}) is then: 
\begin{align*}
\hat{\btheta}^{d}(t)=\hat{\btheta}(t)-\hat{\bPhi}_{t,11}\nabla\cL_t(\hat{\btheta}(t)).
\end{align*}
The following theorem establishes the asymptotic properties of $\hat{\btheta}^{d}(t)$.

\begin{theorem}\label{thm:inf}
Under the conditions of Theorem~\ref{thm:BT}, as $t,n\rightarrow\infty$, we have
\begin{align*}
((\bPhi_{t,11}^*)_{ii})^{-1/2}(\hat{\btheta}_{i}^{d}(t)-\btheta_{i}^{*})\overset{d}\longrightarrow N(0,\Ib_{d})
\end{align*}
for $i\in[n]$, where $(\bPhi_{t,11}^*)_{ii}$ denotes the $i$-th $d\times d$ diagonal block of $\bPhi_{t,11}^*$.
\end{theorem}

\re{
We present the following corollary that establishes the validity of asymptotic calibration when the estimators $\hat{\btheta}$ are plugged in. This result provides a practical foundation for constructing valid inference procedures.
\begin{corollary}\label{cor:inf}
Under the conditions of Theorem~\ref{thm:BT}, as $t,n\rightarrow\infty$, we have
\begin{align*}
((\hat\bPhi_{t,11})_{ii})^{-1/2}(\hat{\btheta}_{i}^{d}(t)-\btheta_{i}^{*})\overset{d}\longrightarrow N(0,\Ib_{d})
\end{align*}
for $i\in[n]$, where $(\hat\bPhi_{t,11})_{ii}$ denotes the $i$-th $d\times d$ diagonal block of $\hat\bPhi_{t,11}$.
\end{corollary}
}


\section{Numerical Studies}\label{sec:simu}
We conduct extensive numerical experiments to assess the performance of the proposed algorithm. To the best of our knowledge, there is no existing reinforcement learning method tailored for the contextual BTL model, so we put forward several methods for comparison to validate the effectiveness of our approach. 

We consider four methods for comparison of our Rank Bandit (RB) algorithm. The first method, referred to as ERA (Explore-Always), maintains exploration at all times. Specifically, at each time step $t\in [T]$, we randomly choose an edge $(i_t,j_t)$ with probability $1/t^\alpha$; otherwise, we set $i_t=\argmax_{i\in[n]} \bX_{t}^\top\hat{\btheta}_{i}(t-1) $, and let $j_t$ be a random item connected to $i_t$ in $\cG$ sampled uniformly.
The second method, denoted as ETA (Exploit-Always), performs action $(i_t,j_t)$ for each time $t$, where $i_t=\argmax_{i\in[n]} \bX_{t}^\top\hat{\btheta}_{i}(t-1) $, and $j_t$ is a random item connected to $i_t$ sampled uniformly. This approach exclusively exploits the current estimation without incorporating exploration.
The third method applies the Upper Confidence Bound (UCB) strategy, denoted as BT-UCB1. In particular, we apply  Theorem~\ref{thm:inf} to obtain the confidence interval for UCB.  At each time $t$, we let $i_{t,\text{UCB}}=\argmax_{k\in[n]} \bX^\top\hat{\btheta}_k^{d}(t)+(\bX^\top (\hat{\bPhi}_{t,11})_{kk}\bX)^{1/2}\sqrt{2\log (nT)}$, and let $j_{t,\text{UCB}}$ be a random item connected to $i_{t,\text{UCB}}$ in $\cG$ sampled uniformly. The fourth method adds an exploration-mixing stage before UCB, denoted as BT-UCB2.
 Specifically, from $1$ to $T_0$, we randomly select an edge $(i_t,j_t)$  with probability $1/t^\alpha$; otherwise, we perform $(i_{t,\text{UCB}},j_{t,\text{UCB}})$. From $T_0+1$ to $T$, we consistently choose $(i_{t,\text{UCB}},j_{t,\text{UCB}})$.

In our simulation, we let $d=10$, $T=500$, $T_0=200$, $\alpha=1/4$ and $\bX\sim N(\bm{1}_d,\diag(0.5,\ldots,0.5))$. We then consider two settings. In Setting A, we set $n=5$, $M=20$ and $p=0.3$, with $\ths_{i,2i-1}=\ths_{i,2i}=1$ for $i\in[n]$, while all other positions of $\ths$ are set to 0. The initial values of $\hth_{i}(0)$ are sampled from $N(i,5)$. 
For Setting B, we let $n=10$, $M=50$ and $p=0.15$, with $(\ths_{1},\ldots,\ths_{10})=\Ib_{10}$. All elements of $\hth(0)$ are initialized using $N(0,5)$. In both settings, we set $np=1.5$, resulting in a sparse comparison graph. The initialization $\hth(0)$ is random and far from the true  $\ths$. We evaluate the empirical performances by the logarithm of regret, $\log(R(t))$, and the logarithm of estimation error, $\log(\max_{i\in[n]}\|\ths_i-\hth_i\|_2/\max_{i\in[n]}\|\ths_i\|_2)$, as shown in Figure~\ref{fig:reg}.
We present the averaged regrets and errors over 200 repetitions.

\begin{figure*}[h]
\centering
\begin{tabular}{cc}
\quad\:\;(a) &\quad\:\:\;\;(b) \\
\includegraphics[width=5.5cm]{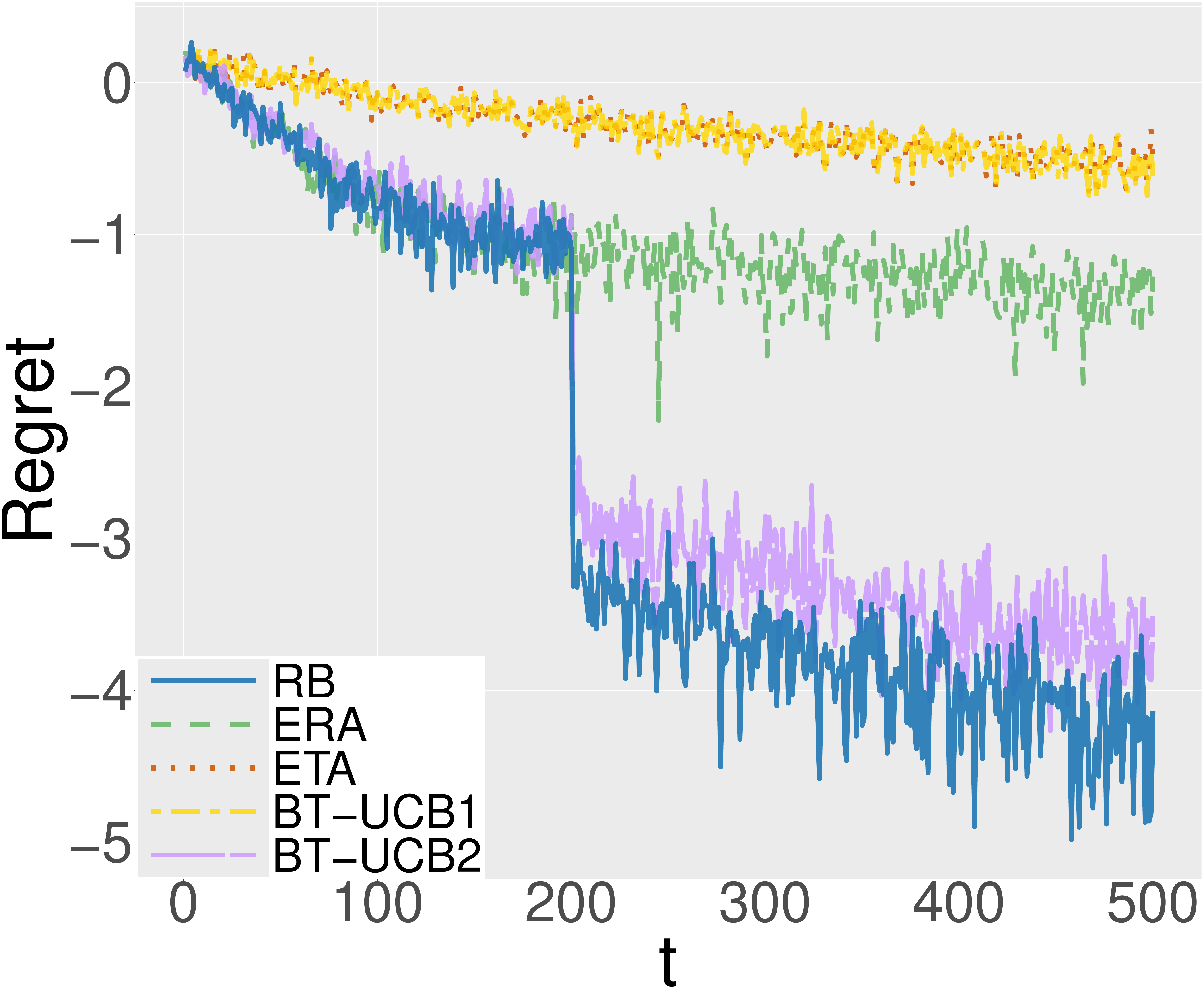}&
\includegraphics[width=5.5cm]{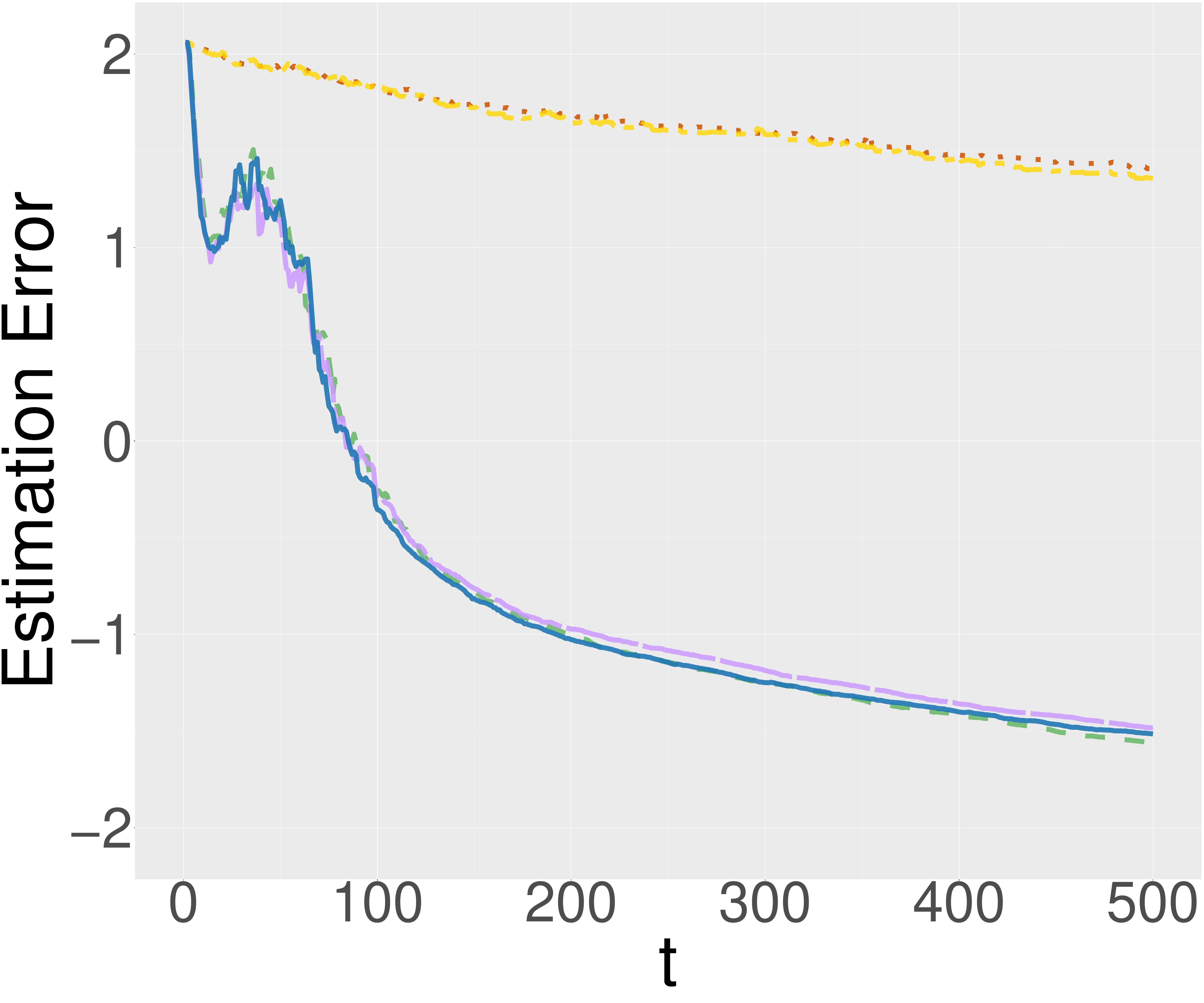}\\
\quad\:\;(c) &\quad\:\:\;\;(d) \\
\includegraphics[width=5.5cm]{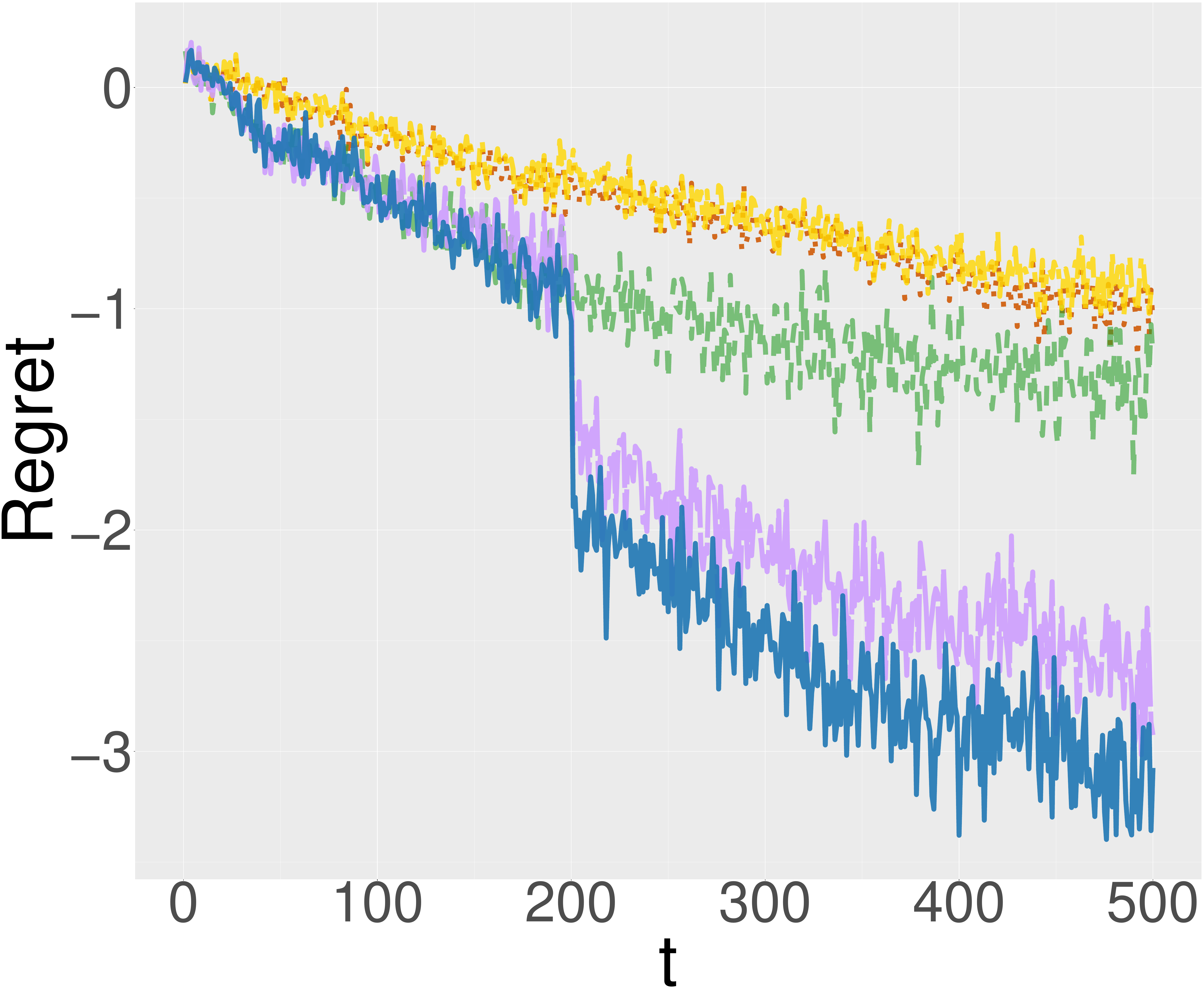}&
\includegraphics[width=5.5cm]{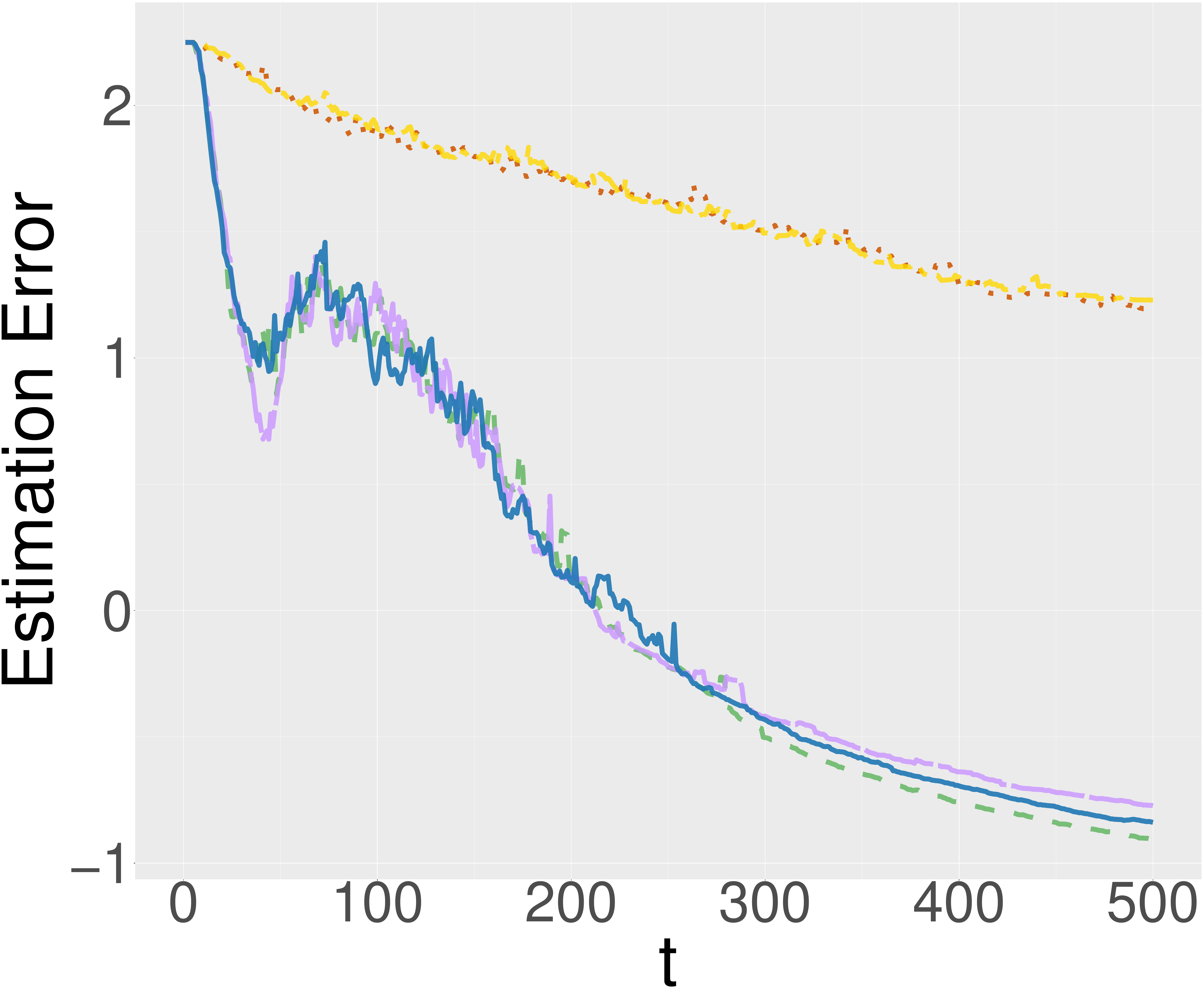}\\
\end{tabular}
\caption{Panels (a) and (b) show the regret and estimation error for Setting A, while panels (c) and (d) show the regret and estimation error for Setting B, with the y-axes shown on logarithmic scales.}
\label{fig:reg}
\end{figure*}

From Figure~\ref{fig:reg}, we observe that the proposed RB method  performs well in terms of both regret and estimation error. ERA achieves the lowest estimation error due to its uniform exploration throughout, which enables more precise estimation. However, this comes at the cost of significantly higher regret after $T_0$, as the strategy lacks sufficient exploitation. Meanwhile, without adequate exploration, ETA fails to accurately estimate item abilities, leading to poor performance in both regret and estimation error. We also observe that BT-UCB1 performs similarly to ETA. The unique characteristics of pairwise comparison data render the use of UCB estimation insufficient. As demonstrated in Theorem~\ref{thm:inf}, the confidence bound widths of different items are of the same order. Further, we observe that while BT-UCB2 incorporates exploration during the first stage, its performance still falls short of RB.

Next, we construct the confidence interval (CI) by plugging  the estimator $\hth(t)$ into $\bPhi_{t,11}^*$. According to Theorem~\ref{thm:inf}, the 95\% CI of $\btheta_{ik}^{*}$ is given by $[\hat{\btheta}_{ik}^{d}(t)+\sqrt{((\hat{\bPhi}_{t,11})_{ii})_{kk}} z_{0.025},\; \hat{\btheta}_{ik}^{d}(t)+\sqrt{((\hat{\bPhi}_{t,11})_{ii})_{kk}} z_{0.975}]$ for $i\in[n]$ and $k\in[d]$, where $z_\alpha$ is the $\alpha$-quantile of the standard normal distribution. To evaluate the performance, we provide the empirical coverage probabilities (CP) through 500 repetitions. 
We present the results in Figures~\ref{fig:CI-A} and \ref{fig:CI-B}. 
The black dots depict the average CP across dimensions, while the grey points indicate the minimum and maximum CP values across dimensions for each corresponding item. The dashed line represents the nominal value of~0.95. Initially, $\hat{\bPsi}_t$ is singular due to insufficient data, so we begin plotting from the points where $\hat{\bPsi}_t$ becomes invertible. From Figures~\ref{fig:CI-A} and \ref{fig:CI-B}, it can be observed that the empirical CPs converge to the nominal value of 0.95 as $t$ increases.
\begin{figure*}[h]
   \centering
\begin{tabular}{c}
\quad\quad Item 1\quad\quad\quad\quad\; Item 2\quad\quad\quad\quad\; Item 3\quad\quad\quad\quad\;\ Item 4\quad\quad\quad\quad\; Item 5 \\
\includegraphics[width=15cm]{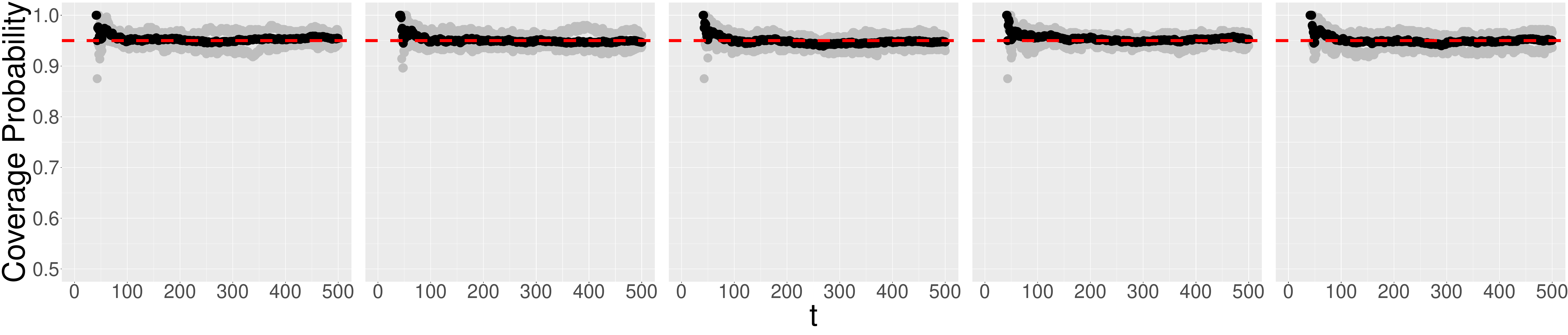}\\
\end{tabular}
    \caption{Coverage probability estimation under Setting A.}
    \label{fig:CI-A}
\end{figure*}

\begin{figure*}[h]
   \centering
\begin{tabular}{c}
\quad\quad Item 1\quad\quad\quad\quad\; Item 2\quad\quad\quad\quad\; Item 3\quad\quad\quad\quad\;\ Item 4\quad\quad\quad\quad\; Item 5 \\
\includegraphics[width=15cm]{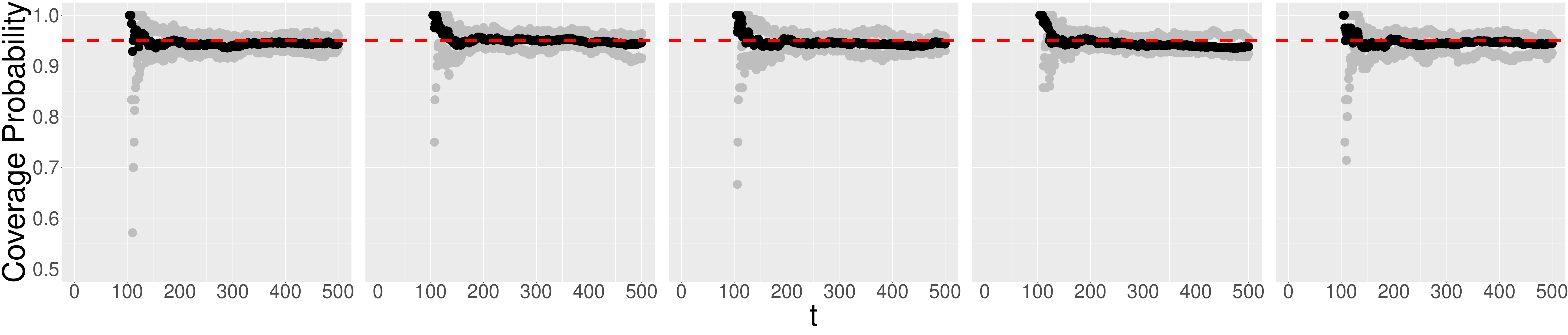}\\
\quad\quad Item 6\quad\quad\quad\quad\; Item 7\quad\quad\quad\quad\; Item 8\quad\quad\quad\quad\;\ Item 9\quad\quad\quad\quad\; Item 10\\
\includegraphics[width=15cm]{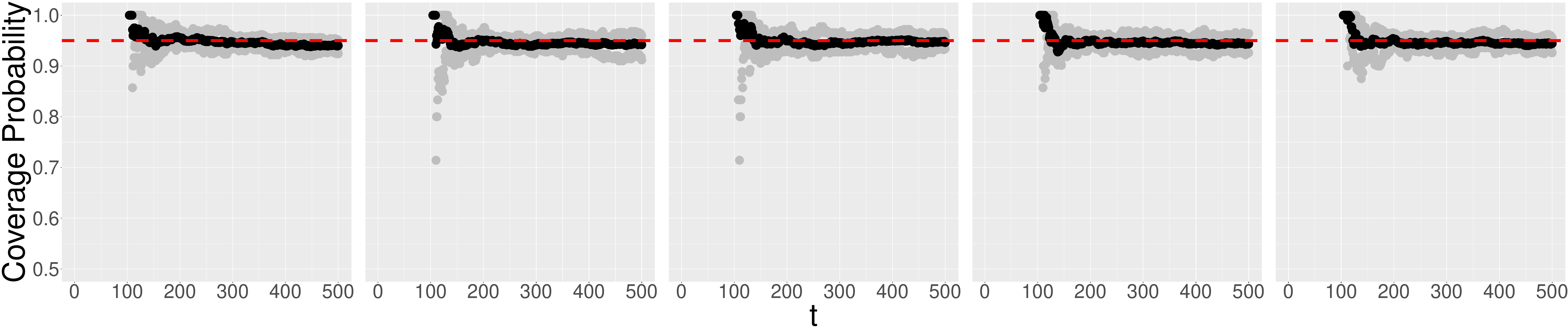}\\
\end{tabular}
    \caption{Coverage probability estimation under Setting B.}
    \label{fig:CI-B}
\end{figure*}


\section{Applications to Language Model Ranking}\label{sec:real}
We apply the method to rank LLMs on the Massive Multitask Language Understanding (MMLU) dataset \citep{hendrycks2021measuring}. We use the MMLU Anatomy dataset, which contains 100 questions. We take five language models into consideration: LLaMA-1 \citep{LLaMA-1}, LLaMA-2 \citep{touvron2023llama}, Alpaca \citep{alpaca}, GPT-3.5 and GPT-4o mini. Pairwise comparison results are generated using GPT-4.

Specifically, at each time  $t$, we randomly select a question as contextual information and apply our method to choose a pair of language models for comparison. The prompts are transformed into representation vectors using OpenAI’s text embedding model, text-embedding-3-small. We use principal component analysis to reduce the embeddings to 10-dimensional vectors, which are set as $\bX$. We set $T_0=3000$, $T=5000$, and initialize the elements of $\hth(0)$ using $N(0,5)$.
For evaluation purpose, we use all the available comparisons to obtain an estimation, which is set as~$\ths$. Figure~\ref{fig:real} displays the regret $R(t)$ and estimation error $\log(\max_{i\in[n]}\|\ths_i-\hth_i\|_2/\max_{i\in[n]}\|\ths_i\|_2)$, averaged over 50 repetitions.
From Figure~\ref{fig:real}, we observe that RB performs well in both regret and estimation errors. Although ERA has a smaller estimation error for large time steps, its lack of exploitation induces a large regret. Meanwhile, the other methods perform poorly in terms of both regret and estimation error.
\begin{figure*}[h]
\centering
\begin{tabular}{cc}
\quad\quad\;(a) &\quad\quad\:\:\;(b) \\
\includegraphics[width=5.5cm]{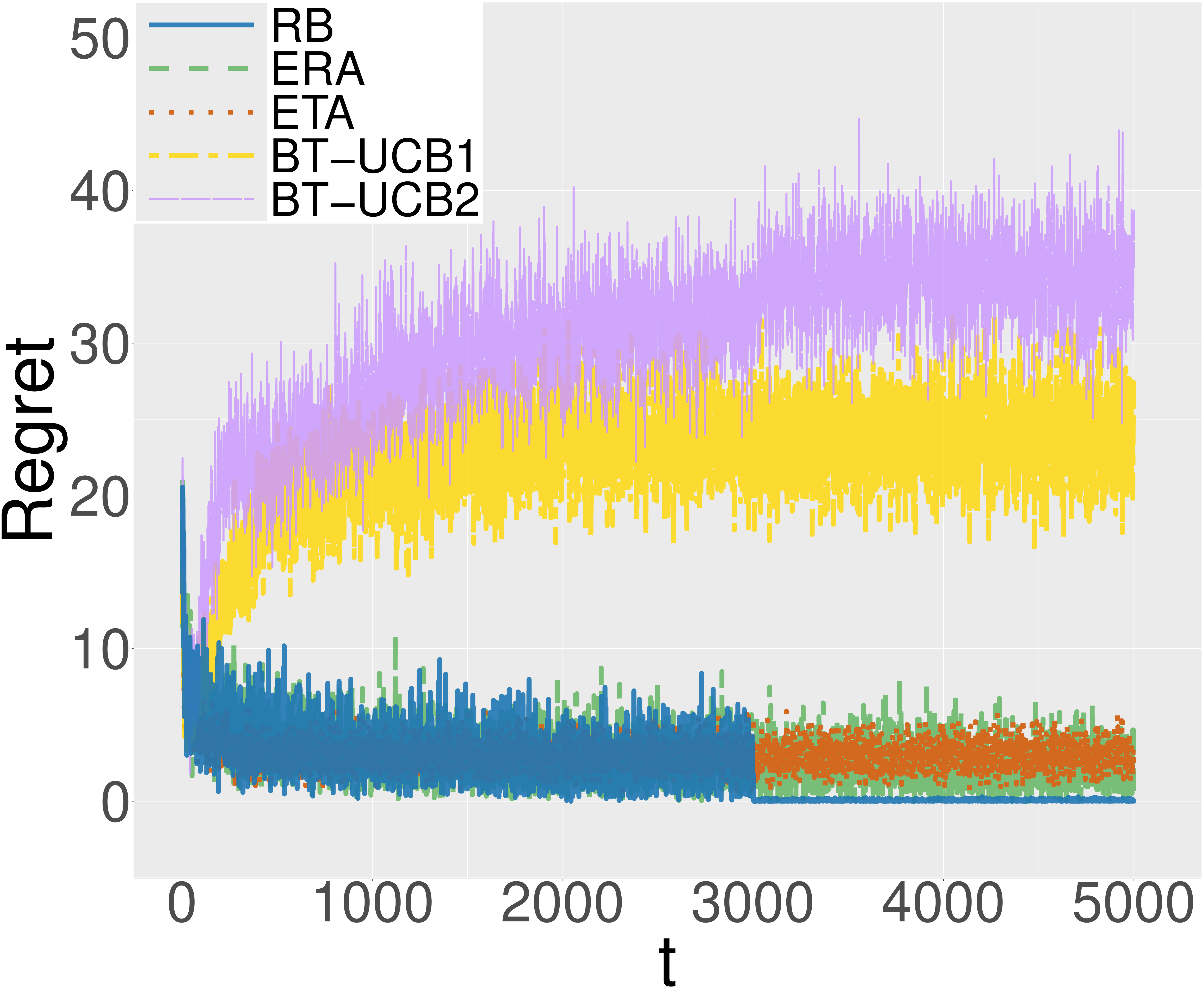}&
\includegraphics[width=5.5cm]{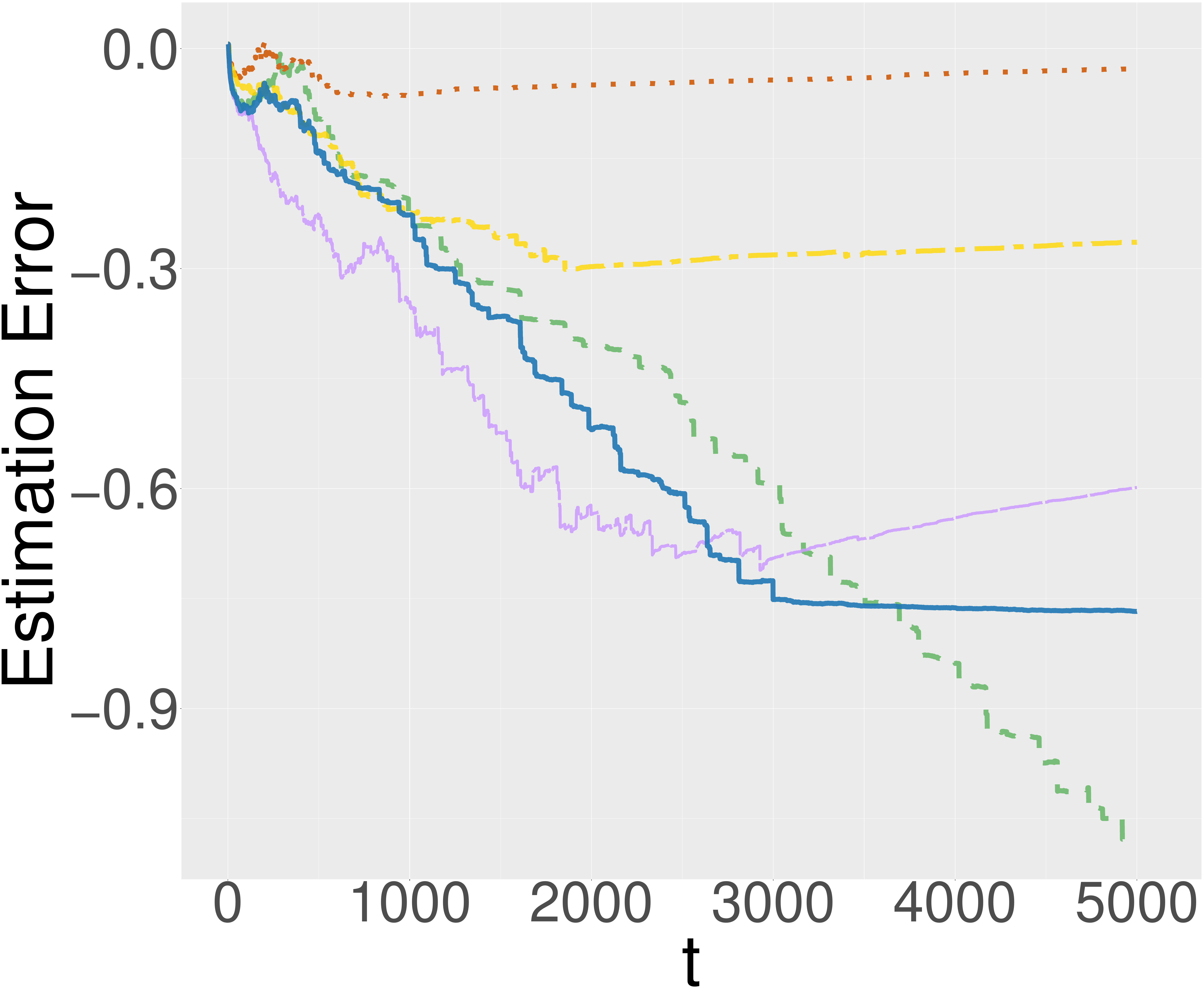}\\
\end{tabular}
\caption{Panels (a) and (b) show the regret and estimation error for the LLM ranking, with the y-axis in panel (b) shown on a logarithmic scale.}
\label{fig:real}
\end{figure*}

Next, we construct ranking diagrams to intuitively illustrate the changes in LLM rankings. Specifically, at each time step $t$, we compare every pair of ranked LLMs for all $\bX$ based on $\hth(t)$. If for more than $95\%$ contextual variables, model $i$ precedes model $j$, we add an edge from $i$ to $j$. To make the graph concise, we remove the edge from $i$ to $k$ whenever there are edges from $i$ to $j$ and $j$ to $k$. Figure~\ref{fig:rank-diag} shows the changes in the ranking results over time. As more data is collected, we observe that model GPT-4o mini surpasses the others, while GPT-3.5 and LLaMA-2 overtake Alpaca and LLaMA-1, which is not apparent at the earlier time points.
\begin{figure}[h]
\centering
\begin{tabular}{cccc}
$t=200$ & $t=400$ &  $t=600$ & $t=800$ \\
\includegraphics[width=3.5cm]{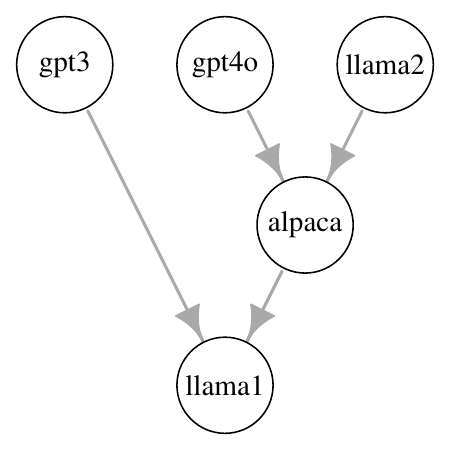}&
\includegraphics[width=3.5cm]{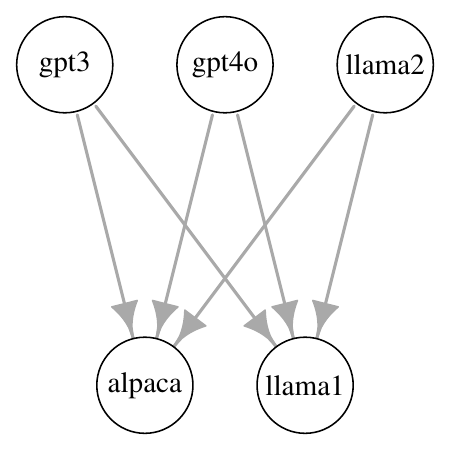}&
\includegraphics[width=3.5cm]{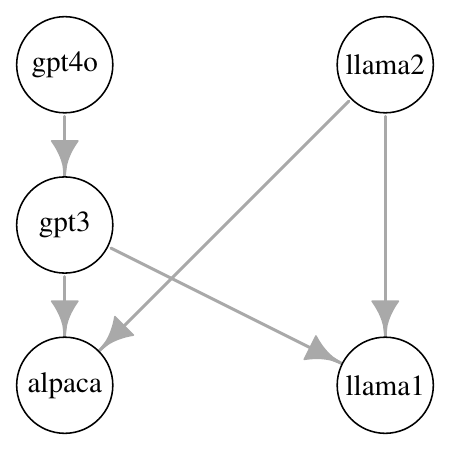}&
\includegraphics[width=3.5cm]{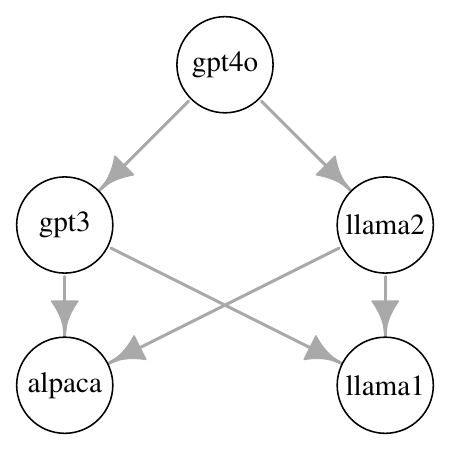}\\
$t=1000$ & $t=1200$ &  $t=1400$ & $t=1600$ \\
\includegraphics[width=3.5cm]{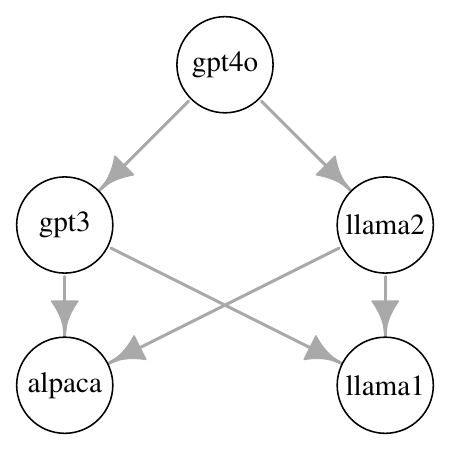}&
\includegraphics[width=3.5cm]{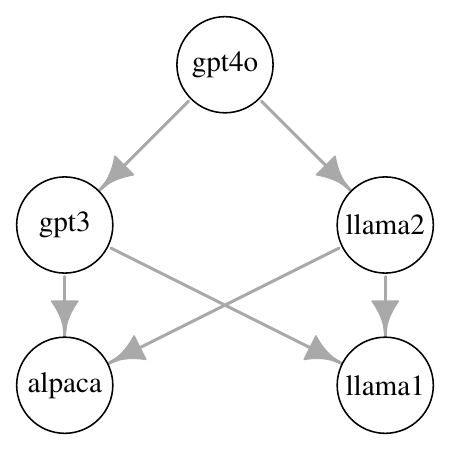}&
\includegraphics[width=3.5cm]{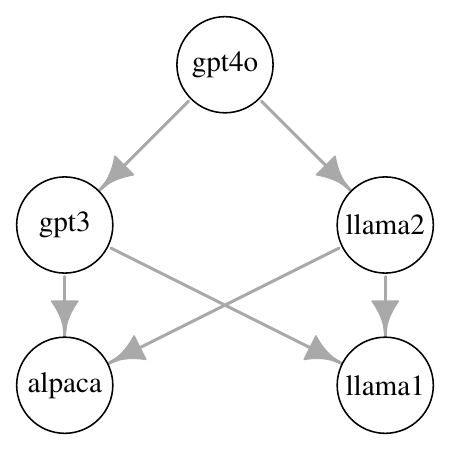}&
\includegraphics[width=3.5cm]{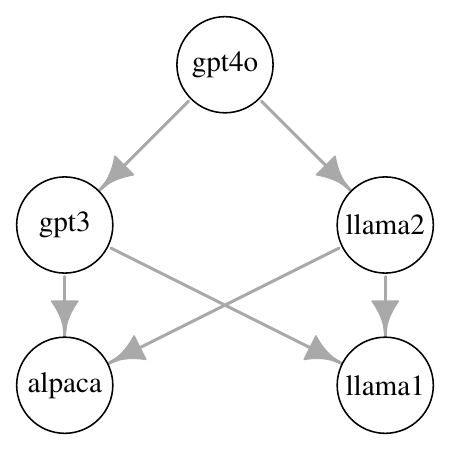}\\
\end{tabular}
\caption{Changes in the ranking of LLMs over time.}
\label{fig:rank-diag}
\end{figure}


\section{Conclusion}\label{sec:conclu}
To conclude, we propose a novel framework for online decision-making based on the contextual BTL model. Different from existing strategies, our approach introduces an innovative two-stage decision-making strategy tailored for pairwise comparison data. We derive sharp error bounds by developing a matrix concentration method tailored to address the challenges posed by the dependent Hessian matrix. Furthermore, we establish the regret bound and demonstrate the superior performance of our method in simulations. We apply our framework to rank LLMs, yielding insightful results. Our work provides a comprehensive solution for online decision-making based on human feedback, offering a powerful tool for ranking applications in diverse domains.

There remain numerous promising directions for future research. We illustrate a few possibilities below. Our proposed two-stage strategy for online decision-making is designed to handle data with dependent structures, which commonly arise in online data collection. Extending our inferential results to other models with similar dependency characteristics is a direction for further exploration. 
\re{For example, we could consider a setting where the agent observes the top choice among more than two candidates, or the Plackett-Luce model where the full ranking order is observed. In the context of online decision-making with time-varying contextual vectors and item-specific factors, we can deploy our two-stage framework to balance exploration and exploitation for regret minimization. However, the theoretical characterization of these results remains scenario-dependent and warrants further study.}
Besides, while this work focuses on the linear contextual BTL model, more complex contextual information may arise in real-world applications. Extending the methodology to accommodate nonlinear contextual models or nonparametric approaches is worth investigating.
Furthermore, building upon the inference results, further investigation into the properties of ranking can yield practical advancements.



\bibliographystyle{chicago}
\bibliography{ref}



\renewcommand{\thefigure}{A.\arabic{figure}}

%

\appendix 
\newpage

\begin{center}
\textit{\large Supplementary material to}
\end{center}
\begin{center}
{\Large Contextual Online Uncertainty-Aware Preference Learning for Human Feedback}
\vskip10pt
\end{center}

This document contains the supplementary material to the paper
``Contextual Online Uncertainty-Aware Preference Learning for Human Feedback".


\re{
\section{Additional Simulation Results}\label{app:sec:simu}
In this section, we present more simulation results. We consider the settings described in Section~\ref{sec:simu}, and for these experiments, we keep all other parameters fixed while varying $M$. The results are presented in Figures~\ref{fig:MA} and \ref{fig:MB}. From the simulations, we observe that as $M$ increases, both the regret and estimation error decrease. Furthermore, our method consistently outperforms the others in terms of regret across all values of $M$.

\begin{figure}[!htbp]
\centering
\foreach \M in {1, 10, 20, 30} { 
    $M=\M$ \\[10pt]
    \begin{tabular}{cc}
    \includegraphics[width=5.5cm]{figs/reg_setting_A_T0_200_T_500_M_\M.pdf} &
    \includegraphics[width=5.5cm]{figs/err_setting_A_T0_200_T_500_M_\M.pdf} \\
    \end{tabular}
    \newpage 
}
\caption{\re{Regret and estimation error for Setting A with varying comparison budgets $M$. The left column shows the regret, and the right column shows the estimation error for each $M$.}}
\label{fig:MA}
\end{figure}

\begin{figure}[!htbp]
\centering
\foreach \M in {20, 50, 80} { 
    $M=\M$ \\[10pt]
    \begin{tabular}{cc}
    \includegraphics[width=5.5cm]{figs/reg_setting_B_T0_200_T_500_M_\M.pdf} &
    \includegraphics[width=5.5cm]{figs/err_setting_B_T0_200_T_500_M_\M.pdf} \\
    \end{tabular}
    \newpage 
}
\caption{\re{Regret and estimation error for Setting B with varying comparison budgets $M$. The left column shows the regret, and the right column shows the estimation error for each $M$.}}
\label{fig:MB}
\end{figure}
}

\section{Proof of Theorem~\ref{thm:2rate}}\label{app:sec:rate}
We present the proof conditioning on  event $\cA$, as defined in Lemma~\ref{dgr}. To prove the theorem, we need the following lemmas, whose proofs are provided in Section~\ref{app:sec:lemerr1}.
\begin{lemma}\label{lem:gr}
Conditioning on  event $\cA$ and under Assumption~\ref{ass:Sigx}, we have $\|\nabla\cL_t(\btheta^*)\|_2\lesssim \sqrt{tM\log T}$ with probability at least $1-O(T^{-3})$.
\end{lemma}
\begin{lemma}\label{lem:pickind}
Conditioning on event $\cA$, let $M_t=\min_{(i,j)\in\cE}\sum_{s=1}^{t}\bm{1}\{\xi_s=1\text{ and }\ba_s=(i,j)\}$.
There exists a constant $c$ such that for $cn^2p\log(T)<t^{1-\alpha}\leq T_0$, and we have $M_t\asymp t^{1-\alpha}/(n^2p)$ with probability at least $1-O(T^{-3})$.
\end{lemma}
Recall that $\be_i$ is the canonical basis in $\RR^n$. Let $\Ab_{ij}=(\be_{i}-\be_{j})(\be_{i}-\be_{j})^\top$, and $\bar{L}_{\mathcal{G}}=\sum_{(i,j)\in \mathcal{E},i>j}\bar{\bC}_{ij}$, where $\bar{\bC}_{ij}=\Ab_{ij}\otimes \bSigma_1$, and $\bSigma_{1}$ is a $d$-dimensional symmetric matrix. The next lemma establishes the relationship between the eigenvalues of $\bar{L}_{\mathcal{G}}$ and $\bSigma_1$.
\begin{lemma}\label{eigmin} 
For any $r>0$, there exists a constant $c$ depending on $r$, such that if $p n\geq c\log n$, we have that  the  event
\begin{align*}
\cA_0(r)=\Big\{& \frac{1}{2}\lambda_{\min }(\bSigma_{1})pn\leq \lambda_{\min,\bot}(\bar{L}_{\mathcal{G}})\leq \|\bar{L}_{\mathcal{G}}\|_{2}\leq \frac{3}{2}\lambda_{\max}(\bSigma_{1})pn,\\
&\text{ for any }\bSigma_{1}\text{ such that }\lambda_{\max}(\bSigma_{1})/\lambda_{\min }(\bSigma_{1})\leq r\Big\}
\end{align*}
holds with probability larger than $1-O\left(n^{-10}\right)$.
\end{lemma}

Now we begin the proof for Theorem~\ref{thm:2rate}. In the proof, since it is sufficient to consider $\cA_0(r)$ for $r=1$ in Lemma~\ref{eigmin}, so we omit the symbol $r$ and denote the event as  $\cA_0$. Also, we denote $\hth(t)$ as $\hth$  in the proof for simplicity.

\begin{proof}
Expanding $\cL_t(\hat{\btheta})$ at $\btheta^*$, we have
\begin{align}\label{expa}
\cL_t(\hat{\btheta})=\cL_t(\btheta^*)+(\hat{\btheta}-\btheta^*)^\top\nabla\cL_t(\btheta^*)
+\frac{1}{2}(\hat{\btheta}-\btheta^*)^\top\nabla^2\cL_t(\zeta)(\hat{\btheta}-\btheta^*),
\end{align}
where $\zeta$ is between $\hat{\btheta}$ and $\btheta^*$. Since $\cL_t(\hat{\btheta})\leq\cL_t(\btheta^*)$ by the definition of $\hat{\btheta}$ and (\ref{expa}), we have
\begin{align*}
\|\hat{\btheta}-\btheta^*\|_2 \|\nabla\cL_t(\btheta^*)\|_2 
&\geq -(\hat{\btheta}-\btheta^*)^\top \nabla\cL_t(\btheta^*)\nonumber \geq \frac{1}{2}(\hat{\btheta}-\btheta^*)^\top\nabla^2\cL_t(\zeta)(\hat{\btheta}-\btheta^*)\nonumber\\
&\geq \frac{1}{2}\lambda_{\min, \perp}(\nabla^2\cL_t(\zeta))\|\hat{\btheta}-\btheta^*\|_{2}^2,
\end{align*}
where $\lambda_{\min, \perp}(\Ab) =\max\{\lambda\given \btheta^\top\Ab\btheta\geq \lambda\|\btheta\|_2^2\text{ for all }\btheta\in\Theta\}$ for matrix $\Ab$. Therefore, we have
\begin{align}\label{2nm}
\|\hat{\btheta}-\btheta^*\|_{2}\leq \frac{2\|\nabla\cL_t(\btheta^*)\|_2 }{\lambda_{\min, \perp}(\nabla^2\cL_t(\zeta))}.
\end{align}
Then we analyze $\lambda_{\min, \perp}(\nabla^2\cL_t(\zeta))$.
Let $\cU_{t,ij}=\{s\in[t]:\xi_s=1,\ba_s=(i,j)\}$ and $\hat{\bSigma}(\cI)=|\cI|^{-1}\sum_{s\in\cI}\bX_s\bX_s^\top$.
For any $i,j\in[n]$, $i\neq j$, $\hat{\bSigma}(\cU_{t,ij})$ is a summation of independent random matrices. In addition, for any $s\in \cU_{t,ij}$, we have $\lambda_{\min}(\bX_s\bX_{s}^\top)\geq 0$. By Assumption~\ref{ass:Sigx}, there exists a constant $\bar{c}_0>0$, such that $\lambda_{\min}(\EE [|\cU_{t,ij}|\hat{\bSigma}(\cU_{t,ij})])\geq |\cU_{t,ij}|\bar{c}_0$. By  the Chernoff bound, we have
\begin{align*}
\PP\Big(\lambda_{\min}(\hat{\bSigma}(\cU_{t,ij}))\leq \frac{1}{2}\bar{c}_{0}\Big)\leq d(\frac{2}{e})^{\bar{c}_{0}|\cU_{t,ij}|/2}.
\end{align*}
Therefore, there exists $c>0$, if $|\cU_{t,ij}|\geq c\cdot\log T$, we have $\lambda_{\min}(\hat{\bSigma}(\cU_{t,ij}))> 1/2\bar{c}_{0}$ holds with probability $1-O(T^{-3})$.
Since  $\cT_{t,ij}=\{s\in[t]:\ba_s=(i,j)\}$, utilizing the decomposition
\begin{align*}
\hat{\bSigma}(\cT_{t,ij})=\frac{|\cU_{t,ij}|}{|\cT_{t,ij}|}\hat{\bSigma}(\cU_{t,ij})+\frac{|\cT_{t,ij}|-|\cU_{t,ij}|}{|\cT_{t,ij}|}\hat{\bSigma}(\cT_{t,ij}\setminus \cU_{t,ij}),
\end{align*}
we  obtain
\begin{align*}
\lambda_{\text{min}}(\sum_{s\in\cT_{t,ij}}\bX_s\bX_s^\top)\geq |\cU_{t,ij}|\lambda_{\text{min}}(\hat{\bSigma}(\cU_{t,ij}))>\frac{|\cU_{t,ij}|}{2}\bar{c}_{0}.
\end{align*}
Since $\sum_{\ell\in\cT_{t,ij}}\bX_\ell\bX_\ell^\top -\min_{(i,j)\in\cE}|\cU_{t,ij}|\bar{c}_{0}/2\Ib_{d}$ and $\Ab_{ij}$ are positive semidefinite for all $(i,j)\in\cE$, we  obtain that  
\begin{align}\label{pos}
\sum_{(i,j)\in\cE,i>j} \Big(\Ab_{ij}\otimes\Big(\sum_{\ell\in\cT_{t,ij}}\bX_\ell\bX_\ell^\top\Big)\Big) \succeq \sum_{(i,j)\in\cE,i>j} \Ab_{ij}\otimes \min_{(i,j)\in\cE}|\cU_{t,ij}|\bar{c}_{0}/2\Ib_{d}.
\end{align}
Note that the left side of (\ref{pos}) equals $\sum_{(i,j)\in \cE,i>j}\sum_{\ell\in\cT_{t,ij}}(\bc_{\ell i}-\bc_{\ell j})(\bc_{\ell i}-\bc_{\ell j})^\top$. Besides, utilizing Lemma~\ref{eigmin}, we have
\begin{align*}
\lambda_{\min, \perp}\Big(\sum_{(i,j)\in\cE,i>j} \Ab_{ij}\otimes \min_{(i,j)\in\cE}|\cU_{t,ij}|\bar{c}_{0}/2\Ib_{d}\Big)\geq \min_{(i,j)\in\cE}|\cU_{t,ij}|\bar{c}_{0}np/4.
\end{align*}
Utilizing the closed form of the Hessian matrix that
\begin{align*}
\nabla^{2} \cL_t(\btheta)=&M\sum_{(i,j)\in \mathcal{E},i>j;}\sum_{\ell\in\cT_{t,ij}}\frac{e^{\bX_{\ell }^\top \btheta_{i}} e^{\bX_{\ell }^\top \btheta_{j}}}{(e^{\bX_{\ell }^\top \btheta_{i}}+e^{\bX_{\ell }^\top \btheta_{j}})^{2}}(\bc_{\ell i}-\bc_{\ell j})(\bc_{\ell i}-\bc_{\ell j})^\top,
\end{align*}
we have $
\frac{e^{\bX_{\ell}^\top\btheta_{i} } e^{\bX_{\ell}^\top\btheta_{j} }}{(e^{\bX_{\ell}^\top\btheta_{i} }+e^{\bX_{\ell}^\top\btheta_{j} })^{2}}\geq \frac{1}{4}e^{-|\bX_{\ell}^\top\btheta_{i} -\bX_{\ell}^\top\btheta_{j} |}\gtrsim \frac{1}{4}e^{-\gamma}$ for $\btheta$ such that $\|\btheta_i\|_2\leq \gamma$ .
Therefore, we have $\lambda_{\min, \perp}\left(\nabla^2 \cL_t(\xi)\right) \gtrsim \min_{(i,j)\in\cE}|\cU_{t,ij}|\bar{c}_{0}npM$.
Using Lemma~\ref{lem:pickind}, we have $\min_{(i,j)\in\cE}|\cU_{t,ij}|\asymp t^{1-\alpha}/(n^2p)$.
Combining this with (\ref{2nm}), we have
\begin{align*}
\|\hat{\btheta}-\btheta^*\|_2\lesssim \frac{\sqrt{t\log T}}{np\sqrt{M} \min_{(i,j)\in\cE}|\cU_{t,ij}|}
\lesssim \sqrt{\frac{\log T}{M}}nt^{\alpha-1/2},
\end{align*}
which concludes the proof.
\end{proof}

\section{Proof of Theorem~\ref{thm:BT}}\label{app:pf:BT}
An essential component of the proof involves bounding the eigenvalues of the Hessian matrix, which we analyze in detail in Section~\ref{sec:Hes}. Next, we establish the property of the gradient and complete the proof in Section~\ref{app:sec:err2}.
\subsection{Analysis of the Hessian Matrix}\label{sec:Hes}
Recall that the Hessian matrix is 
\begin{align*}
\nabla^{2} \cL_t(\btheta)=&\sum_{(i,j)\in \mathcal{E},i>j;}\sum_{\ell\in \cT_{t,ij}}M\frac{e^{\bX_{\ell }^\top \btheta_{i}} e^{\bX_{\ell }^\top \btheta_{j}}}{(e^{\bX_{\ell }^\top \btheta_{i}}+e^{\bX_{\ell }^\top \btheta_{j}})^{2}}(\bc_{\ell i}-\bc_{\ell j})(\bc_{\ell i}-\bc_{\ell j})^\top,
\end{align*}
where $\bc_{\ell i} = \be_{i}\otimes \bX_\ell \in\mathbb{R}^{nd}$. To facilitate further discussion, we introduce some additional notations. In the first stage, we set $\xi_t=1$ if we perform exploration; otherwise, we set $\xi_t=0$. Define $\bSigma_{ij}(t)=\sum_{\ell\in\cT_{t,ij}}\bX_\ell\bX_\ell^\top$, $\bSigma_{i}(t)=\sum_{\ell=1}^t\bm{1}(i_\ell=i)\bX_\ell\bX_\ell^\top$ and $\cS_t=\{\ell\in[t]:\xi_\ell=0\text{ or }\ell>T_0\}$.
For $\ell\in \cS_T$, by our algorithm, given the major item $i_\ell$, item $j_\ell$ follows a multinomial distribution, and $j_\ell$ is independent of $\bX_{\ell}$ and historical information $\cH_{\ell-1}$. To represent  action $\ba_\ell$ more clearly, we slightly abuse the notations. Let $\bb(i)$ denote the randomly selected item associated with the main item $i$. Note that the randomness of $\ba_\ell$ arises from the covariate $\bX_\ell$, the estimation from last step $\hat{\btheta}(\ell-1)$ and $\bb(i_\ell)$, where $i_\ell=\argmax_{i\in[n]}\bX_{\ell}^\top\hth_i(\ell-1)$. Thus, we represent $\ba_{\ell}=\ba(\hth(\ell-1),\bX_\ell,\bb(i_\ell))$ for $\ell\in\cS_T$.

Define $\bSigma_{\ell i}=
\bm{1}(i_\ell=i)\bX_\ell\bX_\ell^\top$, from which we  obtain $\bSigma_{i}(t)=\sum_{\ell=1}^{t}\bSigma_{\ell i}$. Let $\bSigma^{\sharp}_{i}(\btheta)=\EE_{\bX}[\bm{1}(\argmax_j \bX^\top\btheta_j=i)\bX\bX^\top]$, and note that $\bSigma^{\sharp}_{i}(\hat{\btheta}(\ell-1))=
\EE_{\bX}[\bm{1}(\argmax_j(\bX^\top\hth_j(\ell-1))=i)\bX\bX^\top]
=\EE_{\bX_\ell}\bSigma_{\ell i}$ for $\ell\in\cS_T$.
Similarly, define $\tilde{\bSigma}_{i}(t)=\sum_{\ell=1}^t\bm{1}(i_\ell=i)\bX_\ell^\top\bX_\ell$, \re{$\tilde \bSigma_{ij}(t)=\sum_{\ell\in\cT_{t,ij}}\bX_\ell^\top\bX_\ell$.}, $\tilde{\bSigma}_{\ell i}=\bm{1}(i_\ell=i)\bX_\ell^\top\bX_\ell$ and $\tilde{\bSigma}^{\sharp}_{i}(\btheta)=\EE_{\bX}[\bm{1}(\argmax_j \bX^\top\btheta_j=i)\bX^\top\bX]$. We begin by analyzing $\bSigma_{i}(t)$ and $\tilde{\bSigma}_{i}(t)$, and then analyze  $\bSigma_{ij}(t)$ and $\tilde \bSigma_{ij}(t)$. Finally,  leveraging the graph structure, we derive properties of $\nabla^{2} \cL_t(\btheta)$. The proofs of the following lemmas are provided in Section~\ref{app:sec:Hes}.

\begin{lemma}\label{lem:hes}
Let Assumptions~\ref{ass:Sigx} and \ref{ass:dif} hold. Conditioning on  event $\cA$, for $\bar{\btheta}$ such that $\max_{k\in[n]}\|\bar{\btheta}_{k}-\btheta^{*}_{k}\|_{2}= \epsilon_{0}$, we have $\|\bSigma^{\sharp}_{i}(\bar{\btheta})-\bSigma_{i}^*\|_2\lesssim \sqrt{\epsilon_{0}}$ and $|\tilde{\bSigma}^{\sharp}_{i}(\bar{\btheta})
-\tilde{\bSigma}_{i} |\lesssim \sqrt{\epsilon_{0}}$, where $i\in[n]$.
\end{lemma}
\begin{lemma}\label{app:lem:estil}
Let Assumption~\ref{ass:Sigx} hold, and let $c_{\bX}$ represent $\max\|\bX\|_2$. For each $t\geq T_0$ and $i\in[n]$, we have
\begin{align*}
\max\Big\{\Big\|\sum_{\ell\in\cS_t}\Big(\bSigma_{\ell i}-\bSigma^{\sharp}_{i}(\hat{\btheta}(\ell-1))\Big)\Big\|_2,
\Big|\sum_{\ell\in\cS_t}\Big(\tilde{\bSigma}_{\ell i}-\tilde{\bSigma}^{\sharp}_{i}(\hat{\btheta}(\ell-1))\Big)\Big|
\Big\}\leq 4c_{\bX}\sqrt{6|\cS_t|\log T}
\end{align*}
holds with probability $1-\tilde{C}_1T^{-3}$, where $\tilde{C}_1$ is a constant.
\end{lemma}

Let $\tilde{c}_0\leq \lambda_{\min}(\bSigma_{i}^*)\leq \tilde{\bSigma}_{i}\leq \tilde{c}_1$ for all $i\in[n]$. 
We show  the  event 
\begin{align*}
\cA_t^1=\Big\{s\tilde{c}_0/2\leq \lambda_{\min}(\bSigma_{i}(s))\leq \tilde{\bSigma}_{i}(s)\leq 3s\tilde{c}_1/2,\, T_0\leq s\leq t, i\in[n]\Big\}
\end{align*}
holds with high probability. Specifically, we decompose $\cA_t^1$ into the following two events
\begin{align*}
&\cA_t^{1,1}=\Big\{\lambda_{\min}(\sum_{\ell\in \cS_{s}\setminus [\tilde{T}_0]}\bSigma_{i}^*)/2
-\|\sum_{\ell\in \cS_{s}\setminus [\tilde{T}_0]}\Big(\bSigma_{\ell i}-\bSigma^{\sharp}_{i}(\hat{\btheta}(\ell-1))\Big)\|_2\geq 3s\tilde{c}_0/8,\, T_0\leq s\leq t, i\in[n]\Big\},\\
&\cA_t^{1,2}=\Big\{\lambda_{\min}(\sum_{\ell\in \cS_{s}\setminus [\tilde{T}_0]}\bSigma_{i}^*)/2
-\|\sum_{\ell\in \cS_{s}\setminus [\tilde{T}_0]}\Big(\bSigma^{\sharp}_{i}(\hat{\btheta}(\ell-1))-\bSigma_{i}^*\Big)\|_2\geq 3s\tilde{c}_0/8,\, T_0\leq s\leq t, i\in[n]\Big\}.
\end{align*}
Combining the results of Lemmas~\ref{lem:hes}, \ref{app:lem:estil}, and Theorem~\ref{thm:2rate}, we bound the eigenvalues of $\bSigma_{i}(t)$ and $\tilde{\bSigma}_{i}(t)$ in Lemma~\ref{lem:covT}. Since the historical pulling times of item pairs affect the convergence rate of the current estimation, which in turn influences the properties of the Hessian matrix, we present the following result in an inductive manner. The provision of the history estimation error is established in Corollary~\ref{cor}. Recall that we have $\tilde T_0=\max\{(n^2p)^{1/\alpha+1}(\log T)^{1/(1-\alpha)},\\ n^{6/(1-2\alpha)}(\log T)^{1/(1-2\alpha)}/(pM)\}$.
\begin{lemma}\label{lem:covT}
Let Assumption~\ref{ass:dif} and  event $\cA$ hold. Under the conditions of Theorem~\ref{thm:2rate}, there exist constants $c$ and $\tilde{C}$, such that for $T_0\geq c\tilde T_0$, we have $\cA_{T_0}^{1,1}\cap \cA_{T_0}^{1,2}$ holds with probability at least $1-\tilde{C}T_0/T^3$. In addition, $\cA_{s}^{1,1}\cap \cA_{s}^{1,2}$ holds for any $s\in [T_0+1,T]$ with probability at least  $1-\tilde{C}s/T^3$, provided that
$\max_{k\in[n]}\|\hat{\btheta}_k(t)-\ths_k\|_2
\lesssim \sqrt{\frac{n\log T}{ptM}}$ holds with probability at least $1-\tilde{C}_2 T^{-3}$ for $t\in[T_0,s-1]$, where $\tilde{C}_2$ is a constant.
\end{lemma}

\begin{lemma}\label{lem:Sigij}
Let Assumption~\ref{ass:Sigx} hold. For $t\geq T_0$, conditioning on  event $\cA\cap \cA_t^1$, we have
\begin{align*}
\cA_t^2=\{t\tilde{c}_0/(2np)\leq \lambda_{\min}(\bSigma_{ij}(t))\leq \tilde{\bSigma}_{ij}(t)\leq 13t\tilde{c}_1/(2np),\, (i,j)\in\cE\}
\end{align*} holds with probability at least  $1-O(T^{-3})$.
\end{lemma}

Let $\max_{1\leq i< j\leq n}\|\btheta^{*}_{i}-\btheta^{*}_{j}\|_{2}= c_{\btheta}$. Set $\kappa=e^{2c_{\bX}c_{\btheta}}$. By Assumptions~\ref{ass:Sigx} and \ref{assmp:kappa}, there exists constant $C$, such that $\kappa\leq C$. By Lemmas~\ref{lem:Sigij} and \ref{eigmin}, we obtain the following bound on the eigenvalues of the Hessian. 
\begin{lemma}\label{Hessmin}
Under Assumptions~\ref{ass:Sigx} and \ref{assmp:kappa}, conditioning on  event $\cA\cap \cA_0\cap \cA_t^2$, for $\btheta\in\Theta$ such that $\max_{i\in [n]}\|\btheta_{i}-\btheta^{*}_{i}\|_{2}\leq \epsilon_{1}$, we have
\begin{align*}
\lambda_{\min, \perp}\left(\nabla^2 \cL_t(\btheta)\right) \geq \frac{\tilde{c}_0tM}{8\kappa e^{2\epsilon_1c_{\bX}}},
\end{align*}
and for all $\btheta\in\Theta$, we have
\begin{align*}
\lambda_{\max }\left(\nabla^2 \mathcal{L}_{t}(\btheta)\right) \leq 3\tilde{c}_1tM.
\end{align*}
\end{lemma}

\subsection{Proof of the Convergence Rate}\label{app:sec:err2}
We denote   $\hth(t)$ as $\hth$  in the proof for simplicity. We first present the properties of the gradient. Then we derive the convergence rate leveraging a gradient descent sequence, which serves as an intermediary between $\hth$ and $\ths$. The proofs of the following lemmas are provided in Section~\ref{app:sec:ratepf}.

\begin{lemma}\label{gr}
Under Assumption~\ref{ass:Sigx}, conditioning on $\cA\cap \cA_0\cap\cA_t^2$, there exists a constant $c_1$  that the following event
\begin{align*}
\cA_t^3=\left\{\|\nabla\mathcal{L}_{\lambda,t}(\btheta^{*})\|_{2}\leq c_1\sqrt{\tilde{c}_1ntM\log T}\right\}
\end{align*}
holds with probability at least $1-O(T^{-3})$, where  $\lambda\leq  c_{\lambda}\sqrt{\tilde{c}_1tM\log T}$.
\end{lemma}

Let $\eta=1/(\lambda+10\tilde{c}_{1}tM)$ and $\tilde{T}=t$. Define $\{\btheta^{s}\}_{s=0,1,\ldots,\tilde{T}}$ as the gradient descent sequence, where $\btheta^0=\btheta^*$ and $\btheta^{s}=\btheta^{s-1}-\eta \nabla\mathcal{L}_{\lambda,t}\left(\btheta^{s-1}\right)$ for $s\in[\tilde{T}]$.
Using the properties of gradient descent, we obtain the following bound for $\|\btheta^{\tilde{T}}-\hat{\btheta}\|_\infty$.
\begin{lemma}\label{thetaT}
Conditioning on $\cA\cap \cA_0\cap\cA_t^2\cap\cA_t^3$, with $\lambda=  c_{\lambda}\sqrt{\tilde{c}_1tM\log T}$, we have
\begin{align}\label{e1}
\left\|\btheta^{\tilde{T}}-\hat{\btheta}\right\|_\infty \leq \left\|\btheta^{\tilde{T}}-\hat{\btheta}\right\|_2 \lesssim \sqrt{\frac{n\log T}{ptM}}.
\end{align}
\end{lemma}

Next, we analyze $\max_{i\in[n]}\|\btheta^{\tilde{T}}_{i}-\btheta^*_{i}\|_{2}$. We construct a sequence $\left\{\btheta^{s,(q)}\right\}_{s=0,1,\ldots,\tilde{T}}$ for $q\in [n]$ as follows. Let $\btheta^{0,(q)}=\btheta^0=\btheta^*$ and
\begin{align*}
\btheta^{s,(q)}=\btheta^{s-1,(q)}-\eta \nabla\mathcal{L}_{\lambda,t}^{(q)}\left(\btheta^{s-1,(q)}\right), s\in[\tilde{T}],
\end{align*} 
where
\begin{align*}
\mathcal{L}_{\lambda,t}^{(q)}(\btheta) = &\sum_{(i,j)\in\cE,i>j;\atop i,j\neq q}\sum_{\ell\in\cT_{t,ij};}\sum_{m=1}^{M}\Big\{-y_{j, i}^{(m)}(\ell)\left(\btheta_i-\btheta_j\right)^{\top}\bX_{\ell}+\log \left(1+e^{(\btheta_i-\btheta_j)^{\top}\bX_{\ell}}\right)\Big\} \nonumber\\
& +\sum_{j:(j,q)\in\cE;}\sum_{\ell\in\cT_{t,jq}}M\Big\{-\frac{e^{\bX_{\ell}^\top\btheta_j^{*}}}{e^{\bX_{\ell}^\top\btheta_j^{*}}+e^{\bX_{\ell}^\top\btheta_q^{*}}}\left(\btheta_j-\btheta_q\right)^{\top}\bX_{\ell}+\log \left(1+e^{(\btheta_j-\btheta_q)^{\top}\bX_{\ell}}\right)\Big\}+\frac{1}{2} \lambda\|\btheta\|_2^2 .
\end{align*}
We then provide a lemma concerning the distance between the sequence $\left\{\btheta^{s,(q)}\right\}$ and $\ths$.
\begin{lemma}\label{induc}
Under Assumptions~\ref{ass:Sigx}-\ref{assmp:kappa}, conditioning on event $\cA\cap \cA_0\cap\cA_t^2\cap\cA_t^3$, for $s\in \{0,1,\ldots,\tilde{T}\}$, the following inequalities hold with probability at least $1-O(T^{-4})$ that
\begin{equation}
\begin{aligned}\label{d}
\|\btheta^{s}-\btheta^{*}\|_{2} & \leq \frac{C_1 \kappa}{\tilde{c}_0} \sqrt{\frac{n\tilde{c}_1\log T}{tM}}\\
\max_{q\in [n]}\|\btheta_{q}^{s,(q)}-\btheta_{q}^{*}\|_{2}
& \leq \frac{C_{2}\kappa^2}{\tilde{c}_0}\sqrt{\frac{\tilde{c}_1\log T}{ ptM}}\\
\max_{q\in[n]}\|\btheta^{s,(q)}-\btheta^{s}\|_{2}
&\leq
\frac{C_{3}\kappa}{\tilde{c}_0}\sqrt{\frac{\tilde{c}_1\log T}{tM}}\\
\max_{i\in [n]}\|\btheta_{i}^{s}-\btheta_{i}^{*}\|_{2}
&\leq
\frac{C_{4}\kappa^2}{\tilde{c}_0}
\sqrt{\frac{\tilde{c}_1\log T}{ptM}},
\end{aligned}
\end{equation}
where $C_1,\ldots,C_4$ are constants.
\end{lemma}
Then, we take the union bound over $\tilde{T}$ iterations. We have
\begin{align}\label{e2}
\max_{q\in [n]}\|\btheta_{q}^{\tilde{T}}-\btheta_{q}^{*}\|_{2}
\leq \frac{C_{4}\kappa^2}{\tilde{c}_0}\sqrt{\frac{\tilde{c}_1\log T}{ ptM}}
\end{align}
with probability at least $1-O(T^{-3})$. We have $\tilde{c}_0\asymp\tilde{c}_1\asymp 1/n$ under Assumptions~\ref{ass:Sigx} and \ref{ass:Hesx}. Combining (\ref{e1}) and (\ref{e2}), we  obtain
\begin{align*}
\max_{q\in [n]}\|\hat{\btheta}_{q}-\btheta_{q}^{*}\|_{2}  \leq\max_{q\in [n]}\|\btheta_{q}^{\tilde{T}}-\btheta_{q}^{*}\|_{2}+\max_{q\in [n]}\|\btheta_{q}^{\tilde{T}}-\hat{\btheta}_{q}\|_{2}\lesssim \sqrt{\frac{n\log T}{ptM}},
\end{align*}
with probability at least $1-O(T^{-3})$.
Combining Lemmas~\ref{thetaT} and \ref{induc}, we have the following corollary.
\begin{corollary}\label{cor}
Under Assumptions~\ref{ass:Sigx}--\ref{assmp:kappa}, conditioning on  event $\cA\cap \cA_0\cap\cA_t^2\cap\cA_t^3$, we have that
\begin{align*}
\max_{i\in [n]}\|\hat{\btheta}_{i}(t)-\btheta_{i}^{*}\|_{2} \lesssim \sqrt{\frac{n\log T}{ptM}}
\end{align*}
holds with probability at least $1-\tilde{C}_2 T^{-3}$, where $\tilde{C}_2$ is a constant.
\end{corollary}

Note that event $\cA\cap \cA_0\cap\cA_t^2\cap\cA_t^3$ holds with probability $1-O(\max\{tT^{-3},n^{-10}\})$ by Lemma~\ref{lem:covT}, which concludes the proof.

\section{Proof of Theorem~\ref{thm:reg}}
\begin{proof}
Recall that the cumulative regret $R(T)=\frac{1}{T}\sum_{t=1}^{T}\Big(\bX_{t}^\top\btheta_{i_t^*}^{*}-\bX_{t}^\top\btheta_{i_t}^{*}\Big).$
We bound the regret in the two stages separately.
For each exploitation step $t$, we bound the regret due to  mischoice by
\begin{align*}
\bX_t^\top \btheta_{i_t^*}^{*}-\bX_t^\top \btheta_{i_t}^{*} & =  
\bX_t^\top \btheta_{i_t^*}^{*}-\bX_t^{\top} \hth_{i_t^*}(t-1)+\bX_t^{\top} \hth_{i_t^*}(t-1)-\bX_t^{\top} \hth_{i_t}(t-1)+\bX_t^{\top} \hth_{i_t}(t-1)-\bX_t^\top \btheta_{i_t}^{*}\nonumber\\
& \leq 2\max_{i\in[n]}\|\hth_i(t-1)-\ths_i\|_2\|\bX_t\|_2,
\end{align*}
where the last inequality holds by the fact that $\bX_t^\top \hth_{i^*}(t-1)-\bX_t^\top \hth_{i}(t-1)\leq 0$.


Combining the bounds  in Theorems~\ref{thm:2rate} and \ref{thm:BT}, we  obtain
\begin{align}\label{re2}
\sum_{t=1}^{T}\Big(\bX_{t}^\top \btheta_{i_t^*}^{*}-\bX_{t}^\top \btheta_{i_t}^{*}\Big)
\lesssim&\Big(n^2p\log T\Big)^{1/(1-\alpha)}
+\sum_{t=1}^{T_0}\bm{1}(\xi_t=1)
+\sum_{t=1}^{T_0}\sqrt{\frac{\log T}{M}}nt^{\alpha-1/2}
+\sum_{t=T_0+1}^{T}
\sqrt{\frac{n\log T}{ptM}}\nonumber\\
\lesssim&\Big(n^2p\log T\Big)^{1/(1-\alpha)}
+T_0^{1-\alpha}
+ \sqrt{\frac{\log T}{M}}nT_0^{\alpha+1/2}
+\sqrt{\frac{nT\log T}{pM}},
\end{align}
where the last inequality holds by Bernstein inequality. Note that Theorems~\ref{thm:2rate} and \ref{thm:BT} are conditioning on the same event $\cA\cap\cA_0$, as detailed in Lemmas~\ref{eigmin} and \ref{dgr}. By applying the union bound over $T$ steps, (\ref{re2}) holds with probability at least $1-O(\max\{T^{-1},n^{-10}\})$. Thus, we bound the regret by 
\begin{align*}
R(T)
\lesssim&\frac{1}{T}\Big(n^2p\log T\Big)^{1/(1-\alpha)}
+\frac{T_0^{1-\alpha}}{T}
+ \sqrt{\frac{\log T}{M}}\frac{nT_0^{\alpha+1/2}}{T}
+\sqrt{\frac{n\log T}{pMT}},
\end{align*}
which concludes the proof.
\end{proof}

\section{Proof of Theorem~\ref{thm:inf}}\label{app:sec:inf}
We establish the asymptotic distribution of $\hth^{d}(t)$ in this section. Recall that we have 
\begin{align*}
\bPsi_t(\hth(t))\binom{\hth^{d}(t)-\hth(t)}{\mu_t}=\binom{-\nabla \cL_t(\hth(t))}{-f_t(\hth(t))}
\end{align*}
from \eqref{deb}. We first provide some lemmas, and we provide their proofs in Section~\ref{app:sec:infr}. 
\begin{lemma}\label{lem:infma}
Under the conditions of Theorem~\ref{thm:BT}, with probability at least $1-O(\max\{T^{-2},n^{-10}\})$, the $(n+1)d\times (n+1)d$ matrix $\bPsi_t(\btheta)$ is invertible for $\btheta \in \{\ths,\hth(t)\}$ and its inverse is of the form $\bPhi_t(\btheta)=\left(\begin{array}{cc}
\bPhi_{t,11}(\btheta) & \frac{1}{ntM}\Jb \\
\frac{1}{ntM}\Jb^{\top} & \bm{0}
\end{array}\right)$, where $\Jb=(\Ib_{d},\ldots,\Ib_{d})^\top$. Further, letting $(\bPhi_{t,11}(\btheta))_{jk}$ represent the $d$-dimensional block of $\bPhi_{t,11}(\btheta)$ at $(j,k)$ position, we have 
$\lambda_{\min}((\bPhi_{t,11}(\btheta))_{kk})\asymp \|(\bPhi_{t,11}(\btheta))_{kk}\|_{2}\asymp\frac{n}{tM}$ for all $k\in[n]$ with probability at least $1-O(\max\{T^{-2},n^{-10}\})$. 
\end{lemma}

By the above lemma, we have 
\begingroup
\allowdisplaybreaks
\begin{align}\label{expan}
& \binom{\hat{\btheta}^{d}(t)-\ths}{\mu_t}=\binom{\hat{\btheta}^{d}(t)-\hth(t)}{\mu_t}+\binom{\hth(t)-\ths}{0} \nonumber\\
 &=\left(\begin{array}{cc}
\nabla^2 \cL_t(\hth(t)) & tM\Jb \\
tM\Jb^{\top} & \bm{0}
\end{array}\right)^{-1}\binom{-\nabla \cL_t(\hth(t))}{-f_t(\hth(t))}+\left(\begin{array}{cc}
\nabla^2 \cL_t(\hth(t)) & tM\Jb \\
tM\Jb^{\top} & \bm{0}
\end{array}\right)^{-1}\binom{\nabla^2 \cL_t(\hth(t))(\hth(t)-\ths)}{\nabla f_t(\hth(t))^{\top}(\hth(t)-\ths)}\nonumber\\
&=  \underbrace{\left(\left(\begin{array}{cc}
\nabla^2 \cL_t(\hth(t)) & tM\Jb \\
tM\Jb^{\top} & \bm{0}
\end{array}\right)^{-1}-\left(\begin{array}{cc}
\nabla^2 \cL_t(\ths) &tM\Jb \\
tM\Jb^{\top} & \bm{0}
\end{array}\right)^{-1}\right)\binom{-\nabla \cL_t(\hth(t))+\nabla^2 \cL_t(\hth(t))(\hth(t)-\ths)}{0}}_{I_1} \nonumber\\
&\quad +\underbrace{\left(\begin{array}{cc}
\nabla^2 \cL_t(\ths) & tM\Jb \\
tM\Jb^{\top} & \bm{0}
\end{array}\right)^{-1}\binom{\nabla \cL_t(\ths)-\nabla \cL_t(\hth(t))+\nabla^2 \cL_t(\hth(t))\left(\hth(t)-\ths\right)}{0}}_{I_2} \nonumber\\
& \quad +\left(\begin{array}{cc}
\bPhi_{t,11}^* & \frac{1}{ntM}\Jb \\
\frac{1}{ntM}\Jb^{\top} & \bm{0}
\end{array}\right)\binom{-\nabla \cL_t(\ths)}{0}.
\end{align}
\endgroup
The next two lemmas establish the central limit theorem for the main term and the rate for the remainder term, respectively.
\begin{lemma}\label{lem:clt}
Under the conditions of Theorem~\ref{thm:BT}, as $n,t\rightarrow\infty$, we have
\begin{align*}
((\bPhi_{t,11}^*)_{kk})^{-1/2} (\bPhi_{t,11}^*\nabla\cL_t(\ths))_{k}\overset{d}\longrightarrow N(0,\Ib_{d}),
\end{align*}
where $(\bPhi_{t,11}^*\nabla\cL_t(\ths))_{k}$ represents the $((k-1)d+1)$-th through $kd$-th entries of $\bPhi_{t,11}^*\nabla\cL_t(\ths)$.
\end{lemma}

\begin{lemma}\label{lem:remtm}
Under the conditions of Theorem~\ref{thm:BT}, with probability at least $1-O(\max\{T^{-2},n^{-10}\})$, we have
\begin{equation*}
\begin{gathered}
\left\|\nabla \cL_t(\ths)\right\|_{\infty} \lesssim \sqrt{ptM\log T},\\
\left\|\nabla \cL_t(\hth(t))-\nabla \cL_t(\ths)-\nabla^2 \cL_t(\ths)\left(\hth(t)-\ths\right)\right\|_{\infty} \lesssim \frac{\log T}{p},\\
\left\|\nabla^2 \cL_t(\hth(t))-\nabla^2 \cL_t(\ths)\right\|_{\infty} \lesssim  \sqrt{\frac{tM\log T}{np}},\\
\left\|\left(\begin{array}{cc}
\nabla^2 \cL_t(\hth(t)) & \nabla f_t(\hth(t)) \\
\nabla f_t^{\top}(\hth(t)) & \bm{0}
\end{array}\right)^{-1}-\left(\begin{array}{cc}
\nabla^2 \cL_t(\ths) & \nabla f_t(\ths) \\
\nabla f_t^{\top}(\ths) & \bm{0}
\end{array}\right)^{-1}\right\|_2 \lesssim \sqrt{\frac{n^3\log T}{pt^3M^3}}.
\end{gathered}
\end{equation*}
\end{lemma}

By Lemma~\ref{lem:remtm}, we obtain the bound of the first two terms in (\ref{expan}).  
For $I_2$, we have
\begin{align}\label{I2}
\|I_2\|_{\infty}\leq &\left\|\left(\begin{array}{cc}
\nabla^2 \cL_t(\ths) & tM\Jb \\
tM\Jb^{\top} & \bm{0}
\end{array}\right)^{-1}\right\|_\infty
\left\|\nabla \cL_t(\ths)-\nabla \cL_t(\hth(t))+\nabla^2 \cL_t(\hth(t))\left(\hth(t)-\ths\right)\right\|_\infty\nonumber\\
\lesssim &\sqrt{n}\frac{n}{tM}\frac{\log T}{p}\asymp n^{3/2}\frac{\log T}{ptM}
\end{align}
holds with probability at least $1-O(\max\{T^{-2},n^{-10}\})$, where the second inequality holds as 
\begin{align*}
&\left\|\nabla \cL_t(\ths)-\nabla \cL_t(\hth(t))+\nabla^2 \cL_t(\hth(t))\left(\hth(t)-\ths\right)\right\|_\infty\nonumber\\
&\leq \left\|\nabla \cL_t(\ths)-\nabla \cL_t(\hth(t))+\nabla^2 \cL_t(\ths)\left(\hth(t)-\ths\right)\right\|_\infty \\
&\quad +\left\|\left(\nabla^2 \cL_t(\ths)-\nabla^2 \cL_t(\hth(t))\right)\left(\hth(t)-\ths\right)\right\|_\infty\nonumber\\
&\lesssim \frac{\log T}{p}+\sqrt{\frac{tM\log T}{np}} \sqrt{\frac{n\log T}{ptM}}\lesssim \frac{\log T}{p}.
\end{align*}
Similarly, we bound $I_1$ by
\begin{align}\label{I1}
\|I_1\|_{\infty}
\leq&\left\|\left(\begin{array}{cc}
\nabla^2 \cL_t(\hth(t)) & tM\Jb \\
tM\Jb^{\top} & \bm{0}
\end{array}\right)^{-1}-\left(\begin{array}{cc}
\nabla^2 \cL_t(\ths) &tM\Jb \\
tM\Jb^{\top} & \bm{0}
\end{array}\right)^{-1}\right\|_\infty \nonumber \\
& \quad \cdot
\left\|-\nabla \cL_t(\hth(t))+\nabla^2 \cL_t(\hth(t))\left(\hth(t)-\ths\right)\right\|_\infty\nonumber\\
\lesssim&\sqrt{n}\left\|\left(\begin{array}{cc}
\nabla^2 \cL_t(\hth(t)) & tM\Jb \\
tM\Jb^{\top} & \bm{0}
\end{array}\right)^{-1}-\left(\begin{array}{cc}
\nabla^2 \cL_t(\ths) &tM\Jb \\
tM\Jb^{\top} & \bm{0}
\end{array}\right)^{-1}\right\|_2  \nonumber \\
&\quad \cdot \left\|-\nabla \cL_t(\hth(t))+\nabla^2 \cL_t(\hth(t))\left(\hth(t)-\ths\right)\right\|_\infty\nonumber\\
\lesssim&\sqrt{n}\sqrt{\frac{n^3\log T}{pt^3M^3}} \Big(\sqrt{ptM\log T}+\frac{\log T}{p}\Big)\asymp \frac{n^2\log T}{tM}+n^2(\frac{\log T}{ptM})^{3/2}
\end{align}
with probability $1-O(\max\{T^{-2},n^{-10}\})$,
where the last inequality holds since
 \begin{align*}
& \left\|-\nabla \cL_t(\hth(t))+\nabla^2 \cL_t(\hth(t))\left(\hth(t)-\ths\right)\right\|_\infty\nonumber\\
 &\ \leq \left\|\nabla \cL_t(\ths)\right\|_\infty
+\left\|\nabla \cL_t(\ths)-\nabla \cL_t(\hth(t))+\nabla^2 \cL_t(\ths)\left(\hth(t)-\ths\right)\right\|_\infty\nonumber\\
&\ \quad +\left\|\left(\nabla^2 \cL_t(\ths)-\nabla^2 \cL_t(\hth(t))\right)\left(\hth(t)-\ths\right)\right\|_\infty\nonumber\\
& \ \lesssim \sqrt{ptM\log T}+\frac{\log T}{p}.
\end{align*}
Combining (\ref{expan}), Lemma~\ref{lem:clt}, (\ref{I1}) and (\ref{I2}),
when $\frac{n^3\log^2 T}{tM}+\frac{n^2\log^2 T}{p^2tM}=o(1)$, we have
\begin{align*}
(\bPhi_{t,11}^*)_{kk}^{-1/2}(\hat{\btheta}_{k}^{d}-\btheta_{k}^{*})\overset{d}\longrightarrow N(0,\Ib_{d}),
\end{align*}
which concludes the proof.

\re{
\section{Proof of Corollary~\ref{cor:inf}}\label{app:sec:espl}
\begin{proof}
First, note that we can obtain
\begin{align}
&\Big\|((\hat\bPhi_{t,11})_{ii})^{-1/2}((\bPhi_{t,11}^*)_{ii})^{1/2}-\II_d\Big\|_2\nonumber\\
&=\Big\|((\hat\bPhi_{t,11})_{ii})^{-1/2}\Big(((\bPhi_{t,11}^*)_{ii})^{1/2}-((\hat\bPhi_{t,11})_{ii})^{1/2}\Big)\Big\|_2\nonumber\\
&\leq \Big\|((\hat\bPhi_{t,11})_{ii})^{-1/2}\Big\|_2\cdot \Big\|((\bPhi_{t,11}^*)_{ii})^{1/2}-((\hat\bPhi_{t,11})_{ii})^{1/2}\Big\|_2,\label{eq:ppI}
\end{align}
where we have $\|((\hat\bPhi_{t,11})_{ii})^{-1/2}\|_2\lesssim(\frac{tM}{n})^{1/2}$ using Lemma~\ref{lem:infma}.

In addition, we can deduce that 
\begin{align}
&\Big\|((\bPhi_{t,11}^*)_{ii})^{1/2}-((\hat\bPhi_{t,11})_{ii})^{1/2}\Big\|_2\nonumber\\
&\leq \frac{1}{\lambda_{\min}((\bPhi_{t,11}^*)_{ii})^{1/2}+\lambda_{\min}((\hat\bPhi_{t,11})_{ii})^{1/2}}\Big\|(\bPhi_{t,11}^*)_{ii}-(\hat\bPhi_{t,11})_{ii}\Big\|_2\label{eq:1o2}
\end{align}
using Theorem~6.2 in \citet{higham2008functions}.
Note that Lemma~\ref{lem:remtm} yields that
\begin{align*}
&\Big\|(\bPhi_{t,11}^*)_{ii}-(\hat\bPhi_{t,11})_{ii}\Big\|_2\leq \Big\|\bPhi^*_t-\hat{\bPhi}_t\Big\|_2 \lesssim \sqrt{\frac{n^3\log T}{pt^3M^3}}.
\end{align*}
Therefore, utilizing Lemma~\ref{lem:infma}, we have
\begin{align*}
\eqref{eq:1o2} \lesssim \sqrt{\frac{tM}{n}}\cdot \sqrt{\frac{n^3\log T}{pt^3M^3}}\lesssim \sqrt{\frac{n^2\log T}{pt^2M^2}}.
\end{align*}
By combining \eqref{eq:ppI}, we obtain that
\begin{align*}
\|((\hat\bPhi_{t,11})_{ii})^{-1/2}((\bPhi_{t,11}^*)_{ii})^{1/2}-\II_d\|_2\lesssim \sqrt{\frac{tM}{n}} \cdot \sqrt{\frac{n^2\log T}{pt^2M^2}} \leq \sqrt{\frac{n\log T}{ptM}}.
\end{align*}

Therefore, we have $((\hat\bPhi_{t,11})_{ii})^{-1/2}((\bPhi_{t,11}^*)_{ii})^{1/2}\rightarrow \II_d$ in probability. We conclude the corollary by Slutsky's theorem.
\end{proof}
}


\section{Proof of Lemmas in Section~\ref{app:sec:rate}}\label{app:sec:lemerr1}

\subsection{Proof of Lemma~\ref{lem:gr}}
It is a simple version of Lemma~\ref{gr}, and the proof can be deduced by simple modification of the proof of Lemma~\ref{gr}, so we omit it.

\subsection{Proof of Lemma~\ref{lem:pickind}}
For $(i,j)\in\cE$, let $M_{t,ij}=\sum_{s=1}^{t}\bm{1}\{\xi_s=1\text{ and }\ba_s=(i,j)\}=:\sum_{s=1}^{t}\eta_{s,ij}$, where $\eta_{s,ij}$ obeys $\text{Bernoulli}(1/(s^\alpha n_{\text{edge}}))$ and $n_{\text{edge}}$ is the number of edges for the given $\cG$. Note that $\{\eta_{s,ij}, s\in[T]\}$ are independent random variables. By applying Bernstein inequality,
we have $M_{t,ij}\asymp t^{1-\alpha}/n_{\text{edge}}\asymp t^{1-\alpha}/(n^2p)$ with probability $1-O(T^{-3})$ when $t^{1-\alpha}\gtrsim n^2p\log T$.

\subsection{Proof of Lemma \ref{eigmin}}
Let $c=\lambda_{\min }(\bSigma_{1})$ and $C=\lambda_{\max}(\bSigma_{1})$.
First, we show that
for $\bSigma=\sum_{i,j\in[n],i> j}\bar{\bC}_{ij}$,
we can obtain $cn= \lambda_{\min,\bot}(\bSigma)\leq \|\bSigma\|_{2}\leq Cn$.
We have $\bSigma=\Ab\otimes \bSigma_{1}$, where
\begin{align*}
\Ab = \begin{bmatrix}
	n-1 & -1 & \cdots & -1 \\
	-1 & n-1 & \cdots & -1 \\
	\vdots & \vdots & \ddots & \vdots \\
	-1 & -1 & \cdots & n-1
\end{bmatrix}.
\end{align*}
Notice that $\Ab$ is a real symmetric matrix with eigenvalues: $n$ ($n-1$ multiplicity) and 0. Hence, $\|\bSigma\|_{2}=\|\Ab\|_{2} \|\bSigma_{1}\|_{2}\leq Cn$.

Let $\Rb$ be an $(n-1)\times n$ matrix satisfying $\Rb\Rb^{\top}=\Ib_{n-1}$ and $\Rb\bm{1}_{n}=\bm{0}$, where $\bm{1}_{n}$ is the vector of ones. Let $\Ob:=\Rb\otimes \Ib_{d}$ be an $(n-1)d\times nd$ matrix such that $\Ob\Ob^{\top}=\Ib_{(n-1)d}$ and the row space of $\Ob$ is $\Theta$. Then we have \begin{align*}
\lambda_{\min,\perp}(\bSigma)=&\lambda_{\min}(\Ob\bSigma \Ob^{\top})
=\lambda_{\min}\Big((\Rb\otimes \Ib_{d})(\Ab\otimes \bSigma_{1})(\Rb^{\top}\otimes \Ib_{d})\Big)
=\lambda_{\min} \Big((\Rb \Ab\Rb^\top)\otimes(\bSigma_{1})\Big).
\end{align*}
The minimum eigenvalue of $\Rb \Ab\Rb^{\top}$ is $n$ since the smallest eigenvalue of $\Ab$ has eigenvector $\bm{1}_{n}$. Therefore, $\lambda_{\min,\perp}(\bSigma)= cn$.

We now turn to $\bar{L}_{\mathcal{G}}$.
Let $\Ob_{1}$ be any $(n-1)d \times nd$ matrix with orthonormal rows such that $\Ob_{1}\Ob_{1}^{\top}=\Ib_{(n-1)d}$ and the row space of $\Ob_{1}$ is $\Theta$. Then, it follows that
\begin{align*}\lambda_{\min , \perp}\left(\bar{L}_{\mathcal{G}}\right)=\lambda_{\min }\left(\Ob_{1} \bar{L}_{\mathcal{G}} \Ob_{1}^{\top}\right)
\end{align*}
Let $
X_{i, j}=\Ob_{1}\bar{\bC}_{ij} \Ob_{1}^{\top} \mathbf{1}((i, j) \in \mathcal{E})
$ for $i>j$. Then we have $\Ob_{1} \bar{L}_{\mathcal{G}} \Ob_{1}^{\top}=\sum_{i>j} X_{i, j}$, where $X_{i, j} \succeq \bm{0}$ and $\left\|X_{i, j}\right\| =2\left\|\bSigma_{1}\right\|_2\leq 2C.$
Further, we can obtain
\begin{align*}
& \lambda_{\min }\Big(\mathbb{E} \sum_{i>j} X_{i, j}\Big)=\lambda_{\min }\left(p \Ob_{1} \bSigma \Ob_{1}^{\top}\right)=p \lambda_{\min , \perp}(\bSigma) = c p n .
\end{align*}
Using Theorem 5.1.1 in \cite{tropp15}, we have
\begin{align*}
& \mathbb{P}\Big(\lambda_{\min }\Big(\sum_{i>j} X_{i, j}\Big) \leq \frac{1}{2} \lambda_{\min }\Big(\mathbb{E} \sum_{i>j} X_{i, j}\Big)\Big) \leq(n-1)d \cdot (\frac{2}{e})^{c p n / (4C)}.
\end{align*}
As a result, if $p n>c \log n$ for $c\geq \frac{44C}{c(1-\log 2)}$, we have
\begin{align*}
\mathbb{P}\Big(\frac{1}{2} c p n \leq \lambda_{\min }\Big(\sum_{i>j} X_{i, j}\Big) \Big) \geq 1-O(n^{-10}) .
\end{align*}

Similarly, let $Y_{i, j}=\bar{\bC}_{ij} \mathbf{1}((i, j) \in \mathcal{E})$ for $i>j$. We can obtain $\bar{L}_{\mathcal{G}}=\sum_{i>j} Y_{i, j}$, where $Y_{i, j} \succeq \bm{0}$ and $\left\|Y_{i, j}\right\| =2\left\|\bSigma_{1}\right\|_2 \leq 2C$.
We have the following bound of eigenvalues:
\begin{align*}
\lambda_{\max }\Big(\mathbb{E} \sum_{i>j} Y_{i, j}\Big)=p \lambda_{\max }(\bSigma) \leq p n C.
\end{align*}
This further gives
\begin{align*}
\mathbb{P}\Big(\lambda_{\max }\Big(\sum_{i>j} Y_{i, j}\Big) \geq \frac{3}{2} \lambda_{\max }\Big(\mathbb{E} \sum_{i>j} Y_{i, j}\Big)\Big) \leq nd\cdot (\frac{8e}{27})^{p n/4 } .
\end{align*}
Therefore, if $p n>c \log n$ for $c\geq \frac{44}{\log 27-\log 8e}$, we have
\begin{align*}\mathbb{P}\Big(\lambda_{\max }\Big(\sum_{i>j} Y_{i, j}\Big) \leq \frac{3}{2} p n\Big) \geq 1-O(n^{-10}) .
\end{align*}

\section{Proof of Lemmas in Section~\ref{sec:Hes}}\label{app:sec:Hes}
\subsection{Proof of Lemma~\ref{lem:hes}}
Recall the definitions $\bSigma^{\sharp}_{i}(\btheta)=\EE_{\bX}[\bm{1}(\argmax_j \bX^\top\btheta_j=i)\bX\bX^\top]$ and $\bSigma_{i}^*=\EE[\bm{1}(\argmax_j(\bX^\top\btheta_j^{*})=i)\bX\bX^\top]$.
We have
\begin{align}\label{Sigor}
&\|\bSigma^{\sharp}_{i}(\bar{\btheta})-\bSigma_{i}^*\|_2^2\nonumber\\
&=\Big\|\EE_{\bX}\Big[\Big(\bm{1}(\argmax_j \bX^\top  \bar{\btheta}_j=i)-\bm{1}(\argmax_j \bX^\top\btheta_j^{*}=i)\Big)\bX\bX^{\top}\Big]\Big\|_2^2\nonumber\\
&= \Big(\max_{\|\bv\|_2=1} \EE_{\bX}\Big[\Big(\bm{1}(\argmax_j \bX^\top \bar{\btheta}_j=i)-\bm{1}(\argmax_j \bX^\top\btheta_j^{*}=i)\Big)\bv^{\top}\bX\bX^{\top}\bv\Big]\Big)^2\nonumber\\
&\leq  \EE_{\bX}\Big[\Big(\bm{1}(\argmax_j \bX^\top \bar{\btheta}_j=i)-\bm{1}(\argmax_j \bX^\top\btheta_j^{*}=i)\Big)^2\Big]
\max_{\|\bv\|_2=1} \EE_{\bX}\Big[(\bv^{\top}\bX\bX^{\top}\bv)^2\Big].
\end{align}
The first expectation term concerns the appearance probability of item $i$, which can be decomposed as 
\begin{align}\label{Edif}
&\EE_{\bX}\Big[\Big(\bm{1}(\argmax_j \bX^\top \bar{\btheta}_j=i)-\bm{1}(\argmax_j \bX^\top\btheta_j^{*}=i)\Big)^2\Big]\nonumber\\
&=\PP\Big(\argmax_j \bX^\top \bar{\btheta}_j=i\Big)+\PP\Big(\argmax_j \bX^\top\btheta_j^{*}=i\Big)
-2\PP\Big(\argmax_j \bX^\top \bar{\btheta}_j=i\text{ and }\argmax_j \bX^\top\btheta_j^{*}=i\Big).
\end{align}
Notice that
\begin{align}\label{ii}
&\PP(\argmax_{k\in[n]}\bX^{\top}\bar{\btheta}_{ k}=i)-\PP(\argmax_{k\in[n]}\bX^{\top}\bar{\btheta}_{ k}=i,\argmax_{k\in[n]}\bX^\top\btheta_{ k}^{*}=i)\nonumber\\
&= \PP(\argmax_{k\in[n]}\bX^{\top}\bar{\btheta}_{ k}=i,\argmax_{k\in[n]}\bX^\top\btheta_{ k}^{*}\neq i)\nonumber\\
&= \PP(\bX^{\top}\bar{\btheta}_{i}-\max_{j^\prime\in[n]\setminus\{i\}}\bX^{\top}\bar{\btheta}_{j^\prime}>0,\, \bX^\top\btheta^{*}_{i}-\max_{j^\prime\in[n]\setminus\{i\}}\bX^\top\btheta^{*}_{j^\prime}<-2c_{\bX}\epsilon_0)\nonumber\\
&\quad+\PP(\bX^{\top}\bar{\btheta}_{i}-\max_{j^\prime\in[n]\setminus\{i\}}\bX^{\top}\bar{\btheta}_{j^\prime}>0,\, 0>\bX^\top\btheta^{*}_{i}-\max_{j^\prime\in[n]\setminus\{i\}}\bX^\top\btheta^{*}_{j^\prime}\geq -2c_{\bX}\epsilon_0).
\end{align}
Since $|\bX^{\top}\bar{\btheta}_{i}-\max_{j^\prime}\bX^{\top}\bar{\btheta}_{j^\prime}-\bX^\top\btheta^{*}_{i}+\max_{j^\prime}\bX^\top\btheta^{*}_{j^\prime}|
\leq 2\max_{j}|(\bar{\btheta}_{j}-\btheta^{*}_{j})^\top\bX|\leq 2c_{\bX}\epsilon_{0},$
the first term of (\ref{ii}) is $0$. The second term
\begin{align*}
&\PP(\bX^{\top}\bar{\btheta}_{i}-\max_{j^\prime\in[n]\setminus\{i\}}\bX^{\top}\bar{\btheta}_{j^\prime}>0,\, 0>\bX^\top\btheta^{*}_{i}-\max_{j^\prime\in[n]\setminus\{i\}}\bX^\top\btheta^{*}_{j^\prime}\geq -2c_{\bX}\epsilon_0)\nonumber\\
&\leq \PP(0<\max_{j^\prime\in[n]\setminus\{i\}}\bX^\top\btheta^{*}_{j^\prime}-\bX^\top\btheta^{*}_{i}\leq 2c_{\bX}\epsilon_0)\leq 2c_{\bX}C\epsilon_{0}
\end{align*}
using Assumption~\ref{ass:dif}. 
Therefore, we have
\begin{align*}
&\PP(\argmax_{k\in[n]}\bX^{\top}\bar{\btheta}_{ k}=i)-\PP(\argmax_{k\in[n]}\bX^{\top}\bar{\btheta}_{ k}=i,\argmax_{k\in[n]}\bX^\top\btheta_{ k}^{*}=i)\leq 2c_{\bX}C\epsilon_{0}.
\end{align*}
Similarly, we can obtain
\begin{align*}
&\PP(\argmax_{k\in[n]}\bX^\top\btheta_{ k}^{*}=i)-\PP(\argmax_{k\in[n]}\bX^\top\btheta_{ k}^{*}=i,\argmax_{k\in[n]}\bX^{\top}\bar{\btheta}_{ k}=i)\leq 2c_{\bX}C\epsilon_{0}.
\end{align*}

Combining the above results, we have
\begin{align*}
\|\bSigma^{\sharp}_{i}-\bSigma_{i}^*\|_2^2\leq & 4Cc_{\bX}\epsilon_{0}\lesssim \epsilon_0.
\end{align*}
Proof regarding $|\tilde{\bSigma}^{\sharp}_{i}(\bar{\btheta})
-\tilde{\bSigma}_{i}|$ is similar, so we omit it.

\subsection{Proof of Lemma~\ref{app:lem:estil}}
Let $X_\ell=\bSigma_{\ell i}-\bSigma^{\sharp}_{i}(\hat{\btheta}(\ell-1))$ for $\ell\in\cS_t$. Note that $\EE\|X_\ell\|_2\leq2c_{\bX}$ and $\EE[X_\ell \given \cH_{\ell-1}]=0$ for $\ell\in\cS_t$. Hence, $\{X_\ell\}_{\ell\in\cS_t}$ is a martingale difference sequence. We can obtain $\lambda_{\max}(X_\ell)\leq 2c_{\bX}$ and $\|\sum_{\ell\in\cS_t}\EE[X_\ell^2\given \cH_{\ell-1}]\|_2\leq 4|\cS_t|c_{\bX}^2$. Applying Lemma~\ref{FreedmanRec}, we have 
\begin{align*}
\|\sum_{\ell\in\cS_t}(\bSigma_{\ell i}-\bSigma^{\sharp}_{i}(\hat{\btheta}(\ell-1)))\|_2\leq 4c_{\bX}\sqrt{6|\cS_t|\log T}
\end{align*}
holds with probability $1-O(T^{-3})$.
Proof regarding $\Big|\sum_{\ell\in\cS_t}\Big(\tilde{\bSigma}_{\ell i}-\tilde{\bSigma}^{\sharp}_{i}(\hat{\btheta}(\ell-1))\Big)\Big|$ is similar, so we omit it.

\subsection{Proof of Lemma~\ref{lem:covT}}
We consider $\lambda_{\min}(\bSigma_{i}(t))$ first. Let $t_0$ be the $(cn^2p\log T)^{1/(1-\alpha)}$ in Theorem~\ref{thm:2rate}. We have the following decomposition of $\bSigma_{i}(t)$:
\begin{align}\label{al}
\bSigma_{i}(t)
&=\sum_{\ell\in [t]\setminus\cS_{t}}\bSigma_{\ell i}+\sum_{\ell\in \cS_{t}\cap [t_0]}\bSigma_{\ell i}+\sum_{\ell\in \cS_{t}\setminus [t_0]}\bSigma_{\ell i}.
\end{align}
We can further decompose the last term as follows:
\begin{align}\label{de}
\sum_{\ell\in \cS_{t}\setminus [t_0]}\bSigma_{\ell i}
=\sum_{\ell\in \cS_{t}\setminus [t_0]}\bSigma_{i}^*
+\sum_{\ell\in \cS_{t}\setminus [t_0]}\Big(\bSigma_{\ell i}-\bSigma^{\sharp}_{i}(\hat{\btheta}(\ell-1))\Big)
+\sum_{\ell\in \cS_{t}\setminus [t_0]}\Big(\bSigma^{\sharp}_{i}(\hat{\btheta}(\ell-1))-\bSigma_{i}^*\Big).
\end{align}

We first prove $\cA_{T_0}^{1,1}\cap \cA_{T_0}^{1,2}$ holds with high probability. Applying Lemma~\ref{app:lem:estil}, with probability $1-\tilde{C}_1T^{-3}$, we have
\begin{align}\label{de1}
&\Big\|\sum_{\ell\in \cS_{T_0}}\Big(\bSigma_{\ell i}-\bSigma^{\sharp}_{i}(\hat{\btheta}(\ell-1))\Big)\Big\|_2
\leq
4c_{\bX}\sqrt{6|\cS_{T_0}|\log T}.
\end{align}
Note that $|\cS_{T_0}|\asymp T_0$. Therefore, there exists a constant $c$, for $T_0\geq cn^2\log T$, we have $\cA_{T_0}^{1,1}$ holds.
Then we consider $\cA_{T_0}^{1,2}$. Applying Lemma~\ref{lem:hes}, we have the following bound:
\begin{align}\label{Sig2}
&\|\sum_{\ell\in \cS_{T_0}\setminus [t_0]}\Big(\bSigma^{\sharp}_{i}(\hat{\btheta}(\ell-1))-\bSigma_{i}^*\Big)\|_2
\lesssim\sum_{\ell\in \cS_{T_0}\setminus [t_0]}\max_{k\in[n]}\|\hat{\btheta}_k(\ell-1)-\ths_k\|_2^{1/2}\nonumber\\
 &
\lesssim\sum_{\ell\in \cS_{T_0}\setminus [t_0]}(\sqrt{\frac{\log T}{M}}n\ell^{\alpha-\frac{1}{2}})^{\frac{1}{2}}\asymp (\frac{\log T}{M})^{\frac{1}{4}}T_0^{\frac{\alpha}{2}+\frac{3}{4}}n^{\frac{1}{2}}
\end{align}
with probability $1-O(T_0/T^3)$.
There exists a constant $c$, such that if $T_0\geq cn^{6/(1-2\alpha)}(\frac{\log T}{M})^{1/(1-2\alpha)}$, we have $\cA_{T_0}^{1,2}$ holds with probability $1-\tilde{C}_3T_0/T^3$, where $\tilde{C}_3$ is a constant.

Conditioning on $\cA_{s-1}^{1,1}\cap \cA_{s-1}^{1,2}$, we now consider $\cA_{s}^{1,1}$ and $\cA_{s}^{1,2}$.
As for $\cA_{s}^{1,1}$, we still utilize Lemma~\ref{app:lem:estil}, which yields
\begin{align*}
&\Big\|\sum_{\ell\in \cS_{s}}\Big(\bSigma_{\ell i}-\bSigma^{\sharp}_{i}(\hat{\btheta}(\ell-1))\Big)\Big\|_2
\leq
4c_{\bX}\sqrt{2|\cS_{s}|\log (nT)}.
\end{align*}
Therefore, $\cA_{s}^{1,1}$ holds with probability $1-\tilde C_3/T^{3}$.
For $\cA_{s}^{1,2}$, we have 
\begin{align*}
\|\Big(\bSigma^{\sharp}_{i}(\hat{\btheta}(s-1))-\bSigma_{i}^*\Big)\|_2&\leq \max_{k\in[n]}\|\hat{\btheta}_k(s-1)-\ths_k\|_2^{1/2}
\lesssim (\frac{n\log T}{ptM})^{1/4}
\end{align*}
with probability $1-\tilde C_2/T^3$.
Combining the properties on event $\cA_{s-1}^{1,1}\cap \cA_{s-1}^{1,2}$, there exists $c$, such that if $T_0\geq cn^{5}\log T/(pM)$, we have
\begin{align*}
&\lambda_{\min}(\sum_{\ell\in \cS_{s}\setminus [t_0]}\bSigma_{ij}^*)/2
-\|\sum_{\ell\in \cS_{s}\setminus [t_0]}\Big(\bSigma^{\sharp}_{ij}(\hat{\btheta}(\ell-1))-\bSigma_{ij}^*\Big)\|_2\\
&\geq  \lambda_{\min}(\sum_{\ell\in \cS_{s-1}\setminus [t_0]}\bSigma_{ij}^*)/2
-\|\sum_{\ell\in \cS_{s-1}\setminus [t_0]}\Big(\bSigma^{\sharp}_{ij}(\hat{\btheta}(\ell-1))-\bSigma_{ij}^*\Big)\|_2 + \frac{\tilde{c}_0}{2}- \|\Big(\bSigma^{\sharp}_{ij}(\hat{\btheta}(s-1))-\bSigma_{ij}^*\Big)\|_2
\nonumber\\
&\geq \frac{3s\tilde{c}_0}{8}.
\end{align*}
Hence, we have $\cA_{s}^{1,1}\cap \cA_{s}^{1,2}$ holds with probability $1-(\tilde C_2+\tilde C_3)/T^3$ conditioning on $\cA_{s-1}^{1,1}\cap \cA_{s-1}^{1,2}$. Therefore, the event $\cA_{s}^{1,1}\cap \cA_{s}^{1,2}$ holds with probability $1-(\tilde{C}_3T_0+\tilde{C}_2(s-T_0)+\tilde{C}_1(s-T_0))/T^3$.

We can decompose $\tilde{\bSigma}_{ij}(T_0)$ in a similar way as in (\ref{al}) and (\ref{de}). There exists $c$, when $T_0\geq c\max\{(n^2p)^{1/\alpha}, (n^2p)^{\frac{2-\alpha}{1-\alpha}}(\log T)^{\frac{1}{1-\alpha}}\}$, we have $T_0\gtrsim n^2p(T_0^{1-\alpha}+t_0)$. Therefore, for $t\geq T_0$, we can obtain \begin{align}\label{ft}
\sum_{\ell\in [t]\setminus\cS_{t}}\tilde{\bSigma}_{\ell ij}+\sum_{\ell\in \cS_{t}\cap [t_0]}\tilde{\bSigma}_{\ell ij}\leq t\tilde{c}_1/4.
\end{align}
 Following the same procedure as proving $\cA_t^{1,1}$ and $\cA_t^{1,2}$, we have
\begin{align}
&|\sum_{\ell\in \cS_{t}\setminus [t_0]}\Big(\tilde{\bSigma}_{\ell ij}-\tilde{\bSigma}^{\sharp}_{ij}(\hat{\btheta}(\ell-1))\Big)|\leq t\tilde{c}_1/8,\label{lam3}\\
&|\sum_{\ell\in \cS_{t}\setminus [t_0]}\Big(\tilde{\bSigma}^{\sharp}_{ij}(\hat{\btheta}(\ell-1))-\tilde{\bSigma}_{ij}\Big)|\leq t\tilde{c}_1/8.\label{lam4}
\end{align}
We conclude the proof combining (\ref{ft}), (\ref{lam3}) and (\ref{lam4}).

\subsection{Proof of Lemma~\ref{lem:Sigij}}
Let $\bb_\ell(i)$ represent the sample of $\bb(i)$ at time $\ell$. Note that $\bb_\ell(i)$ is independent of $\cH_{\ell-1}$ and $\bx_\ell$. We have 
\begin{align*}
\bSigma_{ij}(t)=&\sum_{\ell\in[t]}\bm{1}(\ba_{\ell}=(i,j))\bX_{\ell}\bX_{\ell}^\top=\sum_{\ell\in[t]}\left(\bm{1}(i_\ell=i)\bm{1}(\bb_\ell(i)=j)+\bm{1}(i_\ell=j)\bm{1}(\bb_\ell(j)=i)\right)\bX_{\ell}\bX_{\ell}^\top.
\end{align*}
We let $d_i$ represent the degree of node $i$ in a given $\cG$ and define the following random matrix
\begin{align*}
Z_\ell =\bm{1}(i_\ell=i)(\bm{1}(\bb_\ell(i)=j)-\frac{1}{d_i})\bX_\ell\bX_\ell^\top.
\end{align*}
Note that $\EE[Z_\ell \given \cH_{\ell-1},\bX_{\ell}]=\bm{1}(i_\ell=i)\EE[\bm{1}(\bb_\ell(i)=j)-\frac{1}{d_i} \given \cH_{\ell-1},\bX_{\ell}]=0$ and $\|Z_\ell\|_2\leq c_{\bX}$, so $Z_\ell$ is a martingale difference sequence. Further, we can obtain
\begin{align*}
\|\sum_{\ell=1}^{t} \EE[Z_\ell Z_\ell^\top \given \cH_{\ell-1},\bX_{\ell}]\|_2=&\|\sum_{\ell=1}^{t} \EE[\bm{1}(i_\ell=i)(\bm{1}(\bb_\ell(i)=j)-\frac{1}{d_i})^2\bX_\ell\bX_\ell^\top\bX_\ell\bX_\ell^\top \given \cH_{\ell-1},\bX_{\ell}]\|_2\nonumber\\
=&\|\sum_{\ell=1}^{t}\bm{1}(i_\ell=i)\bX_\ell\bX_\ell^\top\bX_\ell\bX_\ell^\top \EE[(\bm{1}(\bb_\ell(i)=j)-\frac{1}{d_i})^2 \given \cH_{\ell-1},\bX_{\ell}]\|_2\nonumber\\
\leq &\frac{3t\tilde{c}_1}{2}c_{\bX}\frac{1}{d_i}.
\end{align*}
Using Lemma~\ref{FreedmanRec}, with probability $1-O(n^{-2}T^{-3})$, we have 
\begin{align*}
\|\sum_{\ell=1}^{t}Z_\ell\|_2
\leq 3\sqrt{\frac{2\tilde{c}_1 c_{\bX}t\log (nT)}{d_i}}+4c_{\bX}\log (nT).
\end{align*}
Therefore, with probability $1-O(n^{-2}T^{-3})$, we have
\begin{align*}
\|\bSigma_{ij}(t)-\frac{1}{d_i}\bSigma_{i}(t)-\frac{1}{d_j}\bSigma_{j}(t)\|_2\leq 2(3\sqrt{\frac{2\tilde{c}_1 c_{\bX}t\log (nT)}{d_i}}+4c_{\bX}\log (nT)).
\end{align*}

Since $t\gg n^2p\log (nT)$, we have $\lambda_{\min}(\bSigma_{ij}(t))\geq t\tilde{c}_0/(2np)$. Similarly, we can obtain $\tilde{\bSigma}_{ij}(t)\leq 13t\tilde{c}_1/(2np)$.

\subsection{Proof of Lemma \ref{Hessmin}}
Recall that we have the Hessian as follows:
\begin{align*}
\nabla^{2} \cL_t(\btheta)=&\sum_{(i,j)\in \mathcal{E},i>j}M\sum_{\ell\in \cT_{t,ij}}\frac{e^{\bX_{\ell }^\top \btheta_{i}} e^{\bX_{\ell }^\top \btheta_{j}}}{(e^{\bX_{\ell }^\top \btheta_{i}}+e^{\bX_{\ell }^\top \btheta_{j}})^{2}}(\bc_{\ell i}-\bc_{\ell j})(\bc_{\ell i}-\bc_{\ell j})^\top.
\end{align*}
We first consider $L_{\mathcal{G}}=M\sum_{(i,j)\in \mathcal{E},i>j}\sum_{\ell\in\cT_{t,ij}}(\bc_{\ell i}-\bc_{\ell j})(\bc_{\ell i}-\bc_{\ell j})^\top.$
We have
\begin{align*}
L_{\mathcal{G}}=&M\sum_{(i,j)\in \mathcal{E},i>j;}\sum_{\ell=1}^{t}\bm{1}(\ba_\ell=(i,j))\left((\be_i-\be_j)\otimes\bX_\ell\right)\left((\be_i-\be_j)\otimes\bX_\ell\right)^\top\nonumber\\
=&M\sum_{(i,j)\in \mathcal{E},i>j;}\sum_{\ell=1}^{t}\bm{1}(\ba_\ell=(i,j))\left((\be_i-\be_j)(\be_i-\be_j)^\top\right)\otimes(\bX_{\ell}\bX_{\ell}^\top)\nonumber\\
=&M\sum_{(i,j)\in\cE,i>j} \Big(\Ab_{ij}\otimes(\sum_{\ell\in\cT_{t,ij}}\bX_\ell\bX_\ell^\top)\Big)
\end{align*} 
Conditioning on event $\cA_t^2$, the matrix $\sum_{\ell\in\cT_{t,ij}}\bX_\ell\bX_\ell^\top -t\tilde{c}_{0}/(2np)\Ib_{d}$ is positive semidefinite for all $(i,j)\in\cE$. Therefore, the matrix $\sum_{(i,j)\in\cE,i>j} (\Ab_{ij}\otimes(\sum_{\ell\in\cT_{t,ij}}\bX_\ell\bX_\ell^\top -t\tilde{c}_{0}/(2np)\Ib_{d}))$ is positive semidefinite. We can obtain
\begin{align}\label{lammin}
\lambda_{\min, \perp}(L_{\cG})&=M\lambda_{\min, \perp}\Big(\sum_{(i,j)\in\cE,i>j} (\Ab_{ij}\otimes(\sum_{\ell\in\cT_{t,ij}}\bX_\ell\bX_\ell^\top))\Big) \nonumber\\
&\geq M\lambda_{\min, \perp}\Big(\sum_{(i,j)\in\cE,i>j} \Ab_{ij}\otimes t\tilde{c}_{0}/(2np)\Ib_{d}\Big)\nonumber\\
&\geq \tilde{c}_{0}tM/2,
\end{align}
where the last inequality utilizes the property of $\cA_0$.

Since $\|\bX\|_2\leq c_{\bX}$ and $\|\ths_i-\ths_j\|_2\leq c_{\btheta}$,
we have
\begin{align*}
|\bX_{\ell}^{\top}\btheta_{ i}-\bX_{\ell}^{\top}\btheta_{j}|
\leq& |\bX_{\ell}^\top\btheta_{ i}^{*}-\bX_{\ell}^\top\btheta_{j}^{*}|
+|\bX_{\ell}^\top\btheta_{ i}^{*}-\bX_{\ell}^{\top}\btheta_{i}|
+|\bX_{\ell}^\top\btheta_{ j}^{*}-\bX_{\ell}^{\top}\btheta_{j}|\nonumber\\
\leq& \|\ths_i-\ths_j\|_2\|\bX_{\ell}\|_2+2\max_{i\in[n]}\|\btheta_{i}-\btheta_{i}^{*}\|_{2}\|\bX_\ell\|_2\nonumber\\
\leq &2(c_{\btheta}+\epsilon_{1})c_{\bX}.
\end{align*}
Further, we can obtain $$
\frac{e^{\bX_{\ell}^\top  \btheta_{i}} e^{\bX_{\ell}^\top  \btheta_{j}}}{(e^{\bX_{\ell}^\top  \btheta_{i}}+e^{\bX_{\ell}^\top  \btheta_{j}})^{2}}\geq \frac{1}{4}e^{-|\bX_{\ell}^{\top}\btheta_{ i}-\bX_{\ell}^{\top}\btheta_{j}|}\geq \frac{1}{4}e^{-2(c_{\btheta}+\epsilon_{1})c_{\bX}}
\geq \frac{1}{4\kappa e^{2\epsilon_1c_{\bX}}}.$$
Combining (\ref{lammin}), we have $\lambda_{\min, \perp}\left(\nabla^2 \mathcal{L}(\btheta)\right) \geq \tilde{c}_0tM/(8\kappa e^{2\epsilon_1c_{\bX}}).$
Similarly, we have $\|L_{\cG}\|_2\leq 39\tilde{c}_1tM/4$, so $\lambda_{\max }\left(\nabla^2 \mathcal{L}_{\lambda,t}(\btheta)\right) \leq \lambda+\frac{1}{4}\|L_{\mathcal{G}}\|_{2}\leq \lambda+3\tilde{c}_1tM.$

\section{Proof of Lemmas in Section~\ref{app:sec:err2}}\label{app:sec:ratepf}
\subsection{Proof of Lemma~\ref{gr}}
Utilizing the definition of $\mathcal{L}_{\lambda,t}(\btheta)$, we have the closed form of its gradient,
\begin{align*}
\nabla \mathcal{L}_{\lambda,t}(\btheta^*)=&\sum_{(i,j)\in \mathcal{E},i>j;}\sum_{\ell\in \cT_{t,ij};}\sum_{m=1}^{M}\{-y_{ji}^{(m)}(\ell)+\frac{e^{\bX_{\ell}^\top  \btheta_{i}}}{e^{\bX_{\ell}^\top  \btheta_{i}}+e^{\bX_{\ell}^\top  \btheta_{j}}}\}(\bc_{\ell i}-\bc_{\ell j})
+\lambda \btheta^*\nonumber\\
=&\sum_{(i,j)\in \mathcal{E},i>j;}\sum_{\ell\in \cT_{t,ij};}\sum_{m=1}^{M}z_{ij}^{(m)}(\ell)+\lambda \btheta^*.
\end{align*}
For the given $\cG$, let $(\bar{i}_k,\bar{j}_k)$ represent the $k$-th pair in $\cE$. Consider the sequence\\ $\bZ=\{\bZ_{111},\ldots,\bZ_{11M},\bZ_{121}, \ldots,\bZ_{12M},\ldots, \bZ_{21M},\ldots,\bZ_{21M},\bZ_{221},\ldots,\bZ_{22M},\ldots,\bZ_{t11},\ldots,\\\bZ_{t1M},\bZ_{t21}\ldots,\bZ_{t2M},\ldots\}$, where $\bZ_{\ell km}=\bm{1}(\ba_\ell=(\bar{i}_k,\bar{j}_k))z_{\bar{i}_k,\bar{j}_k}^{(m)}(\ell)$ is a random vector. Notice that $\EE[\bZ_{\ell km}\given \cH_{\ell-1}, \bX_{\ell}]=0$ and $\bZ$ is a martingale difference sequence. The dimension of the elements in $\bZ$ is $1\times nd$.
Let $
\bW_{\text{col}}=\sum_{k}\sum_{\ell=1}^{t}\sum_{m=1}^{M}\EE(\bZ_{\ell km}\bZ_{\ell km}^\top\given \cH_{\ell-1}, \bX_{\ell})$
and $
\bW_{\text{row}}=\sum_{k}\sum_{\ell=1}^{t}\sum_{m=1}^{M}\EE(\bZ_{\ell km}^\top\bZ_{\ell km}\given \cH_{\ell-1}, \bX_{\ell}).$
Since $\mathbb{E}[z_{ij}^{(m)\top}(\ell)z_{ij}^{(m)}(\ell)\given \bX_{\ell}]\leq 2\bX_{\ell}^\top\bX_{\ell}$, we have
\begin{align*}
\bW_{\text{row}}=&\sum_{(i,j)\in\cE,i>j;}\sum_{\ell=1}^{t}\sum_{m=1}^{M}\EE\left(\bm{1}(\ba_\ell=(i,j))z_{ij}^{(m)\top}(\ell)z_{ij}^{(m)}(\ell)\given \cH_{\ell-1}, \bX_{\ell}\right)\nonumber\\
= &\sum_{(i,j)\in\cE,i>j;}\sum_{\ell=1}^{t}\sum_{m=1}^{M}\bm{1}(\ba_\ell=(i,j))\mathbb{E}[z_{ij}^{(m)\top}(\ell)z_{ij}^{(m)}(\ell)\given \bX_{\ell}]\nonumber\\
\leq & \sum_{(i,j)\in\cE,i>j;}\sum_{\ell\in \cT_{t,ij}} M\bm{\tilde{\bSigma}_{\ell ij}}\nonumber\\
\leq &13nt\tilde{c}_1M/2.
\end{align*}
where the last inequality relies on the property of event $\cA_t^2$.
Besides, notice that
\begin{align*}
\mathbb{E}[z_{ij}(\ell)z_{ij}(\ell)^{\top}\given \bX_{\ell}]=&\mathbb{E}\Big[\Big(
-y_{ji}(\ell)+\frac{e^{\bX_{\ell}^\top\btheta^{*}_{i}}}{e^{\bX_{\ell}^\top\btheta^{*}_{i}}+e^{\bX_{\ell}^\top\btheta^{*}_{j}}}
\Big)^2(\bc_{\ell i}-\bc_{\ell j})(\bc_{\ell i}-\bc_{\ell j})^{\top}\given \bX_{\ell}\Big]\nonumber\\
\preceq & (\bc_{\ell i}-\bc_{\ell j})(\bc_{\ell i}-\bc_{\ell j})^{\top},
\end{align*}
Combining Lemma~\ref{Hessmin}, we can obtain
\begin{align*}
\|\bW_{\text{col}}\|_{2}&=\Big\|\sum_{(i,j)\in\cE,i>j;}\sum_{\ell=1}^{t}\sum_{m=1}^{M}\EE\left(\bm{1}(\ba_\ell=(i,j)) z_{ij}^{(m)}(\ell)z_{ij}(\ell)^{(m)\top}\given \cH_{\ell-1}, \bX_{\ell}\right)\Big\|_{2}\nonumber\\
&=\Big\|\sum_{(i,j)\in\cE,i>j;}\sum_{\ell=1}^{t}\sum_{m=1}^{M}\bm{1}(\ba_\ell=(i,j))\mathbb{E}[z_{ij}^{(m)}(\ell)z_{ij}(\ell)^{(m)\top}\given \bX_{\ell}]\Big\|_{2}\nonumber\\
&=\Big\|\sum_{(i,j)\in\cE,i>j;}\sum_{\ell\in\cT_{t,ij}}M\bm{1}(\ba_\ell=(i,j))(\bc_{\ell i}-\bc_{\ell j})(\bc_{\ell i}-\bc_{\ell j})^{\top}\Big\|_{2}\nonumber\\
&=\|L_{\cG}\|_2\leq 2\tilde{c}_1tM.
\end{align*}

Set $\sigma^{2}=\max\{\bW_{\text{row}},\|\bW_{\text{col}}\|_{2}\}=2\tilde{c}_1ntM.$
Using Lemma~\ref{FreedmanRec}, we have 
\begin{align}\label{mar}
&\left\|\nabla \mathcal{L}_{\lambda}(\btheta^*)-\lambda\btheta^*\right\|
\leq\Big\|\sum_{(i,j)\in \mathcal{E},i>j}\sum_{\ell\in\cT_{t,ij}}z_{ij}(\ell)\Big\|_2
\lesssim \sqrt{\tilde{c}_1ntM\log T}
\end{align}
holds with probability $1-O(T^{-3})$.
Further, with probability $1-O(T^{-3})$, we have 
\begin{align*}
\left\|\nabla \cL_{\lambda,t}\left(\btheta^* \right)\right\|_2 \leq \sqrt{\tilde{c}_1ntM\log T}+\lambda\left\|\btheta^*\right\|_2
\lesssim\sqrt{\tilde{c}_1ntM\log T}.
\end{align*}

\subsection{Proof of Lemma \ref{thetaT}}
First, utilizing Lemma~\ref{Hessmin}, we can obtain
\begin{align}\label{thetas}
\left\|\btheta^s-\hat{\btheta}\right\|_2 \leq \rho^{s/2}\left\|\btheta^0-\hat{\btheta}\right\|_2,
\end{align}
where $\rho=e^{-\lambda/(\lambda+2\tilde{c}_1tM)}$.
Actually, note that $\cL_{\lambda,t}\left(\btheta\right)$ is $(\lambda+2\tilde{c}_1tM)$-smooth and $\lambda $-strong convex by definition.
Applying the convergence result of Theorem 3.10 in \cite{Bub} yields (\ref{thetas}).

Then we claim that with $\lambda=c_{\lambda}\sqrt{\tilde{c}_1tM\log T}$, we have
\begin{align*}
\left\|\btheta^0-\hat{\btheta}\right\|_2=\left\|\hat{\btheta}-\btheta^*\right\|_2 \lesssim \frac{1}{c_{\lambda}}\sqrt{n}.
\end{align*}
In fact, expanding $\cL_{\lambda,t}(\hth)$ at $\btheta^*$, we have
\begin{align*}
\cL_{\lambda,t}\left(\btheta^* \right) \geq \cL_{\lambda,t}(\hat{\btheta} )= & \cL_{\lambda,t}\left(\btheta^* \right)+\left\langle\nabla \cL_{\lambda,t}\left(\btheta^* \right), \hat{\btheta}-\btheta^*\right\rangle +\frac{1}{2}\left(\hat{\btheta}-\btheta^*\right)^{\top} \nabla^2 \cL_{\lambda,t}(\tilde{\btheta} )\left(\hat{\btheta}-\btheta^*\right)
\end{align*}
where $\tilde{\btheta}$ is between $\hat{\btheta}$ and $\btheta^*$. This together with the Cauchy-Schwartz inequality gives
\begin{align*}
\frac{1}{2}\left(\hat{\btheta}-\btheta^*\right)^{\top} \nabla^2 \cL_{\lambda,t}(\tilde{\btheta} )\left(\hat{\btheta}-\btheta^*\right) & \leq-\left\langle\nabla \cL_{\lambda,t}\left(\btheta^* \right), \hat{\btheta}-\btheta^*\right\rangle  \leq\left\|\nabla \cL_{\lambda,t}\left(\btheta^* \right)\right\|_2\left\|\hat{\btheta}-\btheta^*\right\|_2.
\end{align*}
The above inequality yields
\begin{align*}
\left\|\hat{\btheta}-\btheta^*\right\|_2 \leq \frac{2\left\|\nabla \cL_{\lambda,t}\left(\btheta^* \right)\right\|_2}{\lambda_{\min }\left(\nabla^2 \cL_{\lambda,t}(\tilde{\btheta} )\right)}.
\end{align*}
From the trivial lower bound $\lambda_{\min }\left(\nabla^2 \cL_{\lambda,t}(\tilde{\btheta} )\right) \geq \lambda$, the preceding inequality gives
\begin{align}\label{thetahs}
\left\|\hat{\btheta}-\btheta^*\right\|_2 &\leq \frac{2\left\|\nabla \cL_{\lambda,t}\left(\btheta^* \right)\right\|_2}{\lambda}\leq \frac{2c_{1}}{c_{\lambda}}\sqrt{n}.
\end{align}

Combining (\ref{thetas}) and (\ref{thetahs}), we have 
\begin{align}\label{thetaTtilde}
\quad \|\btheta^{\tilde{T}}  -\hat{\btheta} \|_2 \leq \rho^{\tilde{T}/2}\left\|\btheta^0-\hat{\btheta}\right\|_2 \leq \rho^{\tilde{T}/2} \frac{2c_{1}}{c_{\lambda}}\sqrt{n}=\exp\left(-\frac{\tilde{T}\lambda}{2\lambda+4\tilde{c}_1tM}\right) \frac{2c_{1}}{c_{\lambda}}\sqrt{n}.
\end{align}
Since $\lambda= c_{\lambda}\sqrt{\tilde{c}_1tM\log T}\lesssim \tilde{c}_1tM\log T$, we can obtain
\begin{align*}
(\ref{thetaTtilde})&<\exp\left(-\frac{\tilde{T}\lambda}{8\tilde{c}_1tM\log T}\right) \frac{2c_{1}}{c_{\lambda}}\sqrt{n} \nonumber= \frac{2c_{1}}{c_{\lambda}}\sqrt{n} \exp \left(-\frac{c_{\lambda}\tilde{T}}{8} \sqrt{\frac{1}{\tilde{c}_1tM\log T}}\right) \lesssim \sqrt{\frac{n\log T}{ptM}}.
\end{align*}

\subsection{Proof of Lemma~\ref{induc}}
Notice that (\ref{d}) holds automatically for $s=0$. We then present the proof in an inductive manner.
\begin{lemma}\label{t2}
Conditioning on event $\cA\cap \cA_0\cap\cA_t^2\cap\cA_t^3$, suppose (\ref{d}) holds for the $s$-th iteration. For $C_1\geq 10c_{1}$, we have 
\begin{align*}
\|\btheta^{s+1}-\btheta^{*}\|_{2}\leq 
\frac{C_1 \kappa}{\tilde{c}_0} \sqrt{\frac{n\tilde{c}_1 \log T}{tM}}.
\end{align*}
\end{lemma}

\begin{proof}[Proof of Lemma \ref{t2}]
We decompose $\btheta^{s+1}-\btheta^*$ as follows:
\begin{align*}
\btheta^{s+1}-\btheta^* & =\btheta^s-\eta \nabla \cL_{\lambda,t}\left(\btheta^s\right)-\btheta^* \\
& =\btheta^s-\eta \nabla \cL_{\lambda,t}\left(\btheta^s\right)-\left[\btheta^*-\eta \nabla \cL_{\lambda,t}\left(\btheta^*\right)\right]-\eta \nabla \cL_{\lambda,t}\left(\btheta^*\right) \\
& =\left\{\Ib_{nd}-\eta \int_0^1 \nabla^2 \cL_{\lambda,t}(\btheta(\tau)) \mathrm{d} \tau\right\}\left(\btheta^s-\btheta^*\right)-\eta \nabla \cL_{\lambda,t}\left(\btheta^*\right)\\
& =: \left\{\Ib_{nd}-\eta A \right\}\left(\btheta^s-\btheta^*\right)-\eta \nabla \cL_{\lambda,t}\left(\btheta^*\right),
\end{align*}
where we denote $\btheta(\tau)=\btheta^*+\tau\left(\btheta^s-\btheta^*\right)$. 
So we can derive that
\begin{align}\label{tp1}
\left\|\btheta^{s+1}-\btheta^*\right\|_2 \leq\left\|\left(\Ib_{nd}-\eta A\right)\left(\btheta^s-\btheta^*\right)\right\|_2+\eta\left\|\nabla \cL_{\lambda,t}\left(\btheta^*\right)\right\|_2 .
\end{align}
Notice that the iteration step can ensure $\{\btheta^{s}\}_{s=0,1,\ldots,\tilde{T}}\subset\Theta$. Utilizing the definition of $\lambda_{\min , \perp}$, the first term on the right hand side of (\ref{tp1}) is controlled by
\begin{align*}
\left\|\left(\Ib_{nd}-\eta A\right)\left(\btheta^s-\btheta^*\right)\right\|_2 \leq \max \left\{\left|1-\eta \lambda_{\min , \perp}(A)\right|,\left|1-\eta \lambda_{\max }(A)\right|\right\}\left\|\btheta^s-\btheta^*\right\|_2 .
\end{align*}
From Lemma~\ref{Hessmin}, we have
\begin{align*}
\frac{\tilde{c}_0tM}{10\kappa}+\lambda 
\leq \lambda_{\min, \perp}\left(\nabla^2 \cL_{\lambda,t}(\btheta)\right) 
\leq \lambda_{\max}\left(\nabla^2 \cL_{\lambda,t}(\btheta)\right)
\leq 3\tilde{c}_1tM+\lambda.
\end{align*}
Since $\eta\lambda_{\max}(A)\leq 1$, we reach
\begin{align}\label{IA}
\left\|\left(\Ib_{nd}-\eta A\right)\left(\btheta^s-\btheta^*\right)\right\|_2 \leq\left(1-\frac{\eta \tilde{c}_0tM}{10 \kappa} \right)\left\|\btheta^s-\btheta^*\right\|_2.
\end{align}
Hence, we conclude that
\begin{align*}
\left\|\btheta^{s+1}-\btheta^*\right\|_2 & \leq\left(1-\frac{\eta \tilde{c}_0tM}{10 \kappa} \right)\left\|\btheta^s-\btheta^*\right\|_2+\eta\left\|\nabla \cL_{\lambda,t}\left(\btheta^*\right)\right\|_2 \\
& \leq\left(1-\frac{\eta \tilde{c}_0tM}{10 \kappa} \right) \frac{C_1 \kappa}{\tilde{c}_0} \sqrt{\frac{\tilde{c}_1\log T}{ptM}}
+\eta c_{1} \sqrt{  \tilde{c}_1ntM\log T}\\
& \leq \frac{C_1 \kappa}{\tilde{c}_0} \sqrt{\frac{\tilde{c}_1 n\log T}{tM}}
\end{align*}
for constant $C_1\geq 10c_{1}$. 
\end{proof}

\begin{lemma}\label{looTr}
Conditioning on event $\cA\cap\cA_t^2$, suppose (\ref{d}) holds for the $s$-th iteration. We have 
\begin{align*}
\max_{q\in [n]}\|\btheta_{q}^{s+1,(q)}-\btheta_{q}^{*}\|_{2}\leq
\frac{C_{2}\kappa}{\tilde{c}_0} 
\sqrt{\frac{\tilde{c}_1\log T}{ptM}}
\end{align*}
for $C_{2}>c_{\lambda}+\frac{5}{2}\sqrt{\frac{3}{2}}\frac{\tilde{c}_1}{\tilde{c}_0}(C_{1}+C_{3}).$
\end{lemma}
\begin{proof}[Proof of Lemma \ref{looTr}]
From the definition of sequence $\{\btheta^{s,(q)}\}$, the $d$-dimensional vector
\begin{align}\label{loo}
&\btheta_q^{s+1,(q)}-\btheta_q^*\nonumber\\&=\btheta_q^{s,(q)}-\eta\left[\nabla \cL_{\lambda,t}^{(q)}\left(\btheta^{s,(q)}\right)\right]_q-\btheta_q^*\nonumber\\
&=\btheta_q^{s,(q)}-\btheta_q^*-\eta \lambda \btheta_q^{s,(q)}-\eta\Bigg[\sum_{j:(j,q)\in\cE;}\sum_{\ell\in\cT_{t,jq}}M\Big\{\frac{e^{\bX_{\ell}^\top\btheta_j^{*}}}{e^{\bX_{\ell}^\top\btheta_j^{*}}+e^{\bX_{\ell}^\top\btheta_q^{*}}}-\frac{e^{\bX_{\ell}^\top\btheta_j^{s,(q)}}}{e^{\bX_{\ell}^\top\btheta_j^{s,(q)}}+e^{\bX_{\ell}^\top\btheta_q^{s,(q)}}}\Big\}\bX_{\ell}\Bigg]\nonumber\\
&=\btheta_q^{s,(q)}-\btheta_q^*-\eta \lambda \btheta_q^{s,(q)}\nonumber\\
&\quad -\eta M\Bigg[\sum_{j:(j,q)\in\cE;}\sum_{\ell\in\cT_{t,jq}}\left\{-\frac{e^{c_{j\ell}}}{\left(1+e^{c_{j\ell}}\right)^2}\left[\bX_{\ell}^\top\btheta_q^{*}-\bX_{\ell}^\top\btheta_j^{*}-\left(\bX_{\ell}^\top\btheta_q^{s,(q)}-\bX_{\ell}^\top\btheta_j^{s,(q)}\right)\right]\bX_{\ell}\right\}\Bigg],
\end{align}
where $\left[\nabla \cL_{\lambda,t}^{(q)}\left(\btheta^{s,(q)}\right)\right]_q$ represents the rows corresponding to item $q$ and $c_{j\ell}$ is some real number lying between $\bX_{\ell}^\top\btheta_q^{*}-\bX_{\ell}^\top\btheta_j^{*}$ and $\bX_{\ell}^\top\btheta_q^{s,(q)}-\bX_{\ell}^\top\btheta_j^{s,(q)}$. 
The last term in (\ref{loo}) is equal to
\begin{align*}
&\eta M\sum_{j:(j,q)\in\cE;}\sum_{\ell\in\cT_{t,jq}}\left\{\frac{e^{c_{j\ell}}}{\left(1+e^{c_{j\ell}}\right)^2}\bX_{\ell} \bX_{\ell}^\top(\btheta_q^{s,(q)}-\btheta_q^{*})\right\}
+\eta M\sum_{j:(j,q)\in\cE;}\sum_{\ell\in\cT_{t,jq}}\left\{-\frac{e^{c_{j\ell}}}{\left(1+e^{c_{j\ell}}\right)^2}\bX_{\ell} \bX_{\ell}^\top(\btheta_j^{s,(q)}-\btheta_j^{*})\right\}.
\end{align*}
Therefore, we have
\begin{align*}
(\ref{loo})=&\Big\{(1-\eta\lambda)\Ib_{d}-\eta M\sum_{j:(j,q)\in\cE;}\sum_{\ell\in\cT_{t,jq}}\frac{e^{c_{j\ell}}}{\left(1+e^{c_{j\ell}}\right)^2}\bX_{\ell} \bX_{\ell}^\top\Big\}(\btheta_{q}^{s,(q)}-\btheta_{q}^{*})-\eta\lambda \btheta_{q}^{*}\nonumber\\
&+\eta M\sum_{j:(j,q)\in\cE;}\sum_{\ell\in\cT_{t,jq}}\left\{-\frac{e^{c_{j\ell}}}{\left(1+e^{c_{j\ell}}\right)^2}\bX_{\ell} \bX_{\ell}^\top(\btheta_j^{s,(q)}-\btheta_j^{*})\right\}.
\end{align*}
This yields the following result:
\begin{align}\label{looa}
\Big\|\btheta_q^{s+1,(q)}-\btheta_q^*\Big\|_{2}&\leq \Big\|(1-\eta\lambda)\Ib_{d}-\eta M\sum_{j:(j,q)\in\cE;}\sum_{\ell\in\cT_{t,jq}}\frac{e^{c_{j\ell}}}{\left(1+e^{c_{j\ell}}\right)^2}\bX_{\ell} \bX_{\ell}^\top\Big\|_{2}\left\|\btheta_{q}^{s,(q)}-\btheta_{q}^{*}\right\|_{2}\nonumber\\
&\quad +\eta M \Big\|\sum_{j:(j,q)\in\cE;}\sum_{\ell\in\cT_{t,jq}}\Big\{-\frac{e^{c_{j\ell}}}{\left(1+e^{c_{j\ell}}\right)^2}\bX_{\ell} \bX_{\ell}^\top(\btheta_j^{s,(q)}-\btheta_j^{*})\Big\}\Big\|_{2}+\eta\lambda \|\btheta_{q}^{*}\|_{2}.
\end{align}
Since $\frac{e^{c_{i\ell}}}{(1+e^{c_{i\ell}})^2}\leq \frac{1}{4}$, the second term of (\ref{looa}) is bounded by
\begin{align}\label{etaM}
&\frac{\eta}{4} M \sum_{j:(j,q)\in\cE}\Big\|\sum_{\ell\in\cT_{t,jq}}\bX_{\ell} \bX_{\ell}^\top\Big\|_{2}\left\|\btheta_j^{s,(q)}-\btheta_j^{*}\right\|_{2}
\leq \frac{\eta\tilde{c}_1tM}{4np} \sum_{j\in[n]\setminus\{q\}}A_{qj}\left\|\btheta_j^{s,(q)}-\btheta_j^{*}\right\|_{2},
\end{align}
where $A_{qj}=\bm{1}((q,j)\in\cE)$.
Applying Cauchy-Schwartz inequality, we have
\begingroup
\allowdisplaybreaks
\begin{align}\label{C13}
(\ref{etaM})\leq & \frac{\eta\tilde{c}_1t M}{4np}\sqrt{\sum_{j\in[n]\setminus\{q\}}A_{qj}}\|\btheta^{s,(q)}-\btheta^{*}\|_{2}\nonumber\\
\leq &\frac{\eta\tilde{c}_1tM }{4np}\sqrt{\sum_{j\in[n]\setminus\{q\}}A_{qj}}(\|\btheta^{s,(q)}-\btheta^{s}\|_{2}+\|\btheta^{s}-\btheta^{*}\|_{2})\nonumber\\
\leq &\frac{\eta\tilde{c}_1tM }{4\tilde{c}_0}(C_{1}+C_{3})\kappa\sqrt{\sum_{j\in[n]\setminus\{q\}}A_{qj}}\sqrt{\frac{\tilde{c}_1\log T}{np^2tM}}.
\end{align}
\endgroup

For the coefficient of the first term in (\ref{looa}), we have
\begin{align*}
\max_{i,\ell}|c_{i\ell}|
&\leq \max_{i:i\neq q} |\bX_{\ell}^\top\btheta^{*}_{q}-\bX_{\ell}^\top\btheta_{i}^{*}|+\max_{i:i\neq q}|\bX_{\ell}^\top\btheta_{q}^{*}-\bX_{\ell}^\top\btheta_{i}^{*}-(\bX_{\ell}^\top\btheta_{q}^{s,(q)}-\bX_{\ell}^\top\btheta_{i}^{s,(q)})|\nonumber\\
&\leq \max_{i:i\neq q} |\bX_{\ell}^\top\btheta^{*}_{q}-\bX_{\ell}^\top\btheta_{i}^{*}|
+\max_{i:i\neq q}\left(|(\btheta_{q}^{*}-\btheta_{q}^{s,(q)})^{\top}\bX_{\ell}|+|(\btheta_{i}^{*}-\btheta_{i}^{s,(q)})^{\top}\bX_{\ell}|\right).
\end{align*}
Utilizing the assumption on step $s$, we can obtain
\begin{align*}
\max_{i\in [n]}\|\btheta_{i}^{s,(q)}-\btheta_{i}^{*}\|_{2}&\leq \max_{i\in [n]}(\|\btheta_{i}^{*}-\btheta_{i}^{s}\|_{2}+\|\btheta_{i}^{s}-\btheta_{i}^{s,(q)}\|_{2})\nonumber\\
&\leq \max_{i\in [n]}\|\btheta_{i}^{*}-\btheta_{i}^{s}\|_{2}+\max_{q\in [n]}\|\btheta^{s}-\btheta^{s,(q)}\|_{2}\nonumber\\
&\leq \frac{(C_{3}+C_{4})\kappa}{\tilde{c}_0}
\sqrt{\frac{\tilde{c}_1\log T}{ptM}}.
\end{align*}
In addition, we have $\max_{i,l}|c_{i\ell}|\leq 2c_{\btheta}c_{\bX}+2\epsilon c_{\bX}$, as long as $\epsilon\geq \frac{(C_{3}+C_{4})\kappa}{\tilde{c}_0}
\sqrt{\frac{\tilde{c}_1\log T}{ptM}}$. Then we have
\begin{align*}
\frac{e^{c_{i\ell}}}{\left(1+e^{c_{i\ell}}\right)^2}=\frac{e^{-\left|c_{i\ell}\right|}}{\left(1+e^{-\left|c_{i\ell}\right|}\right)^2} \geq \frac{e^{-\left|c_{i\ell}\right|}}{4} \geq \frac{1}{5 \kappa}
\end{align*}
for small enough $\epsilon$.
Therefore, we can obtain \begin{align*}
&\lambda_{\min}\Big( M\sum_{j\in[n]\setminus\{q\}}A_{qj}\sum_{\ell\in\cT_{t,jq}}\frac{e^{c_{j\ell}}}{\left(1+e^{c_{j\ell}}\right)^2}\bX_{\ell} \bX_{\ell}^\top\Big) \geq  \frac{\tilde{c}_0tM\sum_{j\in[n]\setminus\{q\}}A_{qj}}{5\kappa np }\geq \frac{\tilde{c}_0tM}{10\kappa}.
\end{align*}
In addition, we have
$$\Big\|\eta M\sum_{j\in[n]\setminus\{q\}}A_{qj}\sum_{\ell\in\cT_{t,jq}}\frac{e^{c_{j\ell}}}{\left(1+e^{c_{j\ell}}\right)^2}\bX_{\ell} \bX_{\ell}^\top\Big\|_{2}\leq 10\eta\tilde{c}_{1} tM.$$ 
Since $1-\eta\lambda-10\eta\tilde{c}_{1} tM>0$, we can obtain 
\begin{align}\label{fst}
&\Big\|(1-\eta\lambda)\Ib_{d}-\eta M\Big[\sum_{j\in[n]\setminus\{q\}}A_{qj}\sum_{\ell\in\cT_{t,jq}}\frac{e^{c_{j\ell}}}{\left(1+e^{c_{j\ell}}\right)^2}\bX_{\ell} \bX_{\ell}^\top\Big]\Big\|_{2}\leq 1-\eta\lambda-\frac{\eta \tilde{c}_0tM}{10\kappa}.
\end{align}

Combining (\ref{etaM}), (\ref{C13}) and (\ref{fst}), we have
\begin{align*}
&\|\btheta_q^{s+1,(q)}-\btheta_q^*\|_{2}\nonumber\\
&\leq \Big(1-\eta\lambda-\frac{\eta \tilde{c}_0tM}{10\kappa}\Big)
\frac{C_{2}\kappa^2}{\tilde{c}_0}
\sqrt{\frac{\tilde{c}_1\log T}{ptM}}
+\frac{\eta\tilde{c}_1 tM }{4\tilde{c}_0}\sqrt{\frac{3np}{2}}(C_{1}+C_{3})\kappa\sqrt{\frac{\tilde{c}_1\log T}{np^2tM}}+\eta c_{\lambda} \sqrt{\tilde{c}_1tM\log T}\nonumber\\
&\leq  \frac{C_{2}\kappa^2}{\tilde{c}_0}\sqrt{\frac{\tilde{c}_1\log T}{ptM}}
\end{align*}
for $C_{2}>c_{\lambda}+\frac{5}{2}\sqrt{\frac{3}{2}}\frac{\tilde{c}_1}{\tilde{c}_0}(C_{1}+C_{3}).$
\end{proof}

\begin{lemma}\label{ttr}
Conditioning on event $\cA\cap \cA_0\cap\cA_t^2$, suppose (\ref{d}) hold for the $s$-th iteration. For $C_3\geq 20\sqrt{186}$, we have
\begin{align*}
\max_{q\in[n]}\|\btheta^{s+1,(q)}-\btheta^{s+1}\|_{2}\leq \frac{C_{3}\kappa}{\tilde{c}_0} \sqrt{\frac{\tilde{c}_1\log T}{tM}}
\end{align*}
with probability $1-O(T^{-4})$.
\end{lemma}

\begin{proof}[Proof of Lemma \ref{ttr}]
For any $1 \leq q \leq n$, the following expression holds:
\begin{align*}
\btheta^{s+1}-\btheta^{s+1,(q)}= & \btheta^s-\eta \nabla \cL_{\lambda,t}\left(\btheta^s\right)-\left[\btheta^{s,(q)}-\eta \nabla \cL_{\lambda,t}^{(q)}\left(\btheta^{s,(q)}\right)\right] \\
= & \btheta^s-\eta \nabla \cL_{\lambda,t}\left(\btheta^s\right)-\left[\btheta^{s,(q)}-\eta \nabla \cL_{\lambda,t}\left(\btheta^{s,(q)}\right)\right] -\eta\left(\nabla \cL_{\lambda,t}\left(\btheta^{s,(q)}\right)-\nabla \cL_{\lambda,t}^{(q)}\left(\btheta^{s,(q)}\right)\right) \\
= & \underbrace{\left(\Ib_{nd}-\eta \int_0^1 \nabla^2 \cL_{\lambda,t}(\btheta(\tau)) \mathrm{d} \tau\right)\left(\btheta^s-\btheta^{s,(q)}\right)}_{:=\boldsymbol{v}_1}-\underbrace{\eta\left(\nabla \cL_{\lambda,t}\left(\btheta^{s,(q)}\right)-\nabla \cL_{\lambda,t}^{(q)}\left(\btheta^{s,(q)}\right)\right)}_{:=\boldsymbol{v}_2},
\end{align*}
where we denote $\btheta(\tau)=\btheta^{s,(q)}+\tau\left(\btheta^s-\btheta^{s,(q)}\right)$. In what follows, we control $\boldsymbol{v}_1$ and $\boldsymbol{v}_2$ separately.
Regarding the term $\boldsymbol{v}_1$, repeating the same argument as (\ref{IA}) yields
\begin{align*}
\left\|\boldsymbol{v}_1\right\|_2 =\Big\|\Big( \Ib_{nd}-\eta \int_{0}^{1}\nabla^{2} \mathcal{L} (\btheta(\tau))\Big) d \tau\,\left(\btheta^{s}-\btheta^{s,(q)}\right)\Big\|_{2}
\leq\left(1-\frac{\eta \tilde{c}_0tM}{10 \kappa}\right)\left\|\btheta^s-\btheta^{s,(q)}\right\|_2.
\end{align*}
When considering $\boldsymbol{v}_2$, utilizing the definition, we have
\begin{align*}
\frac{1}{\eta} \boldsymbol{v}_2 =&\sum_{j\in[n]\setminus\{q\}}A_{qj}\sum_{\ell\in\cT_{t,jq}}\sum_{m=1}^{M}\left(-y_{q, j}^{(m)}(\ell)  +y_{q, j}^*(\ell)\right)\left(\boldsymbol{c}_{\ell j}-\boldsymbol{c}_{\ell q}\right).
\end{align*}
Let $\frac{1}{\eta} \boldsymbol{v}_{2}=\boldsymbol{u}=(\bu_{1}^{\top},\bu_{2}^{\top},\ldots,\bu_{n}^{\top})$, where $\bu_j\in\mathbb{R}_{d}$. Then we have
\begin{align*}
\bu_j= \begin{cases}\sum_{\ell\in\cT_{t,jq}}\sum_{m=1}^{M}\left(-y_{q, j}^{(m)}(\ell)  +y_{q, j}^*(\ell)\right)\bX_{\ell}, & \text { if }(j, q) \in \mathcal{E} ; \\ 
\sum_{s:(s, q) \in \mathcal{E}}\sum_{\ell\in \cT_{t,sq}}\sum_{m=1}^{M}\left(-y_{q, s}^{(m)}(\ell)  +y_{q, s}^*(\ell)\right)\bX_{\ell}, & \text { if } j=q ; \\ 0, & \text { otherwise. }\end{cases}
\end{align*}
We then construct martingale sequences to bound $\bu$. For the given $\cG$, let $\bV_{\ell m}=\bm{1}(\ba_\ell=(j,q))\\\left(-y_{q, j}^{(m)}(\ell)  +y_{q, j}^*(\ell)\right)\bX_{\ell}$. Notice that $\EE[\bV_{\ell m}\given \cH_{\ell-1},\bX_{\ell}]=0$ and $\|\bV_{\ell m}\|_{2}\leq c_{\bX}$. Therefore, the sequence $\{\bV_{1 1},\bV_{1 2},\ldots,\bV_{1M},\ldots,\bV_{t 1},\bV_{t 2},\ldots,\bV_{t M}\}$ is a martingale difference sequence.
Define predictable quadratic variation processes $\bV_{\text{col}}:=\sum_{\ell=1}^{t}\sum_{m=1}^{M}\EE[\bV_{\ell m}\bV_{\ell m}^{\top}\given \cH_{\ell-1},\bX_{\ell}]$ and $\bV_{\text{row}}:=\sum_{\ell=1}^{t}\sum_{m=1}^{M}\EE[\bV_{\ell m}^{\top}\bV_{\ell m}\given \cH_{\ell-1},\bX_{\ell}]$. We can obtain the following inequalities based on $\cA_t^2$:
\begin{align*}
\|\bV_{\text{col}}\|_{2}=&\|\sum_{\ell=1}^{t}\sum_{m=1}^{M}\bm{1}(\ba_\ell=(j,q))\EE[(-y_{q, j}^{(m)}(\ell)  +y_{q, j}^*(\ell))^2\given \bX_{\ell}]\bX_{\ell}\bX_{\ell}^{\top}\|_{2}
\leq \frac{13}{2}\tilde{c}_1 tM/np,\nonumber\\
|\bV_{\text{row}}|=&|\sum_{\ell=1}^{t}\sum_{m=1}^{M}\bm{1}(\ba_\ell=(j,q))\EE[(-y_{q, j}^{(m)}(\ell)  +y_{q, j}^*(\ell))^2\given \bX_{\ell}]\bX_{\ell}^{\top}\bX_{\ell}|
\leq \frac{13}{2}\tilde{c}_1tM/np.
\end{align*}
Using Lemma~\ref{FreedmanRec}, for any $j$ such that $(j,q)\in\cE$, we have 
\begin{align}\label{uj}
\|\bu_{j}\|_{2}\leq \Big\|\sum_{\ell=1}^{t}\sum_{m=1}^{M}\bV_{\ell m}\Big\|_{2}\leq 2\sqrt{\frac{13\tilde{c}_1tM\log T}{np}}+\frac{8}{3}c_{\bX}\log T
\end{align}
with probability $1-O(T^{-4})$.

For the given graph $\cG$, let $\tilde{j}_{k}$ represent the $k$-th item such that $(\tilde{j}_{k},q)\in\cE$. Similarly, set $\tilde{\bV}_{\ell km}=\bm{1}(\ba_{\ell}=(\tilde{j}_{k},q))\left(-y_{q, \tilde{j}_{k}}^{(m)}(\ell)  +y_{q, \tilde{j}_{k}}^*(\ell)\right)\bX_{\ell}$. We have $\EE[\tilde{\bV}_{\ell km}\given \cH_{\ell-1},\bX_{\ell}]=0$ and $\|\tilde{\bV}_{\ell km}\|_{2}\leq c_{\bX}$. Therefore, the sequence $\{\tilde{\bV}_{111},\tilde{\bV}_{112},\ldots,\tilde{\bV}_{11M},\tilde{\bV}_{121},\ldots,\tilde{\bV}_{211},\ldots,\tilde{\bV}_{\ell 11},\tilde{\bV}_{\ell 12},\ldots\}$ is a martingale difference sequence. We have the predictable quadratic variation processes
\begin{align*}
|\tilde{\bV}_{\text{row}}|=&|\sum_{k}\sum_{\ell=1}^{t}\sum_{m=1}^{M}\EE[\tilde{\bV}_{\ell km}^{\top}\tilde{\bV}_{\ell km}\given \cH_{\ell-1},\bX_{\ell}]|\leq  10\tilde{c}_1 tM \text{ and}\nonumber\\
\|\tilde{\bV}_{\text{col}}\|_{2}
=&\|\sum_{k}\sum_{\ell=1}^{t}\sum_{m=1}^{M}\EE[\tilde{\bV}_{\ell km}\tilde{\bV}_{\ell km}^{\top}\given \cH_{\ell-1},\bX_{\ell}]\|_{2} \nonumber\\
=&\|\sum_{k}\sum_{\ell=1}^{t}\sum_{m=1}^{M}\bm{1}(\ba_{\ell}=(\tilde{j}_{k},q))\EE[(-y_{q, \tilde{j}_{k}}^{(m)}(\ell)  +y_{q, \tilde{j}_{k}}^*(\ell))^2 \given \bX_{\ell}]\bX_\ell\bX_\ell^\top\|_{2} \nonumber\\
\leq & 10\tilde{c}_1 tM.
\end{align*}
Applying Lemma~\ref{FreedmanRec} yields
\begin{align}\label{uq}
\|\bu_{q}\|_{2}\leq \|\sum_{k}\sum_{\ell=1}^{t}\sum_{m=1}^{M}\tilde{\bV}_{\ell km}\|_{2}\leq 8\sqrt{5\tilde{c}_1tM\log T}+\frac{8}{3}c_{\bX}\log T
\end{align}
with probability $1-O(T^{-4})$.
Combining (\ref{uj}) and (\ref{uq}) results in
\begin{align*}
\left\|\boldsymbol{u}\right\|_2 ^2=&\sum_{j:(j,q)\in\cE}\|\bu_{j}\|_{2}^{2}+\|\bu_{q}\|_{2}^{2}\leq 744\tilde{c}_1tM\log T
\end{align*}
holds with probability $1-O(T^{-4})$. 
Putting the above results together, we see that
\begin{align*}
\left\|\btheta^{s+1}-\btheta^{s+1,(q)}\right\|_2 \leq & \|\bv_{1}\|_{2}+\eta\|\bu\|_{2}
\leq  \left(1-\frac{\eta\tilde{c}_0tM}{10 \kappa} \right) \frac{C_3 \kappa}{\tilde{c}_0} \sqrt{\frac{\tilde{c}_1\log T}{tM}} +\eta \sqrt{744\tilde{c}_1tM\log T}\nonumber\\
\leq & \frac{C_{3}\kappa}{\tilde{c}_0}\sqrt{\frac{\tilde{c}_1\log T}{tM}},
\end{align*}
as long as $C_3\geq 20\sqrt{186}$.
\end{proof}

\begin{lemma}\label{str}
Conditioning on event $\cA\cap \cA_0\cap\cA_t^2$, suppose (\ref{d}) holds for the $s$-th iteration. We have 
\begin{align*}
\max_{q\in [n]}\|\btheta_{q}^{s+1}-\btheta_{q}^{*}\|_{2}\leq
\frac{C_{4}\kappa^2}{\tilde{c}_0}
\sqrt{\frac{\tilde{c}_1\log T}{ptM}}
\end{align*}
with probability $1-O(T^{-4})$ for $C_{4}\geq C_{2}+C_{3}$.
\end{lemma}
\begin{proof}[Proof of Lemma \ref{str}]
Utilizing Lemmas~\ref{looTr} and \ref{ttr}, we have
\begin{align*}
\max_{q\in [n]}\|\btheta_{q}^{s+1}-\btheta_{q}^{*}\|_{2}&\leq \max_{q\in [n]}\|\btheta_{q}^{s+1}-\btheta_{q}^{s+1,(q)}\|_{2}+\max_{q\in [n]}\|\btheta_{q}^{s+1,(q)}-\btheta_{q}^{*}\|_{2}\leq \frac{C_{4}\kappa^2}{\tilde{c}_0}\sqrt{\frac{\tilde{c}_1\log T}{ptM}}
\end{align*}
with the condition that $C_4 \geq C_3+C_2$.
\end{proof}

\section{Proof of Lemmas in Section~\ref{app:sec:inf}}\label{app:sec:infr}
\subsection{Proof of Lemma~\ref{lem:infma}}
In this proof, we omit $\btheta$ in $\bPhi_t(\btheta)$ and $\bPhi_{t,11}(\btheta)$ for notation simplicity.
First, we show that $\bPsi_t(\btheta)$ is invertible. 
Utilizing Lemma~\ref{Hessmin}, we have 
the non-zero eigenvalues of $\nabla^2\cL_t(\btheta)$ satisfy 
$\tilde{c}_1tM\gtrsim \lambda_1\geq \ldots\geq \lambda_{(n-1)d}\gtrsim \tilde{c}_0tM$ with probability $1-O(\max\{T^{-2},n^{-10}\})$. 
Let $\bv_1,\ldots,\bv_{(n-1)d}$ be the corresponding eigenvectors with $\|\bv_j\|_2=1$, $j\in[(n-1)d]$. 
Then we have 
\begin{align}\label{eig}
\nabla^2\cL_t(\btheta)\bv_j=\lambda_j \bv_j.
\end{align}
From the closed form of $\nabla^2\cL_t(\btheta)$, we have that the remaining eigenvalues are $\lambda_{(n-1)d+1}=\ldots=\lambda_{nd}=0$, with each column of $\Jb$ being an eigenvector.
By noticing that $\bc_{\ell i}^\top \Jb=\bX_{\ell}\Ib_{d}$ for all $i\in[n]$, and combining (\ref{eig}), we have 
\begin{align*}
\lambda_j \Jb^\top \bv_j=\Jb^\top \nabla^2\cL_t(\btheta)\bv_j=0.
\end{align*}
Therefore, for $j=1,\ldots,(n-1)d$, we can obtain
\begin{align}\label{eigSig}
\left(\begin{array}{cc}
\nabla^2 \cL_t(\btheta) & \nabla f_t(\btheta) \\
\nabla f_t(\btheta)^{\top} & \bm{0}
\end{array}\right)
\left(\begin{array}{c}
\bv_j\\
\bm{0}
\end{array}\right)=
\lambda_j \left(\begin{array}{c}
\bv_j\\
\bm{0}
\end{array}\right),
\end{align}
which implies that $\lambda_1, \ldots, \lambda_{(n-1)d}$ are the eigenvalues of $\bPsi_t$, and $(\bv_{j}^\top,\bm{0}_{d}^\top)^\top$ are the corresponding eigenvectors. Besides, noticing that $f_t(\btheta)=tM\sum_{i\in[n]}\btheta_{ i}=tM(\Ib_{d},\ldots,\Ib_{d})(\btheta_1,\ldots,\btheta_{n})^\top$, we have $\nabla f_t(\btheta)=tM\Big(\Ib_{d},\ldots,\Ib_{d}\Big)^\top=tM \Jb$.
Using $\Jb^\top \Jb=n\Ib_{d}$, we can obtain
\begin{align*}
&\left(\begin{array}{cc}
\nabla^2 \cL_t(\btheta) & \nabla f_t(\btheta) \\
\nabla f_t(\btheta)^{\top} & \bm{0}
\end{array}\right)
\left(\begin{array}{c}
\frac{1}{\sqrt{n}}\Jb\\
\Ib_{d}
\end{array}\right)=
\sqrt{n}tM\left(\begin{array}{c}
\frac{1}{\sqrt{n}}\Jb\\
\Ib_{d}
\end{array}\right)\text{ and}\nonumber\\
&\left(\begin{array}{cc}
\nabla^2 \cL_t(\btheta) & \nabla f_t(\btheta) \\
\nabla f_t(\btheta)^{\top} & \bm{0}
\end{array}\right)
\left(\begin{array}{c}
\frac{1}{\sqrt{n}}\Jb\\
-\Ib_{d}
\end{array}\right)=
-\sqrt{n}tM\left(\begin{array}{c}
\frac{1}{\sqrt{n}}\Jb\\
-\Ib_{d}
\end{array}\right).
\end{align*}
Therefore, $\sqrt{n}tM$ and $-\sqrt{n}tM$ are eigenvalues, with the corresponding eigenvectors given by the columns of $(\frac{1}{\sqrt{n}}\Jb^\top,\Ib_{d})^\top$ and $(\frac{1}{\sqrt{n}}\Jb^\top,-\Ib_{d})^\top$ respectively. Combining above results, the eigenvalues of $\bPsi_t$ are $\sqrt{n}tM$ ($d$ multiplication), $\lambda_1$, $\ldots$, $\lambda_{(n-1)d}$, $-\sqrt{n}tM$ ($d$ multiplication) in decreasing order. Hence, $\bPsi_t(\btheta)$ is invertible and $\|\bPhi_t\|_{2}\lesssim \frac{1}{\tilde{c}_0tM}$.

Then we present the specific form of $\bPhi_t$.
Let $
\bPhi_t=\left(\begin{array}{cc}
\bPhi_{t,11}&\bPhi_{t,12}\\
\bPhi_{t,21}&\bPhi_{t,22}
\end{array}\right)$ with the same partition as $\bPsi_t(\btheta)$.
We have 
\begin{equation}
\begin{gathered}\label{eq}
\bPhi_{t,11} \nabla^2 \cL_t(\btheta)+ \bPhi_{t,12}\Big(\nabla f_t(\btheta)\Big)^{\top}=\Ib_{nd} \\
\bPhi_{t,11} \nabla f_t(\btheta)=\bm{0} \\
\bPhi_{t,12}^{\top} \nabla^2 \cL_t(\btheta)+\bPhi_{t,22}
\Big(\nabla f_t(\btheta)\Big)^{\top}=\bm{0} \\
\bPhi_{t,12}^\top
\nabla f_t(\btheta)=\Ib_{d}
\end{gathered}
\end{equation}
Right multiplying the first equation of (\ref{eq}) by $\Jb$, we can obtain $\bPhi_{t,12}=\frac{1}{ntM}\Jb$ by noticing that $\nabla^2\cL_t(\btheta)\Jb=\bm{0}$ and $\Jb^\top \Jb=n\Ib_{d}$. Further, we have $\bPhi_{t,22}=\bm{0}$.
Then we focus on $\bPhi_{t,11}$.
From (\ref{eigSig}),
we have
\begin{align}\label{Siginv}
\frac{1}{\lambda_j }
\left(\begin{array}{c}
\bv_j\\
\bm{0}
\end{array}\right)=\left(\begin{array}{cc}
\bPhi_{t,11}&\frac{1}{ntM}\Jb\\
\frac{1}{ntM}\Jb^\top&\bm{0}
\end{array}\right)
\left(\begin{array}{c}
\bv_j\\
\bm{0}
\end{array}\right).
\end{align}
By combining the second equation in (\ref{eq}),
we can conclude that the eigenvalues of $\bPhi_{t,11}$ are $\frac{1}{\lambda_{(n-1)d}},\ldots,\frac{1}{\lambda_{1}},0(d\text{ multiplication})$. The corresponding eigenvectors are $\bv_{(n-1)d},\ldots,\bv_1$ and $\bw_1,\ldots,\bw_d$, where $\bw_i$ is the $i$-th column of $\Jb$.
Therefore, we have
\begin{align}\label{kk}
\max_{k\in[n]}\|(\bPhi_{t,11})_{kk}\|_{2}\leq\|\bPhi_{t,11}\|_2\lesssim\frac{n}{tM}.
\end{align} 
Let $\Eb_k=\be_k\otimes \bu \in\RR^{nd}$, where $\bu\in \RR^d$ satisfying $\|\bu\|_2=1$. We can represent $\Eb_k$ as follows:
\begin{align}\label{Ek}
\Eb_k=\alpha_1\bv_1+\ldots+\alpha_{(n-1)d}\bv_{(n-1)d}+\alpha_{(n-1)d+1}\bw_1+\ldots+\alpha_{nd}\bw_d.
\end{align}
Left multiplying (\ref{Ek})  by $\Jb^\top$, we have $\bu=(n\alpha_{(n-1)d+1}, \ldots, n\alpha_{nd})^\top$
using (\ref{Siginv}). Hence, we can obtain $\alpha_{(n-1)d+i}=\bu_{i}/n$ for $i\in[n]$.
Besides, we have
\begin{align}\label{E_k}
1=\|\Eb_k\|_2&=\alpha_1^2\bv_1^\top \bv_1 +\ldots+\alpha_{(n-1)d}^2\bv_{(n-1)d}^\top \bv_{(n-1)d} +\alpha_{(n-1)d+1}^2\bw_1^\top \bw_1+\ldots+\alpha_{nd}^2\bw_d^\top \bw_d\nonumber\\
&=\alpha_1^2+\ldots+\alpha_{(n-1)d}^2+\frac{\bv_{1}^2}{n}+\ldots+\frac{\bv_{n}^2}{n}\nonumber\\
&=\alpha_1^2+\ldots+\alpha_{(n-1)d}^2+\frac{1}{n}.
\end{align}
We can bound the minimum eigenvalue of $(\bPhi_{t,11})_{kk}$ as follows:
\begin{align*}
\lambda_{\min}((\bPhi_{t,11})_{kk})=\min_{\Eb_k} \Eb_k^\top \bPhi_{t,11} \Eb_k=\frac{1}{\lambda_1}\alpha_1^2+\ldots+\frac{1}{\lambda_{(n-1)d}}\alpha_{(n-1)d}^2\geq \frac{1}{\lambda_1}\left(\alpha_1^2+\ldots+\alpha_{(n-1)d}^2\right)
\end{align*}
Combining (\ref{E_k}), we can conclude
\begin{align}\label{kkmin}
\lambda_{\min}((\bPhi_{t,11})_{kk})\gtrsim \frac{1}{\tilde{c}_1tM}\asymp \frac{n}{tM}.
\end{align}

By noticing that 
\begin{align*}
(\Eb_j-\Eb_k)^\top \bPhi_{t,11}(\Eb_j-\Eb_k)=\bu^\top (\bPhi_{t,11})_{jj}\bu +\bu^\top (\bPhi_{t,11})_{kk}\bu -2\bu^\top (\bPhi_{t,11})_{kj}\bu,
\end{align*}
we also have
\begin{align}\label{kj}
\|(\bPhi_{t,11})_{kj}\|_2=& \max_{\bu}\bu^\top (\bPhi_{t,11})_{kj}\bu
\lesssim \|\bPhi_{t,11}\|_2+\|(\bPhi_{t,11})_{jj}\|_2+\|(\bPhi_{t,11})_{kk}\|_2\lesssim \frac{n}{tM}.
\end{align}

\subsection{Proof of Lemma~\ref{lem:clt}}
Let $S_t=\nabla\cL_t(\ths)$.  If $\ba_\ell=(i_\ell,j_\ell)$, define
\begin{align}\label{eq:j2-score-increment}
s_{\ell m}
=\left\{-y_{j_\ell i_\ell}^{(m)}(\ell)
+\frac{e^{\bX_{\ell}^{\top}\btheta_{i_\ell}^{*}}}
{e^{\bX_{\ell}^{\top}\btheta_{i_\ell}^{*}}+e^{\bX_{\ell}^{\top}\btheta_{j_\ell}^{*}}}\right\}
(\bc_{\ell i_\ell}-\bc_{\ell j_\ell}),\qquad
S_t=\sum_{\ell=1}^t\sum_{m=1}^M s_{\ell m}.
\end{align}
Let $
\mathcal{G}_{\ell,m-}
=\sigma\left(\cG,\cH_{\ell-1},\bX_\ell,\ba_\ell,
y_\ell^{(1)},\ldots,y_\ell^{(m-1)}\right). $
Conditional on $\mathcal{G}_{\ell,m-}$, the action and the covariate are fixed.  The contextual BTL model and the conditional independence of the $M$ comparison outcomes imply
\begin{align}\label{eq:j2-score-mds}
\EE(s_{\ell m}\given \mathcal{G}_{\ell,m-})=0 .
\end{align}
Thus the comparison-score increments in \eqref{eq:j2-score-increment}, not the terminal inverse-Hessian weighted terms, form the martingale difference array.

Let $
b=t_0=\left\lceil(cn^2p\log T)^{1/(1-\alpha)}\right\rceil $
be the burn-in time from the proof of Lemma~\ref{lem:covT}.  Write
$H_t=\nabla^2\cL_t(\ths)$ and, for the realized action at time $\ell$, $
h_\ell=
\frac{e^{\bX_{\ell}^{\top}\btheta_{i_\ell}^{*}}
e^{\bX_{\ell}^{\top}\btheta_{j_\ell}^{*}}}
{\left(e^{\bX_{\ell}^{\top}\btheta_{i_\ell}^{*}}
+e^{\bX_{\ell}^{\top}\btheta_{j_\ell}^{*}}\right)^2}
(\bc_{\ell i_\ell}-\bc_{\ell j_\ell})(\bc_{\ell i_\ell}-\bc_{\ell j_\ell})^\top . $
Then $H_t=M\sum_{\ell=1}^t h_\ell$.

For $\ell>b$, define an oracle policy which uses the same exploration probability and the same randomization as the actual policy, but replaces the exploitation major item by
$i_\ell^*=\argmax_{i\in[n]}\bX_\ell^\top\btheta_i^*$.  Given the major item, the comparator is sampled uniformly from its graph neighbors, as in the actual policy.  Let $Q_\ell^*$ be the conditional-on-$\cG$ expected one-comparison Hessian contribution at $\ths$ under this oracle policy, where the expectation integrates over $\bX_\ell$ and the round-$\ell$ randomization.  Define
\begin{align}
\bar H_t=H_b+M\sum_{\ell=b+1}^t Q_\ell^*,\qquad
\bar\bPsi_t=
\left(\begin{array}{cc}
\bar H_t&tM\Jb\\
tM\Jb^\top&\bm{0}
\end{array}\right),\qquad
\bar\bPhi_t=\bar\bPsi_t^{-1}.
\end{align}
Since $H_b$ is determined by the history up to time $b$ and each $Q_\ell^*$ depends only on $\cG$ and the oracle rule, $\bar H_t$ is measurable with respect to the information available at time $b$.  Let $\bar\bPhi_{t,11}$ be the $nd\times nd$ upper-left block of $\bar\bPhi_t$, and set
\begin{align}
\bar V_{k,t}=(\bar\bPhi_{t,11})_{kk}.
\end{align}

We next work on the identifiable subspace.  Let $\Rb$ be an $(n-1)\times n$ matrix satisfying $\Rb\Rb^\top=\Ib_{n-1}$ and $\Rb\bm{1}_n=\bm{0}$, and set $
\Ob=\Rb\otimes\Ib_d . $
As in the proof of Lemma~\ref{eigmin}, $\Ob\Ob^\top=\Ib_{(n-1)d}$ and the row space of $\Ob$ is the zero-sum subspace.  Define
\begin{align}
G_t=\Ob H_t\Ob^\top,\qquad \bar G_t=\Ob\bar H_t\Ob^\top .
\end{align}
The same constrained-Hessian algebra used in the proof of Lemma~\ref{lem:infma} gives
\begin{align}\label{eq:j2-kkt-representation}
\bPhi_{t,11}^*=\Ob^\top G_t^{-1}\Ob,\qquad
\bar\bPhi_{t,11}=\Ob^\top \bar G_t^{-1}\Ob .
\end{align}
Indeed, the matrices on the right-hand side annihilate $\Jb$ and invert $H_t$ and $\bar H_t$, respectively, on the zero-sum subspace, while the constraint block in both KKT matrices is $tM\Jb$.

We record the scale of the reference information.  On the event $\cA$ in Lemma~\ref{dgr}, Assumptions~\ref{ass:Sigx} and \ref{assmp:kappa} bound the BTL variance weight at $\ths$ above and below by positive constants.  For the oracle exploitation information, Assumption~\ref{ass:Hesx} gives $\lambda_{\min}(\bSigma_i^*)\asymp 1/n$, the degree bounds in Lemma~\ref{dgr} give $d_i\asymp np$, and the graph spectral gap in Lemma~\ref{eigmin} gives the zero-sum information scale $1/n$ per round.  Exploration contributes a positive semidefinite matrix with operator norm $O(1/n)$.  Moreover, $\|H_b\|_2=O(bM)$.  Since $T_0\geq C\tilde T_0$ and \eqref{tT0} imply $
\frac{nb}{t}\leq \frac{nb}{T_0}
\lesssim n(n^2p)^{-\{1-\alpha-\alpha^2\}/\{\alpha(1-\alpha)\}}=o(1), $
using $np>C\log n$ and $\alpha\in(0,1/2)$, the initial block $H_b$ is negligible relative to $tM/n$.  Hence
\begin{align}\label{eq:j2-barG-scale}
c\frac{tM}{n}\Ib_{(n-1)d}\preceq \bar G_t\preceq C\frac{tM}{n}\Ib_{(n-1)d}.
\end{align}
By \eqref{eq:j2-kkt-representation},
\begin{align}\label{eq:j2-barPhi-scale}
\|\bar\bPhi_{t,11}\|_2\leq C\frac{n}{tM},\qquad
c\frac{n}{tM}\Ib_d\preceq \bar V_{k,t}\preceq C\frac{n}{tM}\Ib_d .
\end{align}

It remains to compare $H_t$ with the reference matrix.  Define
\begin{align}\label{eq:j2-rho-def}
\rho_t=\left\|\bar G_t^{-1/2}(G_t-\bar G_t)\bar G_t^{-1/2}\right\|_2 .
\end{align}
Let $\cF_{\ell-1}=\sigma(\cG,\cH_{\ell-1})$.  Since
\begin{align}\label{eq:j2-H-decomp}
H_t-\bar H_t
=M\sum_{\ell=b+1}^t\{h_\ell-\EE(h_\ell\given\cF_{\ell-1})\}
+M\sum_{\ell=b+1}^t\{\EE(h_\ell\given\cF_{\ell-1})-Q_\ell^*\},
\end{align}
we control the two sums separately.  The first sum in \eqref{eq:j2-H-decomp} is a self-adjoint matrix martingale difference sum with uniformly bounded increments and predictable quadratic variation of order $t$.  Applying Lemma~\ref{FreedmanRec} to its self-adjoint dilation yields
\begin{align}
\left\|\sum_{\ell=b+1}^t\{h_\ell-\EE(h_\ell\given\cF_{\ell-1})\}\right\|_2
=O_p(\sqrt{t\log T}+\log T).
\end{align}

For the second sum, let $
\delta_{\ell-1}=\max_{i\in[n]}\|\hat{\btheta}_i(\ell-1)-\btheta_i^*\|_2 . $
We show that
\begin{align}\label{eq:j2-policy-bias}
\|\EE(h_\ell\given\cF_{\ell-1})-Q_\ell^*\|_2\leq C\delta_{\ell-1}.
\end{align}
The only difference between the two conditional expectations is the exploitation major item.  If $\max_i\|\bar\btheta_i-\btheta_i^*\|_2\leq\delta$, the argument in the proof of Lemma~\ref{lem:hes}, specifically \eqref{ii}, gives $
\PP(\argmax_j\bX^\top\bar\btheta_j=i,\ \argmax_j\bX^\top\btheta_j^*\neq i)
\leq C\delta $
and the same bound with the roles of $\bar\btheta$ and $\btheta^*$ reversed.  Thus the probability that the estimated and oracle winners differ through item $i$ is $O(\delta)$.  For any block vector $v=(v_1^\top,\ldots,v_n^\top)^\top$, the boundedness of $\bX$, the bounded BTL weights, and $d_{\max}/d_{\min}=O(1)$ on $\cA$ give
\[
\left|v^\top\{\EE(h_\ell\given\cF_{\ell-1})-Q_\ell^*\}v\right|
\leq C\delta_{\ell-1}\sum_{i=1}^n
\frac{1}{d_i}\sum_{j:(i,j)\in\cE}\|v_i-v_j\|_2^2
\leq C\delta_{\ell-1}\|v\|_2^2,
\]
which proves \eqref{eq:j2-policy-bias}.

For $b<\ell\leq T_0$, Theorem~\ref{thm:2rate} gives
$\delta_{\ell-1}\leq Cn\sqrt{\log T/M}\,\ell^{\alpha-1/2}$.  For $\ell>T_0$, Theorem~\ref{thm:BT} and Corollary~\ref{cor} give
$\delta_{\ell-1}\leq C\sqrt{n\log T/(p\ell M)}$.  Combining these bounds with \eqref{eq:j2-H-decomp}--\eqref{eq:j2-policy-bias},
\begin{align}\label{eq:j2-H-rate}
\|H_t-\bar H_t\|_2
=O_p\left(
M\sqrt{t\log T}+M\log T
+n\sqrt{M\log T}\,T_0^{\alpha+1/2}
+\sqrt{\frac{Mnt\log T}{p}}
\right).
\end{align}
Using \eqref{eq:j2-barG-scale}, \eqref{eq:j2-H-rate} implies
\begin{align}\label{eq:j2-rho-rate}
\rho_t
=O_p\left(
n\sqrt{\frac{\log T}{t}}
+\frac{n\log T}{t}
+\frac{n^2\sqrt{\log T/M}\,T_0^{\alpha+1/2}}{t}
+n^{3/2}\sqrt{\frac{\log T}{pMt}}
\right).
\end{align}
We now verify the stronger rate needed below.  Since $t\geq T_0\geq C\tilde T_0$, \eqref{tT0} gives $
n^{3/2}\sqrt{\frac{\log T}{T_0}}=o(1), 
n^{3/2}\frac{\log T}{T_0}=o(1),$
using the first term in \eqref{tT0} and $np>C\log n$.  For the third term in \eqref{eq:j2-rho-rate}, the second term in \eqref{tT0} gives $
n^{5/2}\sqrt{\frac{\log T}{M}}\,T_0^{\alpha-1/2}
\leq
C n^{-1/2}p^{(1-2\alpha)/2}M^{-\alpha}=o(1), $
because $p\leq 1$ and $M\geq1$.  The fourth term satisfies, again by the second term in \eqref{tT0}, $
n^2\sqrt{\frac{\log T}{pMT_0}}
\leq
C n^{2-3/(1-2\alpha)}
(\log T)^{-\alpha/(1-2\alpha)}=o(1). $
Therefore
\begin{align}\label{eq:j2-sqrtn-rho}
\sqrt n\,\rho_t=o_p(1).
\end{align}

We first prove the CLT for the reference statistic.  Let $S_b=\nabla\cL_b(\ths)$ and fix $\bu\in\RR^d$.  For $\ell>b$, define $
\zeta_{\ell m}
=\bu^\top \bar V_{k,t}^{-1/2}
(\bar\bPhi_{t,11}s_{\ell m})_k . $
Since $\bar\bPhi_{t,11}$ and $\bar V_{k,t}$ are measurable at time $b$, \eqref{eq:j2-score-mds} gives
$\EE(\zeta_{\ell m}\given \mathcal{G}_{\ell,m-})=0$.  The predictable quadratic variation is
\begin{align}
\sum_{\ell=b+1}^t\sum_{m=1}^M
\EE(\zeta_{\ell m}^2\given\mathcal{G}_{\ell,m-})
=\bu^\top \bar V_{k,t}^{-1/2}
(\bar\bPhi_{t,11}(H_t-H_b)\bar\bPhi_{t,11})_{kk}
\bar V_{k,t}^{-1/2}\bu .
\end{align}
Since $H_t-H_b=(\bar H_t-H_b)+(H_t-\bar H_t)$, and
$\bar\bPhi_{t,11}\bar H_t\bar\bPhi_{t,11}=\bar\bPhi_{t,11}$ by \eqref{eq:j2-kkt-representation}, \eqref{eq:j2-barPhi-scale} and \eqref{eq:j2-rho-def} yield
\begin{align}
\sum_{\ell=b+1}^t\sum_{m=1}^M
\EE(\zeta_{\ell m}^2\given\mathcal{G}_{\ell,m-})
=\|\bu\|_2^2+O_p(nb/t)+O_p(\rho_t)
=\|\bu\|_2^2+o_p(1).
\end{align}
Moreover, by \eqref{eq:j2-barPhi-scale} and the boundedness of $\bX$,
\begin{align}
|\zeta_{\ell m}|\leq C\|\bu\|_2\sqrt{\frac{n}{tM}}=o(1),
\end{align}
uniformly in $\ell$ and $m$.  Hence the conditional Lindeberg condition holds.  Applying the martingale central limit theorem, Corollary~3.1 in \citet{HALL198051}, and then the Cramer--Wold device,
\begin{align}\label{eq:j2-reference-clt-post-b}
\bar V_{k,t}^{-1/2}\{\bar\bPhi_{t,11}(S_t-S_b)\}_k
\overset{d}\longrightarrow N(0,\Ib_d).
\end{align}

The early score is negligible.  From \eqref{eq:j2-score-mds} and the boundedness of the score increments,
$\EE\|S_b\|_2^2=O(bM)$.  Together with \eqref{eq:j2-barPhi-scale},
\begin{align}\label{eq:j2-early-score}
\left\|\bar V_{k,t}^{-1/2}(\bar\bPhi_{t,11}S_b)_k\right\|_2
=O_p\left(\sqrt{\frac{nb}{t}}\right)=o_p(1).
\end{align}
Combining \eqref{eq:j2-reference-clt-post-b} and \eqref{eq:j2-early-score},
\begin{align}\label{eq:j2-reference-clt}
\bar V_{k,t}^{-1/2}(\bar\bPhi_{t,11}S_t)_k
\overset{d}\longrightarrow N(0,\Ib_d).
\end{align}

It remains to replace the reference inverse by the terminal inverse.  Let $
E_t=\bar G_t^{-1/2}(G_t-\bar G_t)\bar G_t^{-1/2}. $
Then $\|E_t\|_2=\rho_t$, and on the event $\rho_t<1/2$, $
G_t^{-1}=\bar G_t^{-1/2}C_t\bar G_t^{-1/2}, $ $
C_t=(\Ib_{(n-1)d}+E_t)^{-1}, $ $
\|C_t-\Ib_{(n-1)d}\|_2\leq C\rho_t . $
Let $E_k=\be_k^\top\otimes\Ib_d$ and $R_k=E_k\Ob^\top$.  Then $
V_{k,t}^*:=(\bPhi_{t,11}^*)_{kk}=R_kG_t^{-1}R_k^\top, 
\bar V_{k,t}=R_k\bar G_t^{-1}R_k^\top . $
Set $
B_k=R_k\bar G_t^{-1/2}, 
P_k=\bar V_{k,t}^{-1/2}B_k . $
Then $P_kP_k^\top=\Ib_d$ and $V_{k,t}^*=B_kC_tB_k^\top$.  Therefore
$\bar V_{k,t}^{-1/2}V_{k,t}^*\bar V_{k,t}^{-1/2}
=P_kC_tP_k^\top$ has eigenvalues in $[1-C\rho_t,1+C\rho_t]$.  Since $d$ is fixed and $x^{-1/2}$ is Lipschitz on a compact interval around one,
\begin{align}\label{eq:j2-normalizer-perturb}
\left\|(V_{k,t}^*)^{-1/2}\bar V_{k,t}^{1/2}-\Ib_d\right\|_2
\leq C\rho_t .
\end{align}
Let $
z_t=\bar G_t^{-1/2}\Ob S_t . $
By \eqref{eq:j2-kkt-representation}, the reference and terminal standardized statistics can be written as $
\bar T_{k,t}=P_kz_t,$ $
T_{k,t}^*=(V_{k,t}^*)^{-1/2}\bar V_{k,t}^{1/2}P_kC_tz_t . $
Using \eqref{eq:j2-normalizer-perturb} and $\|C_t-\Ib_{(n-1)d}\|_2\leq C\rho_t$,
\begin{align}\label{eq:j2-standardized-perturb}
\|T_{k,t}^*-\bar T_{k,t}\|_2
\leq C\rho_t\|z_t\|_2 .
\end{align}
Finally, $\|z_t\|_2=O_p(\sqrt n)$.  Indeed, the early part satisfies
$\|\bar G_t^{-1/2}\Ob S_b\|_2=O_p(\sqrt{nb/t})$, while for the post-$b$ part, $
\EE\left(\|\bar G_t^{-1/2}\Ob(S_t-S_b)\|_2^2\given \cH_b\right)
\leq \|\bar G_t^{-1}\|_2\,\EE\{\tr(H_t-H_b)\given \cH_b\}
\leq Cn, $
because $\|\bar G_t^{-1}\|_2\leq Cn/(tM)$ by \eqref{eq:j2-barG-scale} and the total score quadratic variation has trace $O(tM)$.  Thus \eqref{eq:j2-standardized-perturb} and \eqref{eq:j2-sqrtn-rho} imply
\begin{align}\label{eq:j2-final-perturb}
T_{k,t}^*-\bar T_{k,t}=o_p(1).
\end{align}
The conclusion follows from \eqref{eq:j2-reference-clt}, \eqref{eq:j2-final-perturb}, and Slutsky's theorem:
\[
((\bPhi_{t,11}^*)_{kk})^{-1/2}
(\bPhi_{t,11}^*\nabla\cL_t(\ths))_k
\overset{d}\longrightarrow N(0,\Ib_d).
\]

\subsection{Proof of Lemma~\ref{lem:remtm}}
Similar to the proof of Lemma~\ref{gr}, we can obtain $\left\|\nabla \cL_t(\ths)\right\|_{\infty}\lesssim \sqrt{ptM\log T}$.

For the second inequality, first note that 
\begin{align*}
\nabla \cL_t(\hth(t))-\nabla \cL_t(\ths)
=& \sum_{(i,j)\in \mathcal{E},i>j;}\sum_{\ell\in \cT_{t,ij}}M\{\frac{e^{\bX_{\ell}^\top\hth_{i}(t) }}{e^{\bX_{\ell}^\top\hth_{i}(t) }+e^{\bX_{\ell}^\top\hth_{j}(t) }}-\frac{e^{\bX_{\ell}^\top\tht^{*}_{i} }}{e^{\bX_{\ell}^\top\tht^{*}_{i} }+e^{\bX_{\ell}^\top\tht^{*}_{j} }}\}(\bc_{\ell i}-\bc_{\ell j})\nonumber\\
= & \sum_{(i,j)\in \mathcal{E},i>j;}\sum_{\ell\in \cT_{t,ij}}
-M\frac{e^{c_{ji\ell}}}{(1+e^{c_{ji\ell}})^2}(\hth_j(t)-\hth_i(t)-\ths_j+\ths_i)^\top \bX_{\ell}(\bc_{\ell i}-\bc_{\ell j})\nonumber\\
=& \sum_{(i,j)\in \mathcal{E},i>j;}\sum_{\ell\in \cT_{t,ij}}
M\frac{e^{c_{ji\ell}}}{(1+e^{c_{ji\ell}})^2}(\bc_{\ell i}-\bc_{\ell j})(\bc_{\ell i}-\bc_{\ell j})^\top (\hth(t)-\ths),
\end{align*}
where $c_{ji\ell}$ is between $(\hth_j(t)-\hth_i(t))^\top\bX_{\ell}$ and $(\ths_j-\ths_i)^\top\bX_{\ell}$.
By utilizing the result
\begin{align*}
\Big|\frac{e^{c_{ji\ell}}}{(1+e^{c_{ji\ell}})^2}-\frac{e^{(\ths_j-\ths_i)^\top \bX_{\ell}}}{(1+e^{(\ths_j-\ths_i)^\top \bX_{\ell}})^2}\Big|\lesssim \max_{i\in[n]}\|\hth_i(t)-\ths_i\|_2 \|\bX_{\ell}\|_2,
\end{align*}
we can bound the expansion of the gradient as follows:
\begin{align*}
&\left\|\nabla \cL_t(\hth(t))-\nabla \cL_t(\ths)-\nabla^2 \cL_t(\ths)\left(\hth(t)-\ths\right)\right\|_{\infty}\nonumber\\
&\leq  M\Big\|\sum_{(i,j)\in \mathcal{E},i>j;}\sum_{\ell\in \cT_{t,ij}}
\Big(\frac{e^{c_{ji\ell}}}{(1+e^{c_{ji\ell}})^2}-\frac{e^{(\ths_j-\ths_i)^\top \bX_{\ell}}}{(1+e^{(\ths_j-\ths_i)^\top \bX_{\ell}})^2}\Big)(\bc_{\ell i}-\bc_{\ell j})(\bc_{\ell i}-\bc_{\ell j})^\top\Big\|_{\infty} \|\hth(t)-\ths\|_{\infty}\nonumber\\
&\lesssim \tilde{c}_1tM\max_{i\in[n]}\|\hth_i(t)-\ths_i\|_2^2\asymp \frac{\log T}{p}.
\end{align*}

We then prove the third inequality in Lemma~\ref{lem:remtm}. Using the closed form of the Hessian matrix, we have
\begin{align*}
&\left\|\nabla^2 \cL_t(\hth(t))-\nabla^2 \cL_t(\ths)\right\|_{\infty}\nonumber\\
&=\Big\|\sum_{(i,j)\in \mathcal{E},i>j;}\sum_{\ell\in \cT_{t,ij}}\{\frac{e^{\bX_{\ell}^\top\hth_{i}(t) } e^{\bX_{\ell}^\top\hth_{j}(t) }}{(e^{\bX_{\ell}^\top\hth_{i}(t) }+e^{\bX_{\ell}^\top\hth_{j}(t) })^{2}}
-\frac{e^{\bX_{\ell}^\top\tht^{*}_{i} } e^{\bX_{\ell}^\top\tht^{*}_{j} }}{(e^{\bX_{\ell}^\top\tht^{*}_{i} }+e^{\bX_{\ell}^\top\tht^{*}_{j} })^{2}}\}(\bc_{\ell i}-\bc_{\ell j})(\bc_{\ell i}-\bc_{\ell j})^\top\Big\|_{\infty} \nonumber\\
&\lesssim\tilde{c}_1tM \max_{i\in[n]}\|\hth_i(t)-\ths_i\|_2\asymp \sqrt{\frac{tM\log T}{np}}.
\end{align*}

For the last inequality in Lemma~\ref{lem:remtm}, using Lemma~\ref{lem:infma}, we can obtain
\begin{align*}
 \left\|\hat\bPsi_t^{-1}-\bPsi_t^{*-1}\right\|_2 
\leq &\left\|\hat\bPsi_t^{-1}\right\|_2 
\left\|\left(\begin{array}{cc}
\nabla^2 \cL_t(\hth(t)) & \nabla f_t(\hth(t)) \\
\nabla f_t^{\top}(\hth(t)) & \bm{0}
\end{array}\right)-\left(\begin{array}{cc}
\nabla^2 \cL_t(\ths) & \nabla f_t(\ths) \\
\nabla f_t^{\top}(\ths) & \bm{0}
\end{array}\right)\right\|_2 \left\|
\bPsi_t^{*-1}\right\|_2\nonumber\\
\lesssim &(\frac{n}{tM})^2 \left\|\nabla^2 \cL_t(\hth(t))-\nabla^2 \cL_t(\ths)\right\|_{2}\nonumber\\
\lesssim &\sqrt{\frac{n^3\log T}{pt^3M^3}}.
\end{align*}

\section{Ancillary Lemmas}
We provide the degree property of the ER graph and matrix Freedman inequalities in this section for self-completeness. These results are used throughout the proof.
\begin{lemma}\label{dgr}
Assume $\mathcal{G}_{n, p}$ is an ER graph. Let $d_i$ denote the degree of node $i$, with $d_{\min }=\min _{1 \leq i \leq n} d_i$ and $d_{\max }=\max _{1 \leq i \leq n} d_i$. There exists a constant $c$, such that if $p \geq \frac{c \log n}{n}$, then the following event
$$
\cA=\left\{\frac{n p}{2} \leq d_{\min } \leq d_{\max } \leq \frac{3 n p}{2}\right\}
$$
holds with probability $1-O\left(n^{-10}\right)$.
\end{lemma}
\begin{proof}[Proof of Lemma~\ref{dgr}]
A basic result of the ER graph using Chernoff bound.
\end{proof}

We have the following lemma regarding the deviation behavior of the matrix martingale. The proofs can be found in \cite{Tropp2011FREEDMANSIF}.

\begin{lemma}[Rectangular Matrix Freedman's Inequality]\label{FreedmanRec}
Consider a matrix martingale $\left\{\bY_k: k=0,1,\ldots\right\}$ whose values are matrices with dimension $d_1 \times d_2$. Let $\left\{\bX_k: k=1,2,\ldots\right\}$ be the martingale difference sequence. Assume that the difference sequence is uniformly bounded:
$$
\left\|\bX_k\right\|_2 \leq R \quad \text { almost surely } \quad \text { for } k=1,2, \ldots.
$$
Define two predictable quadratic variation processes for this martingale:
\begin{align*}
& \bW_{\mathrm{col}, k}:=\sum_{j=1}^k \mathbb{E}\left(\boldsymbol{X}_j \boldsymbol{X}_j^\top\given \cF_{j-1}\right) \quad \text { and } \\
& \bW_{\text {row}, k}:=\sum_{j=1}^k \mathbb{E}\left(\bX_j^\top \bX_j\given \cF_{j-1}\right) \quad \text { for } k=1,2,\ldots .
\end{align*}
If $\max \left\{\left\|\boldsymbol{W}_{\text {col}, k}\right\|_2,\left\|\boldsymbol{W}_{\text {row}, k}\right\|_2\right\} \leq \sigma^2$ holds for $k=1,2,\ldots$, then the event
\begin{align*}
\left\{\forall k \geq 0, \, \left\|\boldsymbol{Y}_k\right\|_2\leq \sigma\sqrt{2\log(1/\delta)}+\frac{2}{3}R\log(1/\delta)\right\}
\end{align*}
holds with probability $1-(d_1+d_2)\delta$.
\end{lemma}



\end{document}